\documentclass[11pt,bezier,amstex]{article}

\usepackage[in]{fullpage}
\usepackage[utf8]{inputenc}
\usepackage{microtype}
\usepackage{graphicx}
\usepackage{subfigure}
\usepackage{booktabs}
\usepackage{hyperref}

\usepackage{amsmath}
\usepackage{amssymb,amsopn,algorithm,algorithmic,theorem,float,bbm,bm,enumerate,color,multirow,gensymb}
\usepackage{epsfig,times,subfigure,graphicx}
\usepackage{comment}
\usepackage[dvipsnames]{xcolor}
\usepackage{authblk}
\usepackage{afterpage}
\usepackage{thmtools, thm-restate}

\newcommand{\beq}{\begin{equation}}
\newcommand{\eeq}{\end{equation}}

\newtheorem{theo}{Theorem}

\newtheorem{corr}{Corollary}

\newtheorem{asmp}{Assumption}

\newtheorem{defn}{Definition}

\theorembodyfont{\rm}%

\newcommand\I{\mathbb{I}}

\newcommand\s{\mathbb{S}}
\newcommand\R{\mathbb{R}}

\newcommand\E{\mathbb{E}}

\renewcommand{\v}{\mathbf{v}}

\newcommand{\x}{\mathbf{x}}

\newcommand{\cD}{{\cal D}}

\newcommand{\cN}{{\cal N}}
\newcommand{\cP}{{\cal P}}
\newcommand{\cQ}{{\cal Q}}

\newcommand{\cX}{{\cal X}}
\newcommand{\cY}{{\cal Y}}

\newcommand{\cZ}{{\cal Z}}

\newcommand{\vertiii}[1]{{\left\vert\kern-0.25ex\left\vert\kern-0.25ex\left\vert #1
    \right\vert\kern-0.25ex\right\vert\kern-0.25ex\right\vert}}

\DeclareMathOperator{\tr}{Tr}

\DeclareMathOperator{\diag}{diag}

\newcommand{\proof}{\noindent{\itshape Proof:}\hspace*{1em}}

\newcommand{\qed}{\nolinebreak[1]~~~\hspace*{\fill} \rule{5pt}{5pt}\vspace*{\parskip}\vspace*{1ex}}

\newcommand {\commentout}[1] {}

\def\ints{{{\rm Z} \kern -.35em {\rm Z} }}  
\def\smallints{{{\rm Z} \kern -.3em {\rm Z} }}  
\def\pints{{{\rm I} \kern -.15em {\rm N} }}      
\newcommand{\reals}{\mathbb R}

\def\cplx{{{\rm I} \kern -.45em {\rm C} }}       
\def\l2{\rm {\mathcal L}^{2}(\reals)}            

\newcommand{\be}{\begin{eqnarray}}
\newcommand{\ee}{\end{eqnarray}}
\newcommand{\bea}{\begin{eqnarray}}
\newcommand{\eea}{\end{eqnarray}}
\newcommand{\beaa}{\begin{eqnarray*}}
\newcommand{\eeaa}{\end{eqnarray*}}
\newcommand{\bnad}{\begin{nad}}
\newcommand{\enad}{\end{nad}}

\newcommand{\calU}{{\cal U}}
\newcommand{\calV}{{\cal V}}

\title{Hessian based analysis of SGD for Deep Nets:\\ Dynamics and Generalization}
\author{Xinyan Li \footnote{Equal Contribution}  \footnote{Emails: \{lixx1166, guxxx396, zhou0877, chen6271\}@umn.edu, banerjee@cs.umn.edu}}
\author{Qilong Gu\textsuperscript{*}\textsuperscript{$\dagger$}}
\author{Yingxue Zhou\textsuperscript{*}\textsuperscript{$\dagger$}} 
\author{Tiancong Chen\textsuperscript{$\dagger$}}
\author{Arindam Banerjee\textsuperscript{$\dagger$}}
\affil{Department of Computer Science \& Engineering\\ 
University of Minnesota, Twin Cities\\ 
Minneapolis, MN 55455, USA}

\date{\vspace{-5mm}}

\begin{document}

\maketitle

\begin{abstract}
While stochastic gradient descent (SGD) and variants have been surprisingly successful for training deep nets, several aspects of the optimization dynamics and generalization are still not well understood. In this paper, we present new empirical observations and theoretical results on both the optimization dynamics and generalization behavior of SGD for deep nets based on the Hessian of the training loss and associated quantities. 
We consider three specific research questions:
(1) what is the relationship between the Hessian of the loss and the second moment of stochastic gradients (SGs)? (2) how can we characterize the stochastic optimization dynamics of SGD with fixed and adaptive step sizes and diagonal pre-conditioning based on the first and second moments of SGs? and (3) how can we characterize a scale-invariant generalization bound of deep nets based on the Hessian of the loss, which by itself is not scale invariant? We shed light on these three questions using theoretical results supported by extensive empirical observations,with experiments on synthetic data, MNIST, and CIFAR-10, with different batch sizes, and with different difficulty levels by synthetically adding random labels.
\end{abstract}
	
\section{Introduction}
\label{sec:intro}
\vspace{-3mm}

Stochastic gradient descent (SGD) and its variants have been surprisingly successful for training complex deep nets \cite{zhbh17, smle18, nebm17}. The surprise comes from two aspects: first, SGD is able to `solve' such non-convex optimization problems, and second, the solutions typically have good generalization performance. While numerous commercial, scientific, and societal applications of deep nets are being developed every day \cite{vamr17,mina18,dutt18}, understanding the optimization and generalization aspects of SGD for deep nets has emerged as a key research focus for the core machine learning community. 

In this paper, we present new empirical and theoretical results on both the optimization dynamics and generalization behavior of SGD for deep nets. For the empirical study, we view each run of SGD as a realization of a stochastic process in line with recent perspectives and advances \cite{chso18,mahb17,xiat18,jakb18}. We repeat each experiment, i.e., training a deep net on a data set from initialization to convergence, 10,000 times to get a full distributional characterization of the stochastic process, including the dynamics of the value, gradient, and Hessian of the loss.Thus, rather than presenting an average over 10 or 100 runs, we often show trajectories of different quantiles of the loss and associated quantities, giving a more complete sense of the empirical behavior of SGD.

We consider three key research questions in the paper. First, {\em how does the Hessian of the loss relate to the second moment of the stochastic gradients} (SGs)? In general, since the loss is non-convex, the Hessian will be indefinite with both positive and negative eigenvalues and the second moment of SGs, by definition, will be positive (semi-)definite (PSD). The subspace corresponding to the top (positive) eigenvalues of the second moment  broadly captures the preferred direction of the SGs. Does this primary subspace of the second moment overlap substantially with the subspace corresponding to the top positive eigenvalues of the loss Hessian? We study the dynamics of the relationship between these subspaces till convergence for a variety of problems (Section~\ref{sec:hessian}). Interestingly, the top
sub-spaces do indeed align quite persistently over iterations and for different problems. Thus, within the scope of these experiments, SGD seems to be picking up and using second order information of the loss.

Second, {\em how does the empirical dynamics of SGD look like as a stochastic process and how do we characterize such dynamics and convergence in theory?} We study the empirical dynamics of the loss, cosine of subsequent SGs, and $\ell_2$ norm of the SGs based on 10,000 runs to get a better understanding of the stochastic process \cite{lily18,zozq18,sizp19,arcg19,itos18,sijy18,yusj18,brgm18, lily18}. Under simple assumptions, we then present a distributional characterization of the loss dynamics as well as large deviation bounds for the change in loss at each step using a characterization based on a suitable martingale difference sequence. Special cases of the analysis for over-parameterized least squares and logistic regression provide helpful insights into the stochastic dynamics. We then illustrate that adaptive step sizes or adaptive diagonal preconditioning can be used to convert the dynamics into a super-martingale. Under suitable assumptions, we characterize such super-martingale convergence as well as rates of convergence of such adaptive step size or preconditioned SGD to a stationary point.

Third, {\em can we develop a scale-invariant generalization bound which considers the structure of the Hessian at minima obtained from SGD}? While the Hessian at minima contains helpful information such as local flatness, the Hessian changes if the deep net is reparameterized. We develop a PAC Bayesian bound  based on an anisotropic Gaussian posterior whose precision (inverse covariance) matrix is based on a suitably thresholded version of the Hessian. The posterior is intuitively meaningful, e.g., flat directions in the Hessian have large variance in the posterior. We show that while Hessian itself changes due to re-parameterization, the KL-divergence between the posterior and prior does not, yielding a scale invariant generalization bound. The bound revels a dependency and also trade-off between two scale-invariant terms: a measure of effective curvature and a weighted Frobenius norm of the change in parameters from initialization. Both terms remain unchanged even if the deep net is re-parameterized. We also show empirical results illustrating that both terms stay small for simple problems and they increase for hard learning problems. 

Our experiments explore the fully connected feed-forward network with Relu activations. We evaluate SGD dynamics on both synthetic datasets and some commonly used real datasets, viz., the MNIST database of handwritten digits~\cite{lecun_gradientbased_1998} and the CIFAR-10 dataset~\cite{krizhevsky_learning_2009}. The synthetic datasets, which are inspired by recent work in \cite{saeg17,zhbh17}, consist of equal number of samples drawn from $k$ isotropic Gaussians with different means, each corresponding to one class. We refer to these datasets as Gauss-$k$. We also consider variants of these datasets where a fixed fraction of points have been assigned random labels~\cite{zhbh17} without changing the features. Details of our experimental setup are discussed in the Appendix. Experiments on both Gauss-10 and Gauss-2 datasets are repeated 10,000 times for each setting to get a full distributional characterization of the loss stochastic process and associated quantities, including full eigen-spectrum of the Hessian of the loss and the second moment. In the main paper, we present results based on Gauss-10 and some on MNIST and CIFAR-10. Additional results on all datasets and all proofs are discussed in the Appendix.

The rest of the paper is organized as follows. We start with a brief discussion on related work in Section \ref{sec:related}. In Section \ref{sec:prelim}, we introduce the notations and preliminaries for the paper. In Section \ref{sec:hessian}, we investigate the relationship between the Hessian of the loss and the second moment of the SGs. In Section \ref{sec:dynamics}, we characterize the dynamics of SGD both empirically and in theory by treating SGD as a stochastic process, and reveal the influence of the dynamics of Hessian and second moment on such dynamics. In Section \ref{sec:bound}, we present a PAC-Bayesian based scale-invariant generalization bound which balances the effect of the local structure of the Hessian and the change in parameters from initialization. Finally, we conclude the paper in section \ref{sec:conc}.

\section{Related Work}
\label{sec:related}

{\bf Hessian Decomposition.} The relationship between the stochastic gradients and the Hessian of the loss has been studied by decomposing the Hessian into the covariance of stochastic gradients and the averaged Hessian of predictions \cite{xiat18,saeg17,sabl16}. \cite{saeg17,sabl16} found the eigen-spectrum of the Hessian after convergence are composed of a `bulk' which corresponds to large portion of zero and small eigenvalues and `outliers' which corresponds to large eigenvalues. Later on, \cite{papy18, papy19} found that the `outliers' of the Hessian spectrum come from the covariance of the stochastic gradients and the `bulk' comes from the averaged Hessian of predictions. \cite{gudd18,ghkx19} studied the dynamics of the gradients and the Hessian during training and found that large portion of the gradients lie in the top-$k$ eigenspace of the Hessian. \cite{ghkx19} found that using batch normalization suppresses the outliers of Hessian spectrum and the stochastic covariance spectrum. 

{\bf SGD dynamics and convergence.} To understand how SGD finds minima which generalizes well, various aspects of SGD dynamics have been studied recently. \cite{xiat18,jakb18} studied the SGD trajectory and deduced that SGD bounces off walls of a valley-like-structure  where the learning rate controls the height at which SGD bounces above the valley floor while batch size controls gradient stochasticity which facilitates exploration. \cite{chso18,jaka17,mahb17} studied the stationary distribution of SGD using characterizations based stochastic partial differential equations. In particular, \cite{mahb17} proposed that constant learning rate SGD approximates a  continuous time stochastic process (Ornstein-Uhlenbeck process) with a Gaussian stationary distribution. However, the assumption of constant covariance has been considered unsuitable for deep nets~\cite{xiat18, saeg17}. There have been work studying SGD convergence and local minima. \cite{itos18, sijy18} proved the probability of hitting a local minima in Relu neural network is quite high, and increases with the network width. The convergence of SGD for deep nets has been extensively studied and it has been observed that over-parameterization and proper random initialization can help the optimization in training neural networks \cite{silj18b, arcg19,zhls18,lily18, mabb17}. With over-parameterization and random  initialization, GD and SGD can find the global optimum in polynomial time for deep nets with Relu activation \cite{lily18,zozq18}  and residual connections \cite{silj18b}. Linear rate of SGD for optimizing  over-parameterized deep nets is observed in some special cases and assumptions \cite{sizp19,zhls18,arcg19}. However, \cite{sham18} showed that for linear deep nets, the number of iterations required for convergence scales exponentially with the depth of the network, which is opposite to the idea that increasing depth can speed up optimization \cite{rasc18}.

{\bf Generalization.} 
Traditional approaches attribute small generalization error either to properties of the model family or to the regularization techniques. However, these methods fail to explain why large neural networks generalize well in practice \cite{zhbh17}. Recently, several interpretations have been proposed \cite{nebm17,smle18}. The concept of generalization via achieving flat minima was first proposed in \cite{hosc97b}. Motivated by such idea, \cite{chcs16} proposed the Entropy-SGD algorithm which biases the parameters to wide valleys to guarantee generalization. \cite{kemn17} showed that small batch size can help SGD converge to flat minima. However, for deep nets with positively homogeneous activations, most measures of sharpness/flatness and norm are not invariant to rescaling of the network parameters, corresponding to the same function (``$\alpha$-scale transformation'' \cite{dipb17}). This means that the measure of flatness/sharpness can be arbitrarily changed through rescaling without changing the generalization performance, rendering the notion of `flatness' meaningless. To handle the sensitive to reparameterization, \cite{smle18} explained the generalization behavior through `Bayesian evidence', which penalizes sharp minima but is invariant to model reparameterization. In addition, some spectrally-normalized margin generalization bounds have proposed which depend on the product of the spectral norms of layers \cite{bapf17,nebb17}. \cite{vazk19} proposed deterministic margin bound based on suitable derandomization of PAC-Bayesian bounds in order to address the exponential dependence on the depth \cite{bapf17,nebb17}. Most recently, \cite{rank19, tsss19,yimc19} explored scale-invariant flatness measure for deep network minima. 
   
\section{Preliminaries}
\label{sec:prelim}

In this section, we set up notations and discuss preliminaries which will be used in the sequel. For simplicity, we denote the $j$-th entry of a vector $x$ as $x_j$.
Let $\cD$ be a fixed but unknown distribution over a sample space $\cZ$ and let
\beq
{S} = \{z_1,\ldots,z_n\} \sim \cD^n
\eeq
be a finite training set drawn independently from the true distribution $\cD$. 
For $k$-class classification problems, we have 
\beq
z_i = (x_i,y_i) \in \R^d \times \cY,
\label{eq:zi}
\eeq
where $x_i \in \R^d$ is a data sample, $y_i$ is the corresponding label, and $\cY = \{1, 2, \ldots, k\}$ is the set of labels. 
In this paper, we focus on empirical and theoretical analysis of SGD applied to deep nets such as feed forward neural networks, and convolutaional networks \cite{lebh15,gobc16}. With $\theta \in \R^p$ denoting the parameters of the deep net, the empirical loss on the training set is
\beq
f(\theta) \triangleq \frac{1}{n} \sum_{i=1}^n f(\theta; z_i)~,
\label{eq:tot_loss}
\eeq
for a suitable point-wise loss $f: \R^{p} \times \cZ \rightarrow \R_+$. The gradient of the empirical loss is 
\beq
\mu(\theta)  \triangleq \frac{1}{n} \sum_{i=1}^n \nabla f(\theta; z_i) = \nabla f(\theta)
\eeq
and the covariance of the sample gradient $\nabla f(\theta; z_i)$ is
\beq
\Sigma(\theta)  \triangleq \frac{1}{n} \sum_{i=1}^n (\nabla f(\theta; z_i) - \mu(\theta)) (\nabla f(\theta; z_i) - \mu(\theta))^T~.
\label{eq:cov}
\eeq
Further, let
\beq
H_f(\theta)  \triangleq  \frac{1}{n} \sum_{i=1}^n \nabla^2 f(\theta; z_i) = \nabla^2 f(\theta)
\label{eq:hess}
\eeq
be the Hessian of the empirical loss and 
\beq
M(\theta) \triangleq \frac{1}{n} \sum_{i=1}^n \nabla f(\theta; z_i) \nabla f(\theta; z_i)^T = \Sigma(\theta) + \mu(\theta) \mu(\theta)^T
\label{eq:m2}
\eeq
be the second moment of sample gradient $\nabla f(\theta; z_i)$. Note that these quantities are all defined in terms of the training sample $S$. 

At each step, SGD performs the following update \cite{romo51,nejl09}: 
\beq
    \theta_{t+1} = \theta_t - \eta_t \nabla \tilde{f}(\theta_t)~,
\label{eq:sgd}
\eeq
where $\eta_t$ is the step size and $\tilde{f}(\theta_t)$ is the stochastic gradient (SG) typically computed from
a mini-batch of samples. Let $m$ be the batch size, so that $1 \leq m \leq n$. We denote the mean and covariance of SG as $\mu^{(m)}(\theta_t)$ and  $\Sigma^{(m)}(\theta_t)$ respectively, and we have $\mu^{(m)}(\theta_t) = \mu(\theta_t)$ and $\Sigma^{(m)}(\theta_t) = \frac{1}{m}\Sigma(\theta_t)$. For convenience, we introduce the following notation:
\beq
    \mu_t = \mu(\theta_t)~, \qquad \Sigma_t = \Sigma(\theta_t)~, \qquad M_t = M(\theta_t)~.
    \label{eq:brief}
\eeq

Let $\phi(\theta; x_i) : \R^p \times \R^d \rightarrow \R^k$ be the output of the deep net for a $k$-class classification problem, then the prediction probability of true label $y_i$ is given by:
\beq
p(y_i |x_i, \theta) = \frac{\exp( \phi_{y_i}(\theta; x_i))}{\sum_{j= 1}^k \exp( \phi_j(\theta; x_i))}~,
\eeq
where $\phi_j(\theta; x_i)$ is the $j$-th entry of $\phi(\theta; x_i)$. In this paper, we consider the log-loss, also known as the cross-entropy loss, given by
\beq
    f(\theta; z_i) = - \log p(y_i| x_i, \theta )~.
    \label{loss:log}
\eeq

\section{Hessian of the Loss and Second Moment of SGD}
\label{sec:hessian}
\begin{figure}[t!]
\vspace{-5mm}
\centering
 \subfigure[Gauss-10, batch size: 5, 0\% random labels.]{
 \includegraphics[width = 0.97\textwidth]{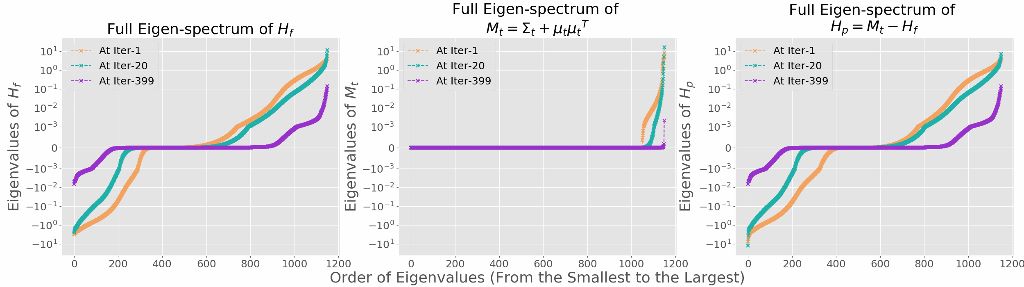}
 }
 \subfigure[Gauss-10, batch size: 5, 15\% random labels.]{
 \includegraphics[width = 0.97 \textwidth]{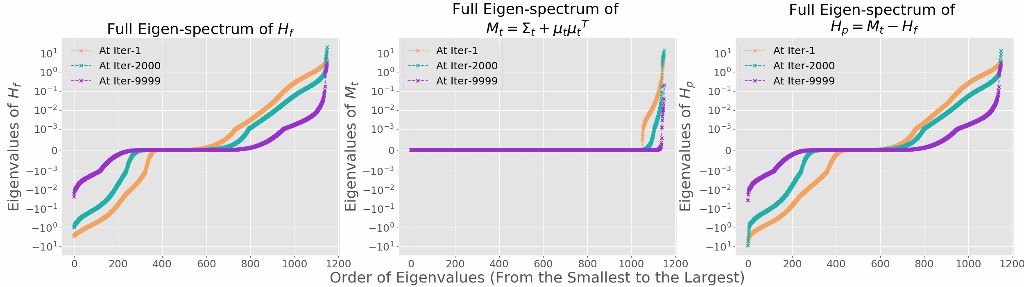}
 }
\caption[]{Eigen-spectrum dynamics of $H_f(\theta_t)$ (left),  $M_t$ (middle), and $H_p(\theta_t)$ (right). The network is trained on Gauss-10 dataset with small batches containing one twentieth of the training samples (5/100). $H_p$ remains significant even after SGD converges, and is close to $-H_f(\theta_t)$.}
 \label{fig:full_spectrum}
\end{figure}

In this section, we investigate the relationship between the Hessian of the training loss $H_f(\theta_t)$ and the second moment of SGs $M_t$.
We compute and compare the eigenvalues and eigenvectors of both $H_f(\theta_t)$ and $M_t$. Our experimental results show that the primary subspaces of $H_f(\theta_t)$ and $M_t$ overlap while the eigenvalue distributions (eigen-spectra) of the two matrices have substantial differences. We also illustrate that the overlap of the primary subspaces cannot be quite explained based the Fisher information matrix. 
Different from  \cite{gudd18,ghkx19}, our work not only focuses on the full eigenvalue distribution at the beginning and the end of SGD, but also reveals how such distribution evolves during the training by providing additional result at intermediate iteration. Comparing with \cite{gudd18}, we use a more well-established metric to characterize the overlap between two subspaces.

\begin{figure}[t!]
\centering
 \subfigure[Gauss-10, batch size: 5, 0\% random labels.]{
 \includegraphics[width = 0.97 \textwidth]{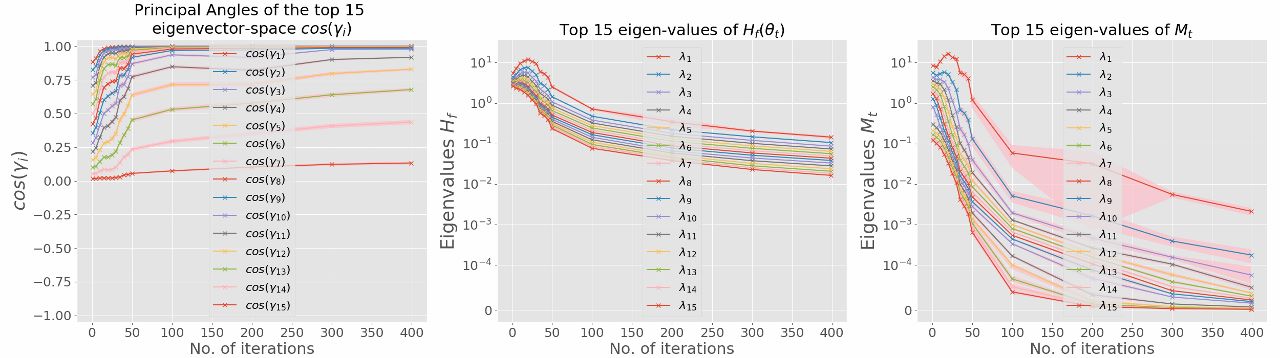}}
 \vspace{-1mm}
 \subfigure[Gauss-10, batch size: 5, 15\% random labels.]{
 \includegraphics[width = 0.97 \textwidth]{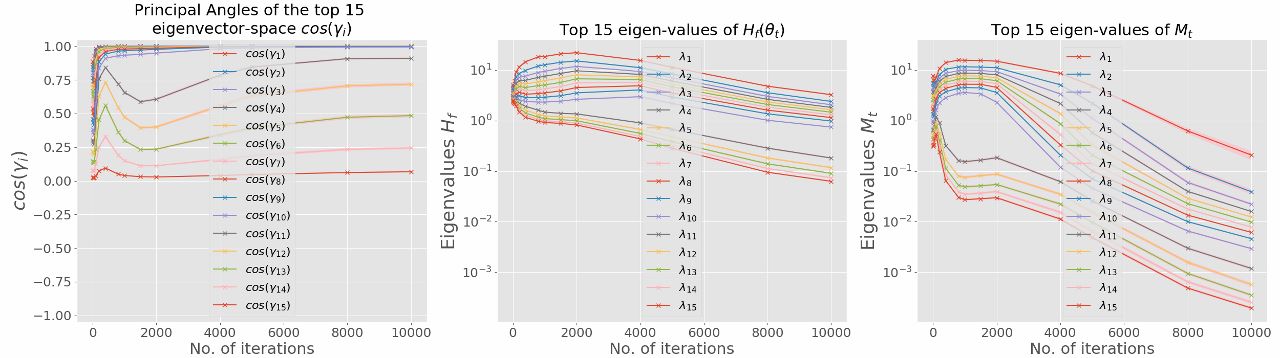}}
 \caption[]{Dynamics of top 15 principal angles between $H_f(\theta_t)$ and $M_t$ (left) and corresponding eigenvalues of $H_f(\theta_t)$ (middle) and $M_t$ (right) for Gauss-10. Solid lines represent the mean value over all runs, and shaded bands indicate the 95\% confidence intervals. $\cos(\gamma_1)$ to $\cos(\gamma_{10}) \approxeq 1$, indicating the top 10 principal subspaces are well aligned.
 }
 \label{fig:principal_angle_gauss10}
\end{figure}

\subsection{ Hessian Decomposition} 
For the empirical log-loss based on \eqref{eq:tot_loss} and \eqref{loss:log}, the Hessian $H_f(\theta_t)$ of the loss and the second moment $M_t$ of the stochastic gradients are related as follows~\cite{saeg17,xiat18,mart14,jaka17}:

\begin{restatable}{prop}{hessdecomp}
\label{prop:hessian}
For $H_f(\theta_t)$ and $M_t$ as defined in \eqref{eq:hess} and $\eqref{eq:brief}$, we have 
\beq
H_f(\theta_t) = M_t - H_p(\theta_t)~,
\label{eq:emphess}
\eeq
where
$H_p(\theta_t) =  \frac{1}{n} \sum_{i=1}^n \frac{1}{p(\theta_t;z_i)} \frac{\partial^2 p(\theta_t;z_i)}{\partial \theta^2}$.
\end{restatable}

\begin{figure}[t!]
\vspace{-4mm}
\centering
\subfigure[MNIST, No. hidden layers:~3, batch size: 64.]{
\includegraphics[width = 0.97 \textwidth]{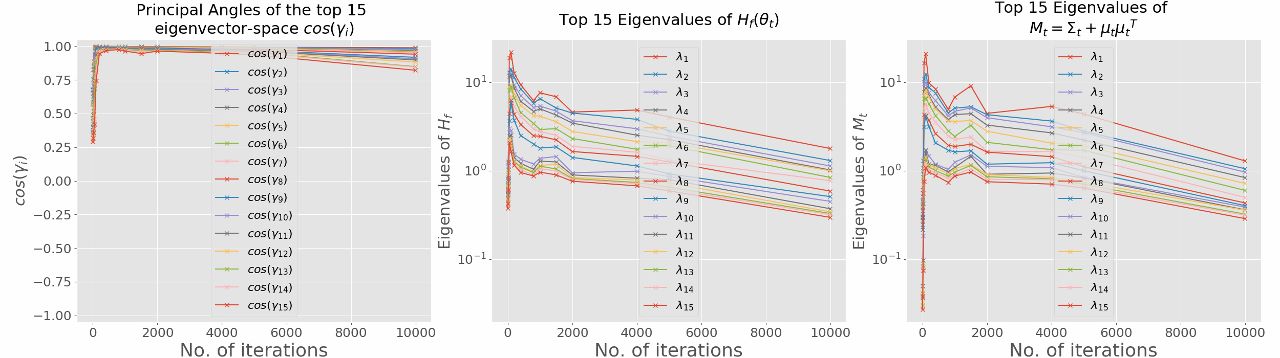}
 }\vspace{-1mm}
\subfigure[MNIST, No. hidden layers:~3, batch size: 256.]{
\includegraphics[width = 0.97 \textwidth]{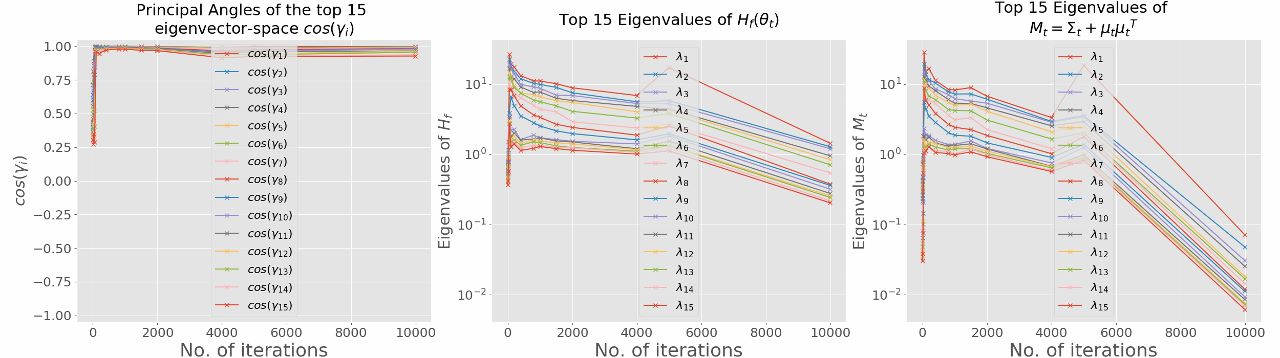}
 }\vspace{-1mm}
 \caption[]{Dynamics of top 15 principal angles between $H_f(\theta_t)$ and $M_t$ (left) and corresponding eigenvalues of $H_f(\theta_t)$ (middle) and $M_t$ (right) for MNIST. 
 $\cos(\gamma_1)$ to $\cos(\gamma_{10}) \approxeq 1$, indicating the top 10 principal subspaces are well aligned.} 
 \label{fig:principal_angle_mnist}
\end{figure}

Figure~\ref{fig:full_spectrum} shows the full eigen-spectrum (averaged over 10,000 runs) of $H_f(\theta_t)$, $M_t$, and the residual term $H_p(\theta_t)$ at the first, one intermediate, and the last iteration of SGD trained on Gauss-10. Overall, the eigenvalue dynamics of $H_f(\theta_t)$ and $M_t$ exhibit similar trend (Figure ~\ref{fig:full_spectrum},~\ref{fig:principal_angle_gauss10}: middle and right) with slight increase at early iterations, followed by a decrease. Eigenvalues of $M_t$ usually drop faster than those of  $H_f(\theta_t)$ (Figure~\ref{fig:principal_angle_gauss10}: middle and right). In general, when the training data has $k$ classes, $M_t$ is a positive definite matrix with an order of $k$ non-trivial eigenvalues \cite{saeg17}. 

For simple problems in which data are easy to separate, as SGD converges, we would expect the average gradient $\mu_t$ as well as the gradient $\nabla f(\theta_t; z_i)$ for each individual sample to be close to 0 \cite{silj18b,mabb17}. Thus $M_t$ approaches 0 as almost all non-trivial eigenvalues vanish (Figure~\ref{fig:full_spectrum} (a): middle). As a result, the residual $H_p(\theta_t)$ approaches $-H_f(\theta_t)$ (Figure~\ref{fig:full_spectrum} (a): left and right). 

When dealing with a hard problem (Figure~\ref{fig:full_spectrum}(b)), at convergence, even though the average gradient $\mu_t$ may be close to 0, certain individual gradients $\nabla f(\theta_t; z_i)$ need not. Therefore, $M_t$ may still have a handful of non-trivial eigenvalues, but they are at least one order of magnitude smaller than the top eigenvalues of $H_f(\theta_t)$ (Figure~\ref{fig:full_spectrum}(b): left and middle). Such an observation is not only true for Guass-10 with 15\% random labels, but also observed on MNIST (Figure~\ref{fig:principal_angle_mnist}: middle and right) and CIFAR-10 dataset (Figure~\ref{fig:principal_angle_cifar10}: middle and right). Hence, the top eigenvalues of $H_p(\theta_t)$ are much closer to $H_f(\theta_t)$ than those of $M_t$. The bottom eigenvalues of $H_p(\theta_t)$ always approaches $-H_f(\theta_t)$ since $M_t$ only has non-negative eigenvalues. Overall, the eigen-spetrum of $H_p(\theta_t)$ looks very similar to $-H_f(\theta_t)$ close to convergence. Our observations disagree with existing claims~\cite{xiat18} which suggest that $H_f(\theta_t)$ is almost equal to $M_t$ near the minima by assuming the residual term $H_p(\theta_t)$ disappears.

\begin{figure}[t!]
\vspace{-4mm}
\centering
\subfigure[CIFAR-10, No. hidden layers:~3, batch size: 256.]{
\includegraphics[width = 0.97 \textwidth]{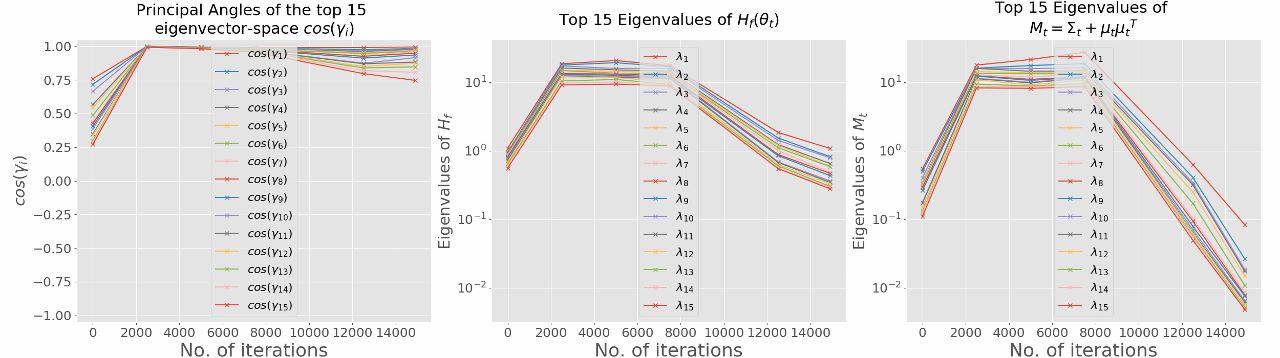}
 }
\subfigure[CIFAR-10, No. hidden layers:~3, batch size: 512.]{
\includegraphics[width = 0.97 \textwidth]{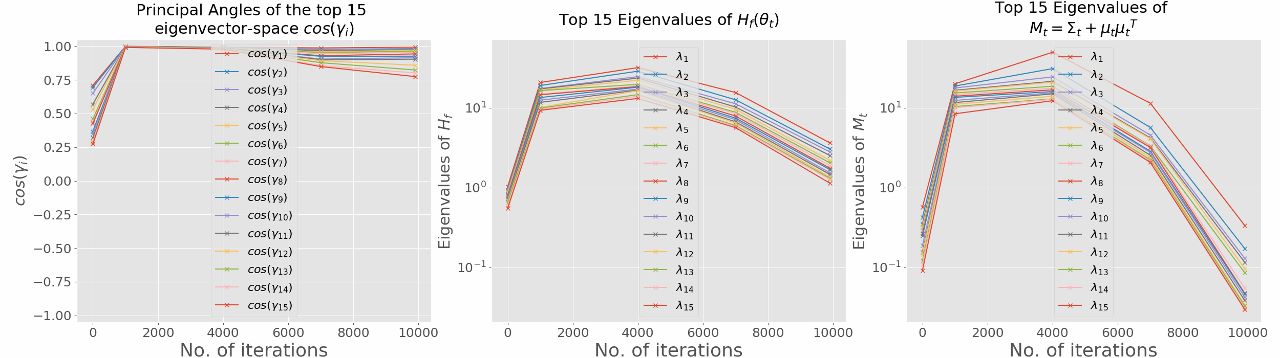}
 }
 \caption[]{Dynamics of top 15 principal angles between $H_f(\theta_t)$ and $M_t$ (left) and corresponding eigenvalues of $H_f(\theta_t)$ (middle) and $M_t$ (right) for CIFAR-10. $\cos(\gamma_1)$ to $\cos(\gamma_{10}) \approxeq 1$, indicating the top 10 principal subspaces are well aligned.} 
 \label{fig:principal_angle_cifar10}
\end{figure}

\subsection{Top Subspaces: Hessian and Second Moment} Proposition \ref{prop:hessian} indicates that, at any time during training, $H_f(\theta_t)$ and $M_t$ are related but differ by a residual term $H_p(\theta_t)$. Our empirical results show that the impact of $H_p(\theta_t)$ on the eigen-spectrum dynamics of $M_t$ and $H_f(\theta_t)$ is persistent and not negligible. But {\em how does $H_p(\theta_t)$ impact the primary subspaces}, i.e., corresponding to the largest positive eigenvalues, of $M_t$ and $H_f(\theta_t)$?

To answer this question, we carefully assess the overlap between the principal eigen-spaces of $H_f(\theta_t)$ and $M_t$ determined by the eigenvectors corresponding to the top (positive) eigenvalues of these matrices during training based on {\em principal angles}; additional results based on Davis-Kahan \cite{daka70,yuws15} perturbation can be found in the appendix.
Recall that~\textit{principal angles} \cite{golo13} $\{\gamma_i\}_{i=1}^q$ between two subspaces $P$ and $Q$ in $\R^p$, whose dimensions satisfy $p \geq \dim(P) \geq \dim(Q) = q \geq 1$, are defined recursively by 
\beq
\cos(\gamma_i)  \triangleq u_i^T v_i = \max_{\substack{u \in P, \|u\|_2 = 1,\\ u^T[u_1,\ldots,u_{i-1}] = 0}}~~ \max_{\substack{v \in Q, \|v\|_2 = 1,\\ v^T[v_1,\ldots,v_{i-1}] = 0}} ~~u^Tv~.
\label{eq:principal_angle}
\eeq

Let $U$ and $V$ be two orthogonal matrices, e.g., eigenvectors of $H_f(\theta_t)$ and $M_t$, whose range are $P$ and $Q$ respectively, then we have $\cos(\gamma_i) = \omega_i$, where $\omega_i$ denotes the singular values of $U^TV$. Such relationships are also considered in Canonical Correlation Analysis (CCA) \cite{golo13}.

Figures~\ref{fig:principal_angle_gauss10}-\ref{fig:principal_angle_cifar10} show the evolution of the subspace overlap between $H_f(\theta_t)$ and $M_t$ in terms of the top-15 principal angles for Gauss-10, MNIST, and CIFAR-10 datasets. All datasets have 10 classes. The key observation is the top 10 principal subspaces of $H_f(\theta_t)$ and $M_t$ quickly align with each other and overlap almost completely during the entire training period. Additional results for Gauss-2 dataset whose top 2 principal subspaces overlap can be found in the appendix (Figure \ref{fig:principal_angle_k2}). Such persistent overlap occurs in both synthetic and real datasets, suggesting that the second moment of the SGs somehow carry second order information about the loss. 

Notice that in some scenarios, e.g., MNIST dataset (Figure~\ref{fig:principal_angle_mnist} (b)), the cosine value of all 15 principal angles stays high. However, comparing with its top 10 eigenvalues, the remaining eigenvalues of $M_t$ have much smaller values (Figure~\ref{fig:principal_angle_mnist} (b):right). Thus, the overlap in these subspaces will not significantly affect the behavior of SGD.

\begin{table}
\caption{Summary of the probabilistic model $p(y|x, \theta)$, the Hessian of the empirical loss $H_f(\theta_t)$, the second moment $M_t$ and the residual term $H_p(\theta_t)$ for least squares (second column) and binary logistic regression (last column). Here $\sigma_{\theta_t}(x_i) = p_{\theta_t}(x_i)(1 - p_{\theta_t}(x_i))$. In high dimensional case, when $\theta_t$ approaches the optimum, $H_p(\theta_t)$ does not approach zero for least squares, and $H_p(\theta_t)$ has the same order as $H_f(\theta_t)$ for logistic regression. }
\label{tab:ex:fisher}
    \centering
    \begin{tabular}{|c|c|c|}
        \hline
         & Least Squares & Binary Logistic Regression \\
        \hline
        \rule{0pt}{4ex} 
        $p(y|x, \theta)$ \rule[-2ex]{0pt}{0pt}  & $\frac{1}{\sqrt{2 \pi} \sigma} e^{- \frac{(x^T
		\theta - y)^2}{2 \sigma^2}}$ & $\frac{1}{(1 + e^{- x_i^T \theta})^{y_i}} \cdot \frac{1}{(1 + e^{ x_i^T \theta})^{1 - y_i}} $ \rule[-2ex]{0pt}{0pt}  \\
		\hline
		\rule{0pt}{4ex} 
		$H_f(\theta_t)$ & $\displaystyle \frac{1}{n \sigma^2} X^T X$ \rule[-2ex]{0pt}{0pt} & $\frac{1}{n} \sum_{i=1}^n \sigma_{\theta_t}(x_i) x_i x_i^T$ \\
		\hline
		\rule{0pt}{4ex} 
		$M_t$ & $\frac{1}{n \sigma^4} \sum_{i=1}^{n} (x_i^T \theta_t - y_i)^2 x_i x_i^T$ \rule[-2ex]{0pt}{0pt} & $\frac{1}{n} \sum_{i=1}^{n} (p_{\theta_t}(x_i) - y_i)^2 x_i x_i^T$ \\
        \hline
        \rule{0pt}{4ex} 
        $H_p(\theta_t)$ & $\frac{1}{n \sigma^4} \sum_{i=1}^{n} [(x_i^T \theta_t - y_i)^2 - \sigma^2] x_i x_i^T$ \rule[-2ex]{0pt}{0pt} & $\frac{1}{n} \sum_{i=1}^{n} [(p_{\theta_t}(x_i) - y_i)^2 - \sigma_{\theta_t}(x_i)] x_i x_i^T$ \\
        \hline
    \end{tabular}
\end{table}

\subsection{Relationship with Fisher Information} 
The observation that the primary subspaces of $M_t$ indeed overlap with the primary suspaces $H_f(\theta_t)$ is somewhat surprising. One potential explanation can be based on connecting the decomposition in Proposition~\ref{prop:hessian} with the Fisher Information matrix.

Denoting $\theta^*$ as the true parameter for the generative model $p(\theta^*; Z)$ and with $f(\theta^*; Z) = -\log p(\theta^*; Z)$, the Fisher Information matrix \cite{leca98,rao45} is defined as 
\begin{equation}
I(\theta^*) = \E_Z [ M(\theta^*) ] = \E_Z \left[ \nabla f (\theta^*; Z) \nabla f(\theta^*; Z)^T \right]~,
\end{equation}
where $\nabla f(\theta^*; Z)$ is often referred to as the score function.
Under so-called regularity conditions \cite{coth06,amna00} (see Appendix \ref{sec:app:fisher} for more details), 
 the Fisher Information can also be written as 
\begin{equation}
    I(\theta^*) =  \E_Z [ H_f(\theta^*) ] = \E_Z \left[ \nabla^2 f(\theta^*; Z) \right]~.
    \label{eq:fish1}
\end{equation}
In particular, using integration by parts and recalling that $f(\theta^*; Z) = - \log p(\theta^*; Z)$, we in fact have
\begin{equation}
\E_Z \left[ \nabla^2 f(\theta^*; Z) \right] = \E_Z \left[ \nabla f (\theta^*; Z) \nabla f(\theta^*; Z)^T \right] + \E_Z \left[ \frac{1}{p(\theta^*;Z)} \nabla^2 p(\theta^*; Z) \right]~,
\label{eq:fish2}
\end{equation}
which is line with Proposition~\ref{prop:hessian}. However, under the regularity conditions, we have (see Appendix \ref{sec:app:fisher}) $\E_Z \left[ \frac{1}{p(\theta^*;Z)} \nabla^2 p(\theta^*; Z) \right] =0$, which makes the two forms of $I(\theta^*)$ equal. However, for finite samples, the expectation $\E_Z$ is replaced by $\frac{1}{n} \sum_{i=1}^n$, so the $\frac{1}{p(\theta^*;Z)}$ term does not cancel out, and the finite sample $H_p(\theta)$ in Proposition~\ref{prop:hessian} does not become zero. In addition, $\theta_t$ during the SGD iterations are not the true parameter $\theta^*$, so the quantities involved in Proposition~\ref{prop:hessian} are not quite the finite sample versions of Fisher Information due to model misspecification. 

Our empirical results show that in the context of Proposition~\ref{prop:hessian} the quantities corresponding to \eqref{eq:fish1} and \eqref{eq:fish2} are not the same possibly due to the finite sample over-parameterized setting, inaccurate estimate of $\theta^*$, and the non-smoothness of the Relu activation.

Especially when the model is over-parameterized, even for smooth loss functions, we may still observe $H_f(\theta)$ to be different from $M_t$, e.g., see the analysis on over-parameterized least squares (Example \ref{ex:ls}) and over-parameterized logistic regression (Example~\ref{ex:lr}) below. Thus, Fisher information alone is not sufficient to explain the overlap between the primary subspaces of $H_f(\theta_t)$ and $M_t$. 

\begin{restatable}[Least Squares]{exmp}{exmpi}
\label{ex:ls}
Let $X = [x_1, \ldots, x_n]^T \in \R^{n \times p}$ be the design matrix and $Y = [y_1, \ldots, y_n]^T \in \R^{n}$ be the response vector for $n$ training samples. Given a sample $z_i=(x_i,y_i)$ defined in \eqref{eq:zi}, we assume the following linear relationship holds: $y_i = x_i^T\theta + \epsilon_i$, where $\epsilon_i \sim \mathcal{N}(0,\sigma^2)$ is a Gaussian noise with mean $0$ and variance $\sigma^2$. The empirical loss of the least squares problem is given by 
\begin{equation}
{f}(\theta) = \frac{1}{2n\sigma^2} \sum_{i=1}^{n} (x_i^T\theta - y_i  )^2~.  
\end{equation}
The Hessian of the empirical loss ${H}_f(\theta_t)$, the second moment ${M}_t$, and the residual term ${H}_p(\theta_t)$ can be directly calculated (see Appendix \ref{subsec:ex1}) and has been summarized in Table~\ref{tab:ex:fisher} (second column). In high dimensional case when $n<p$, the optimal solution $\hat{\theta}$ satisfies $X\hat{\theta} = Y$. As $\theta_t$ approaches $\hat{\theta}$, we have
\beq
	H_p(\hat{\theta}) = - \frac{1}{n \sigma^2} X^T X,\qquad \text{and} \qquad M(\hat{\theta})=0~,
\eeq
so that $H_f(\hat{\theta}) = - H_p(\hat{\theta})$ and $H_f(\hat{\theta}) \neq M(\hat{\theta})$.
\end{restatable}

\begin{restatable}[Logistic Regression]{exmp}{exmpii}
\label{ex:lr}
The empirical loss of binary logistic regression for $n$ observations $(x_i,y_i), i=1,\ldots,n$ and $y_i \in \{0,1\}$ is given by
\beq
    {f}(\theta) = -\frac{1}{n} \sum_{i=1}^n \log \left( p_{\theta}(x_i)^{y_i}(1 - p_{\theta}(x_i))^{(1 - y_i)} \right)~,\qquad \text{where} \qquad p_{\theta}(x_i) = \frac{1}{1 + e^{- x_i^T \theta}}~.
    \label{loss:lr}
\eeq
The Hessian of the empirical loss ${H}_f(\theta_t)$, the second moment ${M}_t$, and the residual term ${H}_p(\theta_t)$ can be calculated directly (see Appendix \ref{subsec:ex2}) and has been summarized in Table~\ref{tab:ex:fisher} (last column). In the high dimensional case, the data are always linearly separable. Thus $\sum_{i=1}^n (p_{\theta_t}(x_i)-y_i)^2$ can be arbitrarily small, depending on $\| x_i \|_2$. When $\theta_t$ approaches the optimum, we have $\sigma_{\theta_t}(x_i) \gg (p_{\theta_t}(x_i)-y_i)^2 \approx 0$, therefore

\beq
	-H_p(\theta_t)  \approx H_f(\theta_t) = \frac{1}{n} \sum_{i=1}^n \sigma_{\theta_t}(x_i) x_i x_i^T \succ \frac{1}{n} \sum_{i=1}^n (p_{\theta_t}(x_i)-y_i)^2 x_i x_i^T = M_t \approx 0~,
\eeq

so that $-{H}_p(\theta_t)$ approaches ${H}_f(\theta_t)$ and $M_t \nrightarrow H_f(\theta_t) $ as $t \rightarrow \infty$.  
\end{restatable}

\begin{figure}[t!]
\vspace{-5mm}
\centering
\includegraphics[width = 0.5 \textwidth]{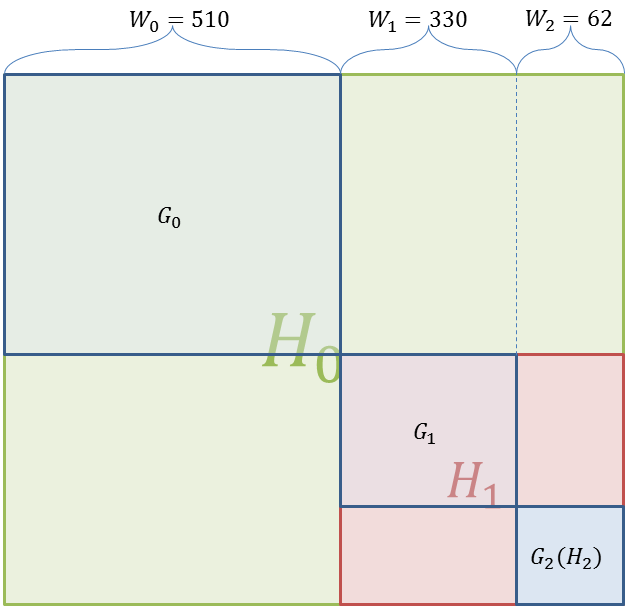}
\caption[]{Notations for layer-wise Hessian analysis. $G_h,h=0,1,2$: the diagonal blocks of $H_f(\theta_t)$, where the partial derivatives are taken with respect to the weights in the same layer. $H_h,h=0,1,2$: the Hessian of the sub-network starting from the layer $h$. $H_0$ is the full Hessian $H_f(\theta_t)$, and $H_2$ is the same as $G_2$.}
 \label{fig:layerwise_block_def}
\end{figure}

\begin{figure}[p!]
\centering
 \subfigure[Gauss-10, batch size 5, 0\% Random labels.]{
 \includegraphics[width = 0.48 \textwidth]{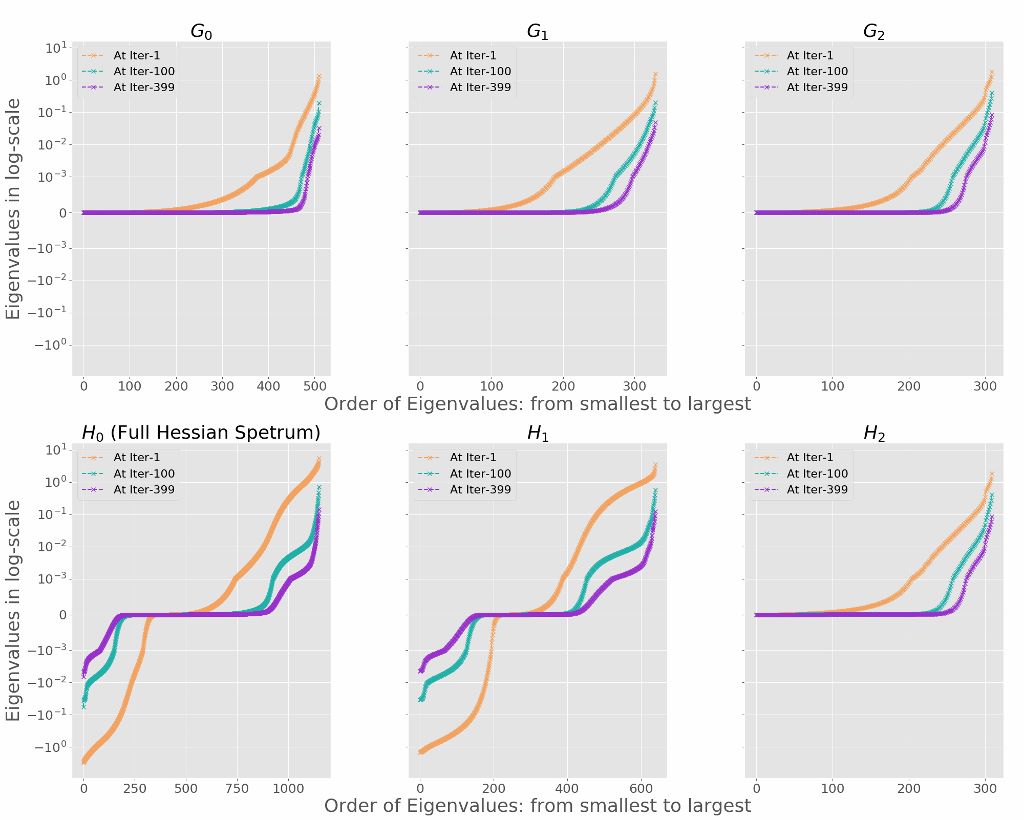}
 }
 \subfigure[Gauss-10, batch size 5, 15\% Random labels.]{
 \includegraphics[width = 0.48 \textwidth]{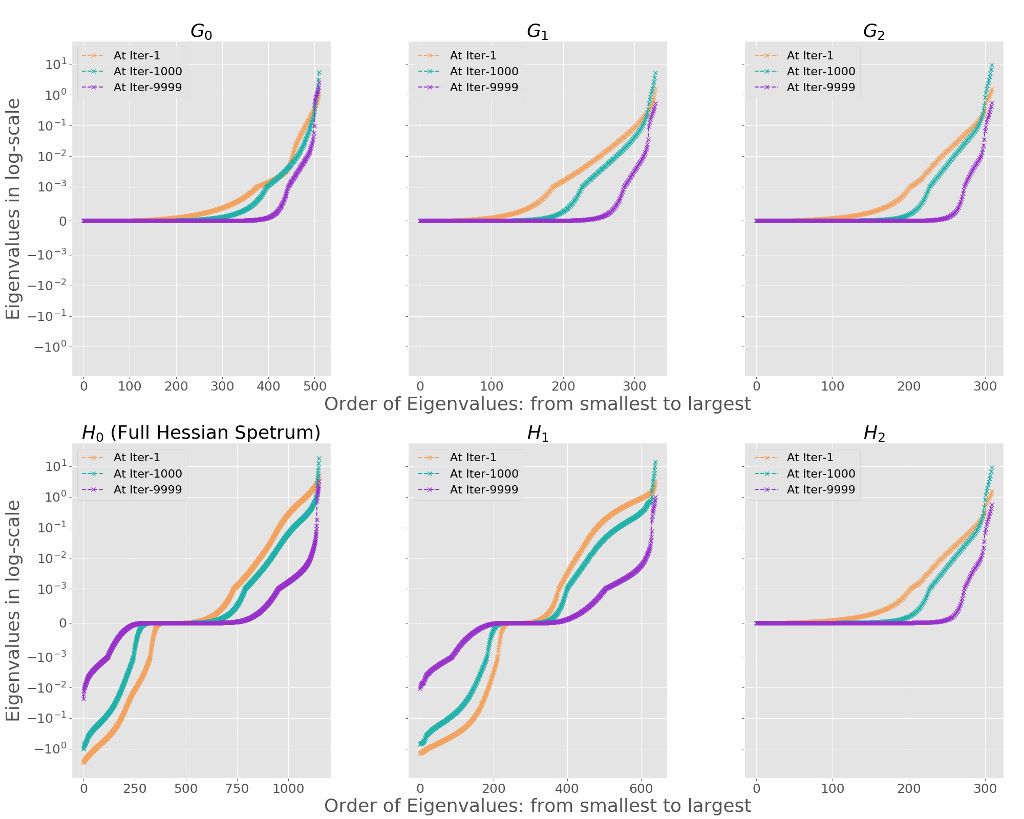}
 }
\subfigure[Gauss-10, batch size 50, 0\% Random labels.]{
 \includegraphics[width = 0.48 \textwidth]{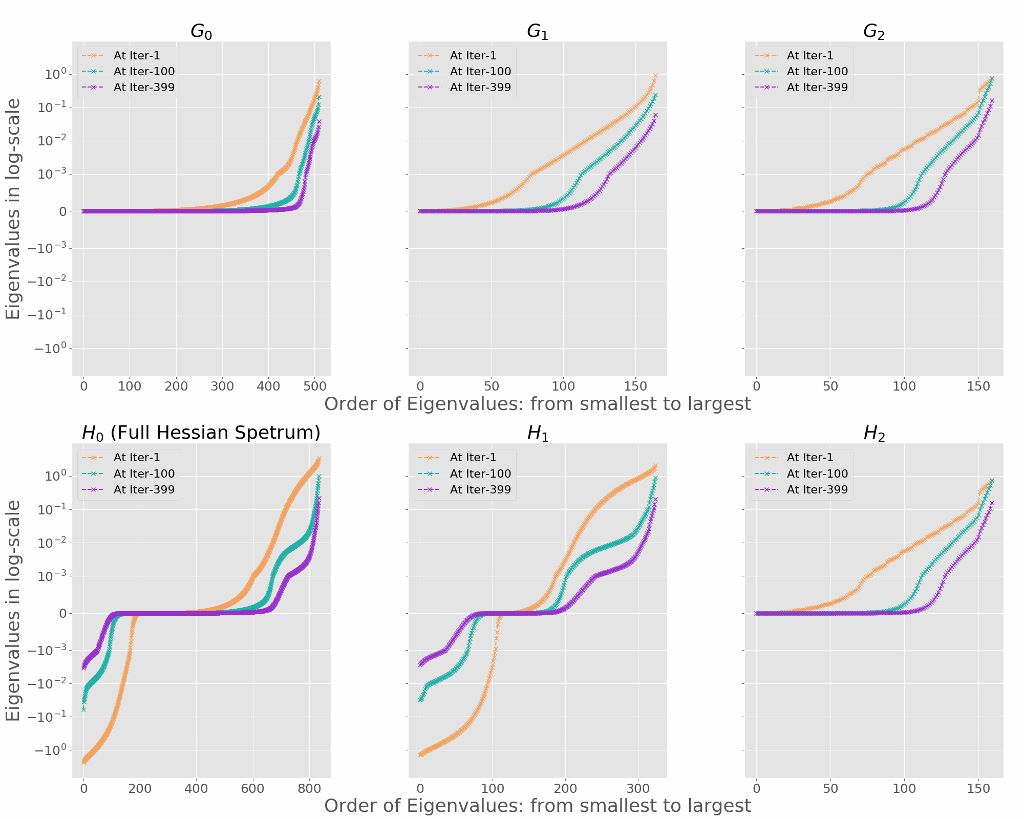}
 }
 \subfigure[Gauss-10, batch size 50, 15\% Random labels.]{
 \includegraphics[width = 0.48 \textwidth]{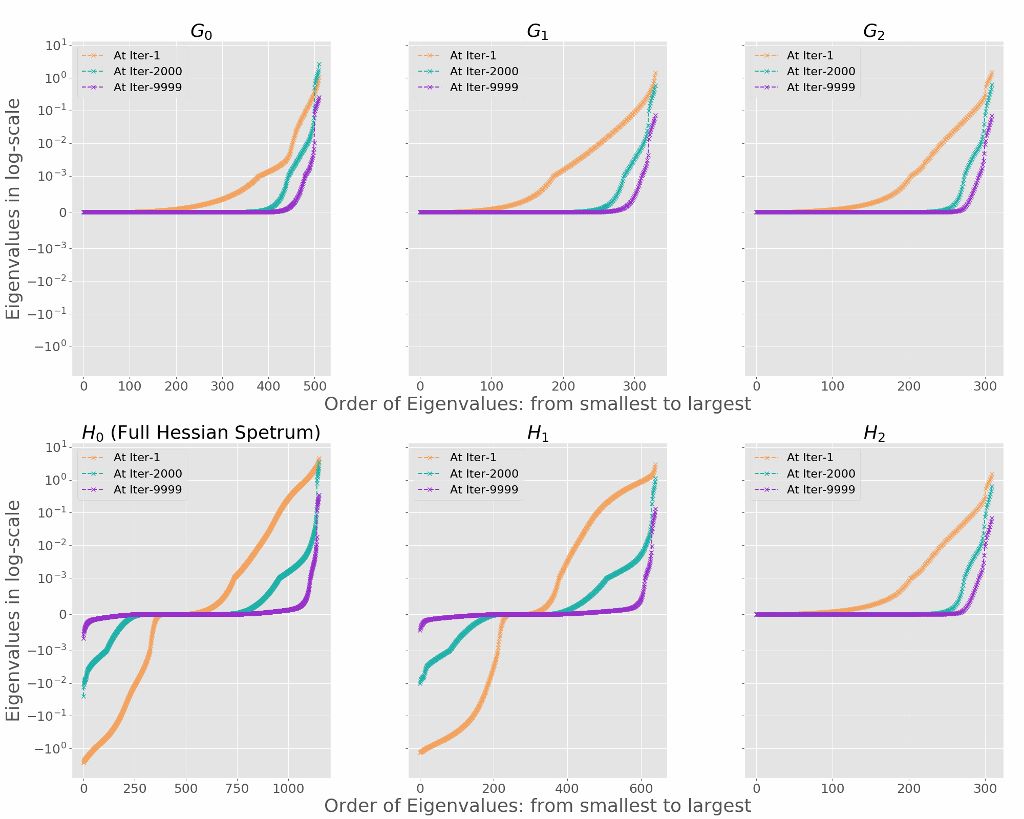}
 }
 \caption[]{Eigen-spectrum dynamics of $H_h$ and $G_h, h=0,1,2$ for networks trained on Gauss-10 dataset. (a) and (b): small batches containing one twentieth of the training samples (5/100); (c) and (d): large batches containing half of the training samples (50/100). All $G_h$s are positive semi-definite matrices whose top eigenvalues have the same order of magnitude, indicating that the top few large eigenvalues of $H_f(\theta_t)$ can not come from the output layer alone.}
 \label{fig:layerwise_hessian_k10}
\end{figure}

\begin{figure}[p!]
\vspace{-5mm}
\centering
 \subfigure[Gauss-10, batch size 5, 0\% Random labels.]{
 \includegraphics[width = 0.9 \textwidth]{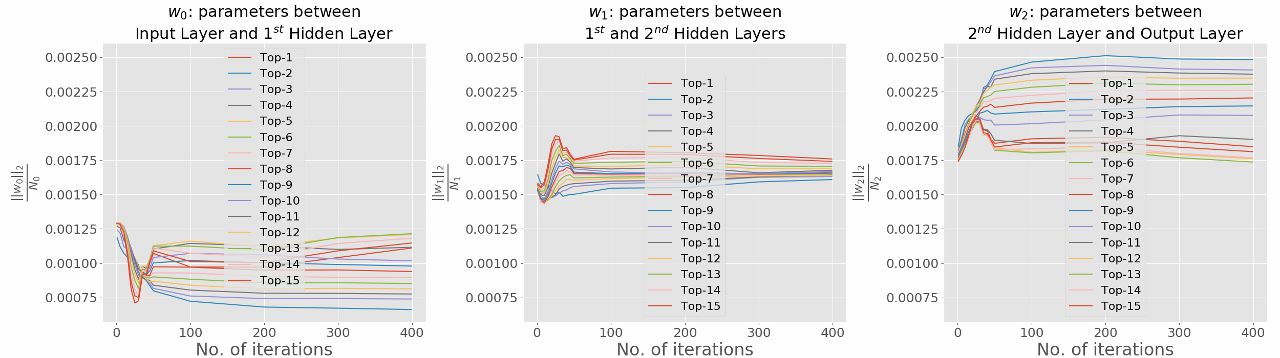}
 }
 \subfigure[Gauss-10, batch size 5, 15\% Random labels.]{
 \includegraphics[width = 0.9 \textwidth]{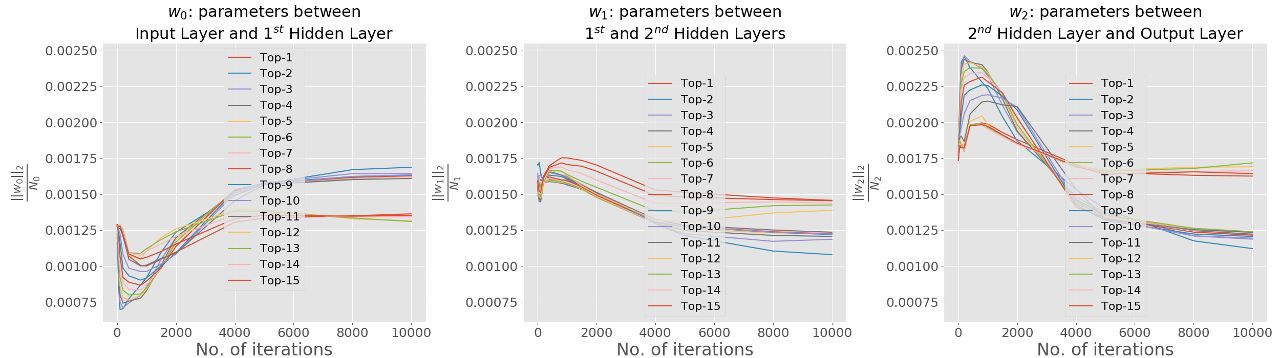}
 }
 \subfigure[Gauss-10, batch size 50, 0\% Random labels.]{
 \includegraphics[width = 0.9 \textwidth]{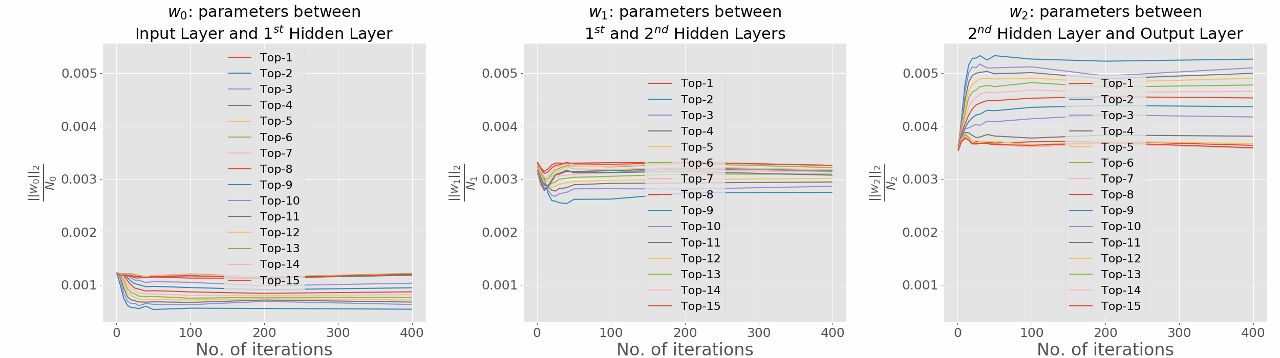}
 }
 \subfigure[Gauss-10, batch size 50, 15\% Random labels.]{
 \includegraphics[width = 0.9 \textwidth]{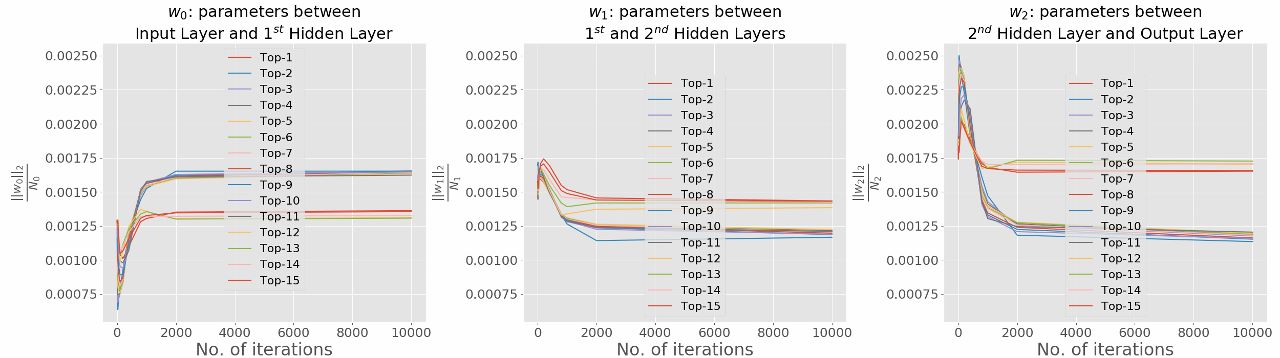}
 }
 \caption[]{Layer-wise eigenvector loadings for networks trained
on Gauss-10 dataset. (a) and (b): small batches containing one twentieth of the training samples (5/100); (c) and (d): large batches containing half of the training samples (50/100). Even though for simple problem, it seems that the 
output layer always has the loading while layer 0 contributes less, such relationship does not always hold for hard problem.}
 \label{fig:layerwise_loadings}
\end{figure}

\subsection{Layer-wise Hessian} 
We also provide a layer-wise analysis of the curvature as obtained from the Hessian. The Hessian $H_f(\theta_t)$ of a 2-hidden layer Relu network can be thought of as a collection of several block matrices (Figure~\ref{fig:layerwise_block_def}). We use $h = 0$ to denote the lowest layer which connects to the input and $h = 2$ corresponds to the output layer. Let $G_h,h=0,1,2$ denote the diagonal blocks of $H_f(\theta_t)$, such that each element in $G_h$ is the partial derivative taken with respect to the weights in the same layer, 
\beq
    G_h[kl,k'l']=\frac{\partial^2 f(\theta)}{\partial w_{kl}[h] \partial w_{k'l'}[h]} = \left(\frac{\partial\phi}{\partial w_{kl}[h]} \right)^T\nabla^2_{\phi} f(\theta) \frac{\partial\phi}{\partial w_{k'l'}[h]}
    \label{hess:layer}
\eeq

where $\phi \in \R^k$ is the output of the network defined in Section \ref{sec:prelim}, $w_{kl}[h]$ is the weight at the layer $h$ connecting the node $k$ from the previous layer and node $l$ at the layer $h$, and $H_h,h=0,1,2$ is the Hessian of the sub-network starting from the layer $h$, with $H_0$ being the full Hessian $H_f(\theta_t)$, and $H_2$ being the same as $G_2$. From \eqref{hess:layer}, every matrix $G_h$ is positive semi-definite (PSD) since $\nabla^2_\phi(\theta)$, the Hessian of the logistic loss, is PSD. The definitions of $G_h$ and $H_h$ has been depicted in Figure~\ref{fig:layerwise_block_def}. Figure~\ref{fig:layerwise_hessian_k10} shows the eigen-spectrum dynamics of $H_h$ and $G_h$, for $h=0,1,2$. We observe that the top eigenvalus of all $G_h$ 
are of the same order of magnitude. 
To get a better sense of which layer contributes more, we also analyze the eigenvectors corresponding to the top eigenvalues of $H_0$. We evaluate the magnitude of the eigenvector components corresponding to each layer, normalized by the layer size, and the results can be found in Figure \ref{fig:layerwise_loadings}. Overall, for simple problem, layer 2 (connected to the output) always has the largest value while layer 0 contributes less. Such a relationship holds for hard problem at the beginning of the training, then the difference among the 3 layers shrinks, and eventually all layers have almost equal contributions.

    
\section{SGD Dynamics}
\label{sec:dynamics}

\begin{figure}[t!]
\centering
 \subfigure[Gauss-10, batch size 5, 0\% random labels]{
 \includegraphics[width = 0.98 \textwidth]{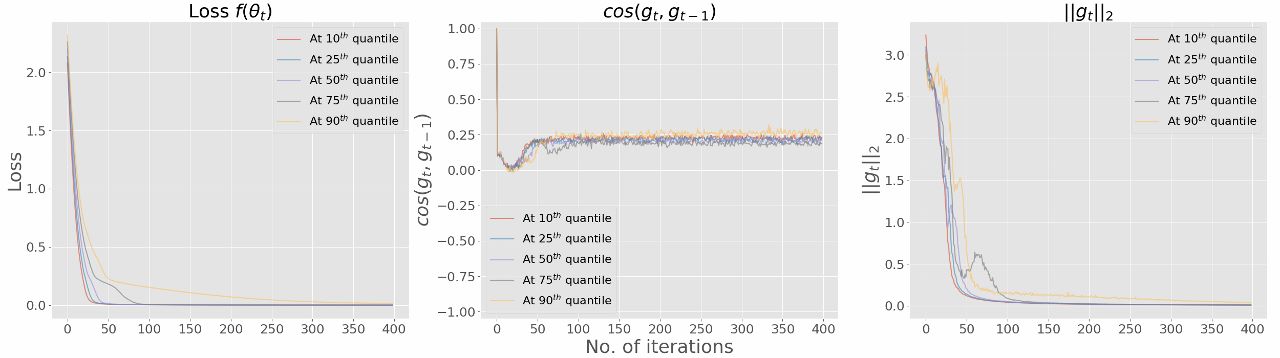}
 }
 \subfigure[Gauss-10, batch size 5, 15\% random labels]{
 \includegraphics[width = 0.98 \textwidth]{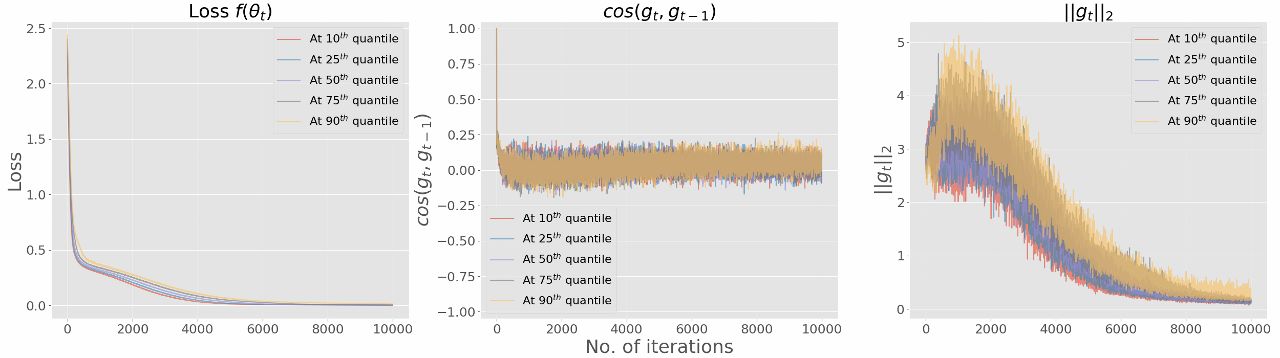}
 }
 \caption[]{Gauss-10: Dynamics of the loss $f(\theta_{t})$ (left), the angle of two successive SGs $\cos(g_t,g_{t-1})$ (middle), and the norm of the SGs $\|g_t\|_2$ (right) at different quantiles of $f(\theta_{t})$.} 
 \label{fig:loss_cos_g2}
\end{figure}

In this section, we study the empirical dynamics of the loss, cosine of the angle between subsequent SGs, and the $\ell_2$ norm of the SGs based on constant step-size SGD for fixed batch sizes and averaged over 10,000 runs. We then present a distributional characterization of the loss dynamics as well as a large deviation bound of the change in loss at each step. Further, we specialize the analysis for the special cases of least squares regression and logistic regression to gain insights for these cases.  Finally, we present convergence results to a stationary point for mini-batch SGD with adaptive step sizes as well as adaptive preconditioning.

\begin{figure}[t!]
\centering
\subfigure[Gauss-10, Batch size: 5, 0\% Random labels.]{
\includegraphics[width = 0.99 \textwidth]{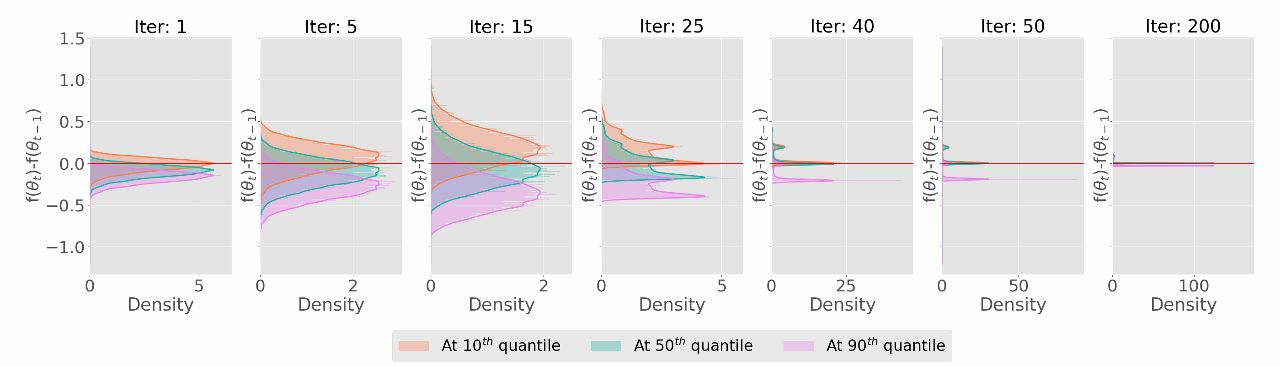}
\label{subfig:gauss10r0}
}
\subfigure[Gauss-10, Batch size: 5, 15\% Random labels.]{
\includegraphics[width = 0.99 \textwidth]{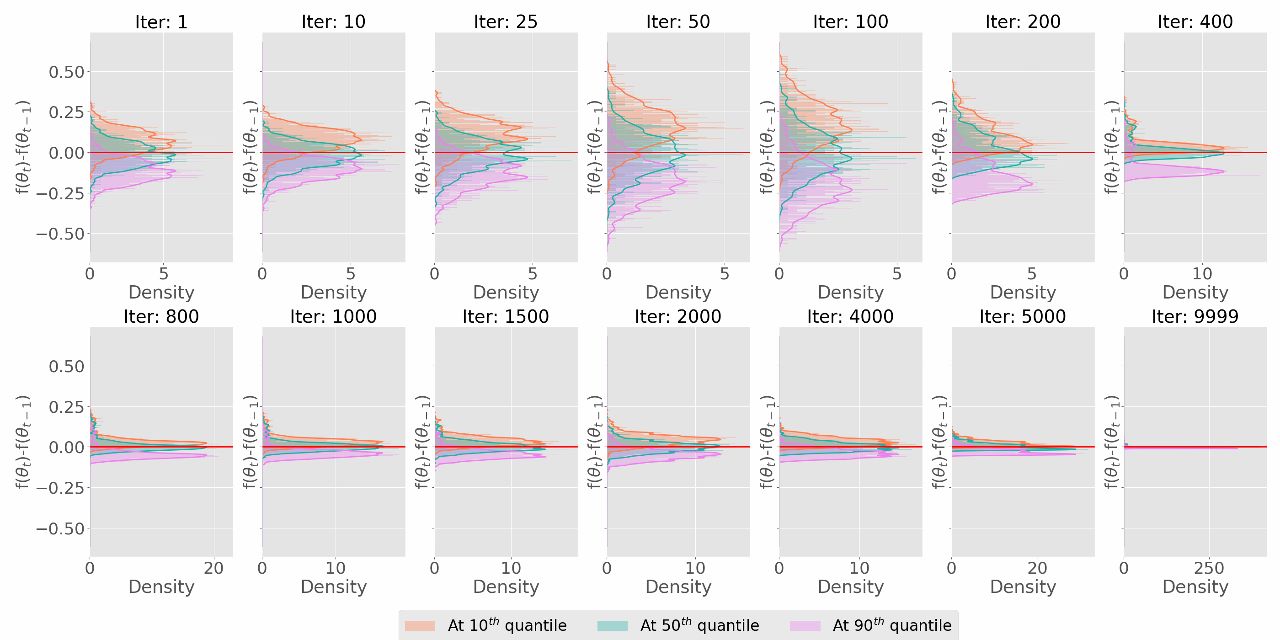}
\label{subfig:gauss10r15}
}
\caption[]{Dynamics of the distribution $\Delta_t(f)=f(\theta_{t+1}) - f(\theta_t)$ conditioned at $\theta_t$ for the Gauss-10 dataset trained with small batches containing one twentieth of training samples (5/100). Red horizontal line highlights the value of $\Delta_t(f)=0$. The loss-difference dynamics mainly consist of two phases: (1) the mean of $\Delta_t(f)$ decreases with an increase of variance (see (a): iteration 1 to 15, and (b): iteration 1 to 100); (2) the mean of $\Delta_t(f)$ increases and reaches 0 while the variance shrinks ( see (a):  iteration 25 to 200, and (b): iteration 200 to the end).} 
\label{fig:loss_dynamic}
\end{figure}

\begin{figure}[t!]
\centering
 \subfigure[Gauss-10, batch size 5, 0\% Random labels]{
 \includegraphics[width = 0.4 \textwidth]{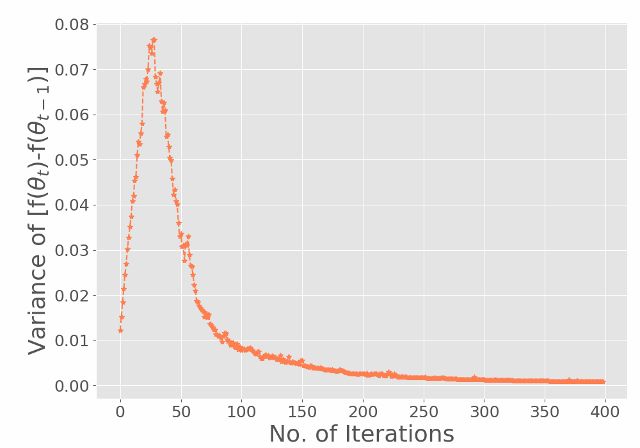}
 } 
 \subfigure[Gauss-10, batch size 5, 15\% random labels]{
 \includegraphics[width = 0.4 \textwidth]{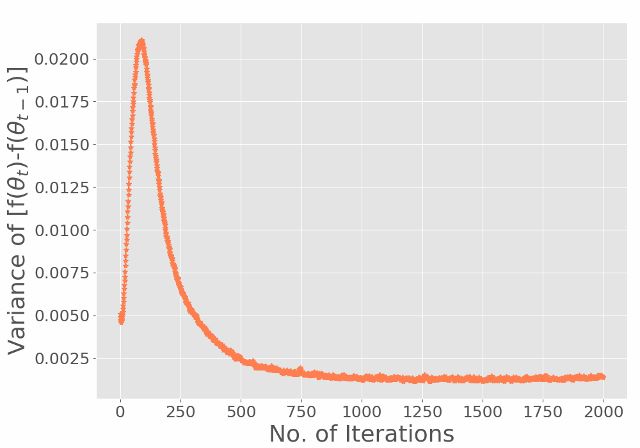}
 }
 \subfigure[Gauss-10, Batch size 50, 0\% Random labels.]{
 \includegraphics[width = 0.4 \textwidth]{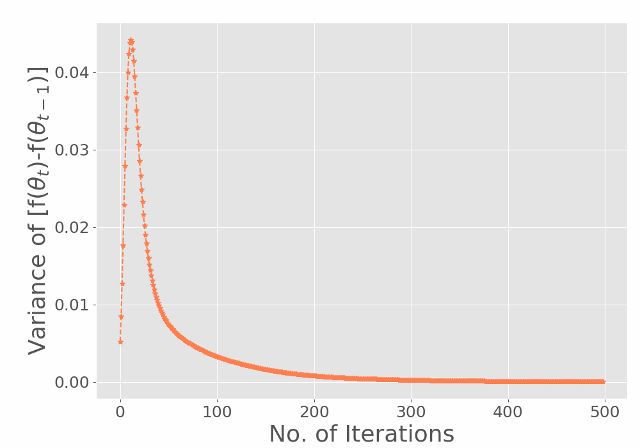}
 } 
 \subfigure[Gauss-10, Batch size 50, 15\% Random labels.]{
 \includegraphics[width = 0.4 \textwidth]{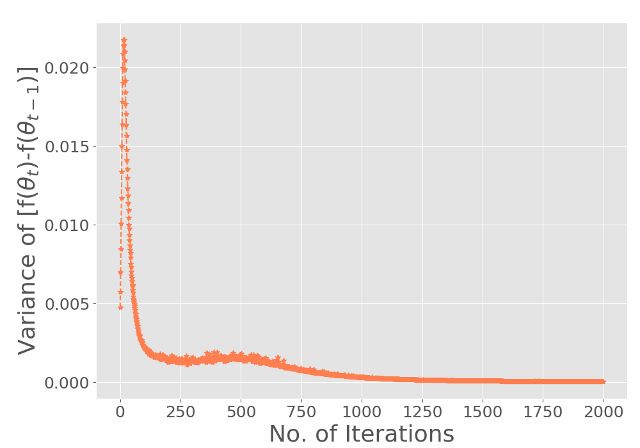}
 }
 \caption[]{The dynamics of the variance of $\Delta_t(f)=f(\theta_{t+1}) - f(\theta_t)$ conditioned at $\theta_t$ during training. The variance sharply increases with a short period of time at the beginning, then continues to decrease until convergence. For both easy and hard problem with various batch sizes, the variance exhibits a similarly behavior. }  
 \label{fig:loss_diff_var}
\end{figure}

\subsection{Empirical Loss Dynamics}  
\label{sec:emp}

The stochastic parameter dynamics of $\theta_t$ as in \eqref{eq:sgd} and associated quantities such as the loss $f(\theta_t)$ can be interpreted as a stochastic process. Since we have 10,000 realizations of the stochastic process, i.e., parameter and loss trajectory based on SGD, we present the results at different quantiles of the loss $f(\theta_t)$ at each iteration $t$.

{\bf SGD dynamics.} Figure~\ref{fig:loss_cos_g2} shows the dynamics of the loss $f(\theta_t)$, the angle between two consecutive SGs $\cos(g_t,g_{t-1})$, and the $\ell_2$ norm of the SGs $\|g_t\|_2$ for problems with different levels of difficulty, i.e., percentage of random labels in Gauss-10. For the simple problem with true labels (0\% random) (Figure \ref{fig:loss_cos_g2}(a)), SGD converges fast in less than 50 iterations. The three quantities $f(\theta_t), \cos(g_t, g_{t-1})$, and $\|g_t\|_2$ of all quantiles exhibit similar behavior, and decrease rapidly in the early iterations. Both $f(\theta_t)$ and $\|g_t\|_2$ converge to zero, while $\cos(g_t, g_{t-1})$ reaches a steady state and oscillates around 0.25, indicating subsequent SGs are almost orthogonal to each other (the angle between subsequent SGs is greater than $75^{\circ}$).

The dynamics become more interesting with 15\% random labels, the more difficult problem. In the initial phase, all quantiles of the loss and the angle between subsequent SGs drop sharply, similar to the 0\% random label case. On the other hand, the gradient norms $\|g_t\|_2$  of all quantiles increase despite the slight drop for the $75^{th}$ and $90^{th}$ quantiles at the very beginning (see Figure \ref{fig:loss_cos_g2_first100} in the appendix for more details). Once the gradient norm $\|g_t\|_2$  peaks,  and $\cos(g_t, g_{t-1})$ hits a valley, SGD enters a convergence phase. At this late phase, the gradient norm $\|g_t\|_2$  shows a steady decrease, while $\cos(g_t, g_{t-1})$ grows again until it reaches a steady state and begins to oscillate around 0. The loss persistently reduces to 0, but the rate of change also declines. We also observe similar dynamics in both MNIST and CIFAR-10 (see Figures~\ref{fig:loss_cos_g2_cifar} and~\ref{fig:loss_cos_g2_mnist} in the Appendix).

{\bf Empirical loss dynamics.} Let $\Delta_t(f)$ denote the stochastic loss difference, i.e., 
\beq
\Delta_t(f) = f(\theta_{t + 1}) - f(\theta_t)~.
\label{eq:loss_diff}
\eeq

Figure~\ref{fig:loss_dynamic} shows the empirical distributions of $\Delta_t(f)$ at the $10^{th},50^{th}$ and $90^{th}$ quantiles of the loss at iteration $t$ respectively. In particular, we get the  empirical distribution of the $q^{th}$ quantile for every iteration from 10,000 runs of SGD. Overall, the distributions of $\Delta_t(f)$ are roughly symmetric, and mainly contain two stages of change: in the first stage, the means of both the upper ($90^{th}$) and lower ($10^{th}$) quantile distributions  move away from the red horizontal line where $\Delta_t(f)=0$ (Figure~\ref{fig:loss_dynamic_k10} (a): iteration 1 to 15, and (b): iteration 1 to 100) while the variance of $\Delta_t(f)$ grows for all quantiles (Figure~\ref{fig:loss_diff_var} (a) and (b)). In the subsequent stage, the mean of both the upper ($90^{th}$) and lower ($10^{th}$) quantiles moves towards zero (Figure~\ref{fig:loss_dynamic_k10} (a): iteration 25 to 200, and (b): iteration 200 to 9999), while the variance of all quantiles shrinks significantly. As SGD converges, the mean of $\Delta_t(f)$ at all quantiles stays near zero, and the variance becomes very small.

\subsection{Deviation Bounds for Loss Dynamics} 
We consider the following two types of SGD updates: 
\beq
\theta_{t+1} = \theta_t - \eta_t \nabla \tilde{f}(\theta_t), 
\eeq
to be referred to as {\em vanilla SGD} in the sequel, and
\beq
\theta_{t+1} = \theta_t - A_t \nabla \tilde{f}(\theta_t),
\label{eq:sgd:pre}
\eeq
to be referred to as {\em preconditioned SGD} with diagonal preconditioner matrix $A_t$. Recall that here $\nabla \tilde{f}(\theta_t)$ represents the SG of $f$ at iteration $t$ computed based on a mini-batch of $m$ samples. Notice that if we take $A_t$ to be $\eta_t\I$, preconditioned SGD becomes vanilla SGD. 

Assuming SGs follow a multi-variate distribution with mean $\mu_t$ and covariance $\frac{1}{m}\Sigma_t$, we can represent the SG as
\begin{equation}
\nabla \tilde{f}(\theta_t) = \mu_t + \frac{1}{\sqrt{m}}\Sigma_t^{1/2} g~,
\label{eq:generative}
\end{equation}
where $g$ is a random vector sampled uniformly from $\sqrt{p}\s^{p-1}$  \cite{vers12,leta91} representing a sphere of radius $\sqrt{p}$ in $\R^p$. The assumption on $g$ is reasonable since in a high dimensional space, sampling from an isotropic Gaussian distribution is effectively the same as
sampling from (a small annular ring around) the sphere with high probability~\cite{vers12}.
Let us denote the batch dependent second moment of the SG as
\begin{equation}
    M_t^{(m)} \triangleq \E \left[\nabla \tilde{f}(\theta_t)\nabla \tilde{f}(\theta_t)^T \right] = \frac{1}{m}\Sigma_t + \mu_t \mu_t^T~.
\end{equation}
We show in theory that for vanilla SGD, our observations of the two-phase dynamics of $\Delta_t(f)$ conditioned on $\theta_t$, i.e., the inter-quantile range and variance increasing first then decreasing, in Figure~\ref{fig:loss_dynamic} can be characterized by the Hessian $H_f(\theta)$, the covariance $\Sigma_t$ and associated quantities introduced in Section \ref{sec:hessian}. To proceed with our analysis, we make Assumption \ref{asp:rs}, and then present Theorem \ref{thm:dynamic} characterizing a conditional expectation and large deviation bound for $\Delta_t(f)$ defined in \eqref{eq:loss_diff}.

\begin{asmp}[Bounded Hessian]
\label{asp:rs}
    Let $R(\theta_t) = \{ \theta_t - \eta_t\mu_t - \eta_t \frac{1}{\sqrt{m}}\Sigma_t^{1/2}g: \forall g \in \sqrt{p}B_{p}\}$, where $\eta_t$ is the step size in \eqref{eq:sgd}, and $\sqrt{p}B_2^{p}$ is a ball of radius $\sqrt{p}$ in $\R^p$. There is an $L > 0$ such that $\|H_f(\theta)\|_2 \leq L$ for all $\theta \in R(\theta_t) $.
\end{asmp}

Assumption \ref{asp:rs} is the so called local smoothness condition \cite{valz15}. From Figure \ref{fig:principal_angle_gauss10}, \ref{fig:principal_angle_mnist}, and \ref{fig:principal_angle_cifar10}, the largest eigenvalues of $H_f(\theta_t)$ will decrease after the first few iterations. Therefore it is reasonable to assume the spectral norm of $H_f(\theta)$ to be bounded when $\theta$ is close to a point in SGD iterations.

\begin{restatable}{theo}{theoi}
\label{thm:dynamic}
Let $\Delta_t(f) = f(\theta_{t + 1}) - f(\theta_t)$. If Assumption \ref{asp:rs} holds, we have for vanilla SGD \eqref{eq:sgd}
\begin{equation} \label{Eq: dynamic}
    - \eta_t \| \mu_t \|_2^2 - \frac{\eta_t^2}{2}L\tr M_t\leq\mathbb{E}_{\theta_{t + 1}} \left[ \Delta_t(f)\bigg|\theta_t \right]
    \leq- \eta_t \| \mu_t \|_2^2 + \frac{\eta_t^2}{2}L\tr M_t~. 
\end{equation}    

Further, for all $s > 0$, we have 
\begin{align}
P  \bigg[ \Delta_t(f) - \bigg\{ - \eta_t \| \mu_t \|_2^2 + \frac{\eta_t^2}{2}L\tr M_t \bigg\}  \geq  ~~~ s  \bigg| \theta_t \bigg]  
  \leq 2\exp \left[ -c \min \left( \frac{s^2}{ \alpha_{t,1}^2} , \frac{s^2}{ \alpha_{t,2}^2}, \frac{s}{ \alpha_{t,3}} \right) \right]~,
\\
P \bigg[ \Delta_t(f) - \bigg\{ - \eta_t \| \mu_t \|_2^2 - \frac{\eta_t^2}{2}L\tr M_t \bigg\}  \leq  -s  \bigg| \theta_t \bigg] 
  \leq 2\exp \left[ -c \min \left( \frac{s^2}{ \beta_{t,1}^2} , \frac{s^2}{ \alpha_{t,2}^2}, \frac{s}{ \alpha_{t,3}} \right) \right]~,
\end{align}
where 
\begin{align*}
\alpha_{t,1}  = \eta_t |1 - \eta_t L| \frac{1}{\sqrt{m}}\| \Sigma_t^{1/2} \mu_t \|_2~,& \qquad \beta_{t,1}  = \eta_t |1 + \eta_t L| \frac{1}{\sqrt{m}} \| \Sigma_t^{1/2} \mu_t \|_2~,\\
\alpha_{t,2}  = \frac{\eta_t^2 L}{2m}\left\| \Sigma_t \right\|_F~,& \qquad \alpha_{t,3}  = \frac{\eta_t^2 L}{2m}\left\| \Sigma_t \right\|_2~, 
\end{align*}
and $c > 0$ is an absolute constant. 
\end{restatable}

The proof of Theorem \ref{thm:dynamic} can be found in Appendix \ref{app:dynam}. At iteration $t$, Theorem \ref{thm:dynamic} tells us that the conditional distribution of $\Delta_t(f)$ stays  in the interval $- \eta_t \| \mu_t \|_2^2 \pm \frac{\eta_t^2}{2}L\tr M_t$ with high probability, where the concentration depends on dynamic quantities $\alpha_{t, 1}, \beta_{t,1}, \alpha_{t, 2}$, and $\alpha_{t, 3}$ related to SG covariance and expectation.

The interval depends on two key quantities: (1) the negative of the 2-norm of the full gradient $\|\mu_t\|_2^2$ (first moment of the SG), and (2) the trace of the second moment of the SG $\tr M_t$. The first term tends to push the mean downward, while the second term lifts the mean. When $- \eta_t \| \mu_t \|_2^2 + \frac{\eta_t^2}{2}L\tr M_t$ is less than zero and $\alpha_{t, 1}, \alpha_{t, 2}$ and $\alpha_{t, 3}$ are small, SGD will decrease the loss with high probability.

The dynamics of the variance of $\Delta_t(f)$ depends on $\Sigma_t$, i.e., the variance of the change in loss function depends on the covariance of the SGs (see Figures \ref{fig:full_spectrum} and \ref{fig:loss_diff_var}). 

For constant step size $\eta_t = \eta$, the dynamics of $\alpha_{t, 2}$ corresponds to the dynamics of $\|\Sigma_t\|_F$ and the dynamics of $\alpha_{t, 3}$ corresponds to the dynamics of $\|\Sigma_t\|_2$. The eigenvalues of $\Sigma_t$ first increase then decrease (Figure \ref{fig:full_spectrum}: middle), and so does $\|\Sigma_t\|_F$ and $\|\Sigma_t\|_2$. Therefore, the dynamics of the variance of $\Delta_t(f)$ follow a similar trend (Figure \ref{fig:loss_diff_var}).

SGD is able to escape certain types of stationary point or local minima. Consider a scenario where $\theta_t$ reaches a stationary point of $f(\theta)$ such that $\nabla f(\theta_t) = 0$, but $\theta_t$ is not the local minima of all $f(\theta; z)$, i.e., $\exists i \in \{1, 2, \ldots, n\}$, such that $\nabla f(\theta; z_i) \neq 0$. Then we have $\mu_t = 0$ but $M_t \neq 0$, and the deviation bound becomes
\[
    P  \bigg[ \bigg| \Delta_t(f) - \frac{\eta_t^2}{2}L\tr M_t \bigg| \geq s  \bigg| \theta_t \bigg]  \leq 4\exp \left[ -c \min \left(\frac{s^2}{ \alpha_{t,2}^2}, \frac{s}{ \alpha_{t,3}} \right) \right]~,
\]
so that $\Delta_t(f)$ concentrates around $\frac{\eta_t^2}{2}L\tr M_t > 0$ in the current setting since $M_t$ is positive definite. Therefore $f$ will increase and escape such stationary point or local minima. 

We give a detailed characterization of $\Delta_t(f)$ for two over-parameterized problems using Theorem \ref{thm:dynamic}: 
(a) high dimensional least squares and (b) high dimensional logistic regression. For both problems, SGD has two stages of change as discussed in Section \ref{sec:emp}. 

\subsubsection{High Dimensional Least Squares} 
Considering the least squares problem in Example \ref{ex:ls}, we have the following result:

\begin{restatable}{corr}{corri}
Consider high dimensional least squares as in Example \ref{ex:ls}.
Let us assume $\|H_f(\theta)\|_2 \leq L$, $\sigma_{min} (\frac{1}{n}X X^T) \geq \alpha > 0$, and $\max_i \|x_i\|_2^2 \leq \beta$, where $\sigma_{min}$ is the minimum singular value. Choosing $\eta = \frac{\alpha}{\beta L}$, for $s > 0$ we have

\beq
 P\left[\left.\Delta_t(f) \geq -\frac{\alpha^2}{2 \beta L n} \|X \theta_t - y\|_2^2 + s \right| \theta_t\right] 
\leq  \exp\left[ -c \min \left( \frac{s^2}{ \alpha_{t,1}^2}, \frac{s^2}{ \alpha_{t,2}^2}, \frac{s}{ \alpha_{t,3}} \right) \right],
\label{prob:ls}
\eeq
where 
$\alpha_{t,1} = |1 - \frac{\alpha}{2\beta}|\frac{1}{\sqrt{m}}\|\Sigma_t^{\frac{1}{2}} \mu_t\|_2, \alpha_{t,2} = \frac{\alpha^2}{2 \beta^2 L m}\|\Sigma_t\|_F, \alpha_{t,3} = \frac{\alpha^2}{2 \beta^2 L m}\|\Sigma_t\|_2$,
and $c$ is a positive constant.
\label{cor:ls}
\end{restatable}

Note that Corollary \ref{cor:ls} presents a one-sided version of the concentration, but focuses on the side of interest, which characterizes the lower side or decrease in $\Delta_t(f)$. From Corollary \ref{cor:ls}, SGD for high dimensional least squares has two phases. Early on in the iterations, $-\frac{\alpha^2}{2 \beta L n} \|X \theta_t - y\|_2^2$ will be much smaller than zero, thus SGD can sharply decrease $f$. $\alpha_{t,1}, \alpha_{t,2}, \alpha_{t,3}$ are large since $\theta_t$ is not close to $\theta^*$ and $\Sigma_t$ has large eigenvalues. Therefore the probability density of $\Delta_t(f)$ will spread out over its range, and facilitates exploration. In later iterations, both $\frac{\alpha^2}{2 \beta L n} \|X \theta_t - y\|_2^2$ and $\|\Sigma_t\|_2$ are small, therefore SGD will help with the loss approaching the global minima with a sharp concentration.

\subsubsection{High Dimensional Logistic Regression} 
For the binary logistic regression problem in Example \ref{ex:lr}, we have:
\begin{restatable}{corr}{corrii}
\label{cor:lr}
Consider high dimensional logistic regression given by \eqref{loss:lr}.  Let us assume $\|H_f(\theta)\|_2 \leq L$, $\sigma_{\min}(\frac{1}{n} \sum_{i=1}^n x_i x_i^T) \geq \alpha > 0$, and $\max_i \|x_i\|_2^2 \leq \beta$, where $\sigma_{min}$ is the minimum singular value. If we choose $\eta = \frac{\alpha}{\beta L}$, we have
\beq
	P\left[ \left. \Delta_t (f) \geq - \frac{\alpha^2}{2 \beta L n} \sum_{i=1}^{n}(y_i - p_{\theta_t}(x_i))^2 + s \right| \theta_t \right] 
	\leq  \exp\left[ -c \min \left( \frac{s^2}{ \alpha_{t,1}^2}, \frac{s^2}{ \alpha_{t,2}^2}, \frac{s}{ \alpha_{t,3}} \right) \right],
\eeq
where $\alpha_{t,1}  = \eta |1 - \eta L| \frac{1}{\sqrt{m}} \| \Sigma_t^{1/2} \mu_t \|_2~,\alpha_{t,2}  = \frac{\eta^2 L}{2 m}\left\| \Sigma_t \right\|_F~,\alpha_{t,3}  = \frac{\eta^2 L}{2 m}\left\| \Sigma_t \right\|_2~$, and $c > 0$ is an absolute constant.
\end{restatable}
As in the case of least squares, \ref{cor:lr} focuses on a one-sided bound, focusing on the decrease in $\Delta_t(f)$. Proofs of Corollary~\ref{cor:ls} and \ref{cor:lr} are in appendix \ref{app_Ex_regression}.
In the high dimensional case, the data are always linearly separable. In this case $\sum_{i=1}^n (y_i - p_{\theta_t}(x_i))^2$ can be arbitrarily small (depending on $\| x_i \|_2$), therefore SGD will have a similar behavior as least squares. In the early phase, $ \eta_t \| \mu_t \|_2^2 - \frac{\eta_t^2}{2}L\tr M_t$ is large, thus SGD can sharply decrease $f$; $\alpha_{t,1}, \alpha_{t,2}, \alpha_{t,3}$ are also large, allowing SGD to explore more directions. In the later phase, SGD will steadily decrease the loss and eventually approach the global minima.

\subsection{Deviation Bound for Loss Process} 
Since the SGD update \eqref{eq:sgd} is a Markov process \cite{gall16}, if we condition on $\theta_t$, then $\Delta_t(f)$ is independent from $\theta_{t'}$ for all $t' < t$. The sequence $\Delta_t(f)-\mathbb{E}[\Delta_t(f)|\theta_t]$ is a Martingale Difference Sequence (MDS) \cite{boucheron_concentration_2013} because the expectation of the sequence conditioned on the history $\theta_t$ equals to zero. Now we focus on the deviation behavior of the random process $\Delta_t(f)$ for any choice of stepsize $\eta_t$. Utilizing two sided tail bounds for $\Delta_t(f)$ in Theorem \ref{thm:dynamic}, the MDS $\Delta_t(f)-\mathbb{E}\left[\Delta_t(f)|\theta_t\right]$ has sub-exponential tail \cite{vershynin_high_2018}. Then the Theorem \ref{thm:Bernstein} is a consequence of the Azuma-Bernstein inequality~\cite{melnyk_estimating_2016} for sub-exponential MDSs.

\begin{restatable}{theo}{theoii}
\label{thm:Bernstein}
If Assumption \ref{asp:rs} holds, stepsize $\eta_t\leq\eta$, gradient $\|\mu_t\|_2\leq\mu$, covariance $\|\Sigma_t\|_2^{1/2}\leq \sigma^{1/2}$, and the loss function is Lipschitz, i.e., $|f(\theta_1)-f(\theta_2)|\leq L_1\|\theta_1 - \theta_2\|_2$  for arbitrary $\theta_1, \theta_2 \in \R^p$, then for vanilla SGD, we have
\beq
P\left[|f(\theta_T)-\mathbb{E}f(\theta_T)|\geq \sqrt{T}s\right]\leq 2 \exp \left[-c \min\left(\frac{s^2}{8K_2^2 },\frac{\sqrt{T}s}{2K_2} \right) \right]
\eeq
where for vanilla SGD, 
\begin{equation}
 K_2= \max{\left(2L_1\eta\left(\mu+\frac{\sqrt{\sigma p}}{\sqrt{m}}\right),\left( \eta^2 L\left(\sigma+\mu^2\right)+\frac{\eta\mu^2}{\sqrt{m}}\left(1+\eta L\right)\sqrt{\sigma}+\frac{(m+1)\eta^2L\sigma}{2m}\right)\right)}~, 
\end{equation}
   
and $c$ is an absolute constant.
\end{restatable}

With $s=\sqrt{T}K_2$, Theorem \ref{thm:Bernstein} becomes 
\begin{equation}
\mathbb{P}[|f(\theta_T)-\mathbb{E}f(\theta_T)|\geq K_2T]\leq 2\exp(-cT)~,
\end{equation}
for some suitable constant $c$.
For a Brownian motion sequence $B_t$ which has $B_t \sim N(0,\sigma^2 t)$, we have the tail bound ${P}[|B_T|\geq \sigma T]\leq 2\exp(-T/2)$. We can see the tail of $f(\theta_T)-\mathbb{E}f(\theta_T)$ shares the same upper bound as the tail of $B_{2cT}$ with $\sigma=\frac{K_2}{2c}$. 

\subsection{Convergence of Adaptive SGD}
To sharpen the analysis, continuing with Theorem \ref{thm:dynamic}, if we take step size $\eta_t$ such that $- \eta_t \| \mu_t \|_2^2 + \frac{\eta_t^2}{2}L\tr M_t\leq 0$, then $\mathbb{E}[\Delta_t(f)|\theta_t]\leq - \eta_t \| \mu_t \|_2^2 + \frac{\eta_t^2}{2}L\tr M_t \leq 0$, and the random process $\Delta_t(f)$ becomes a non-negative super-martingale~\cite{jacod_probability_2012}. The martingale convergence theorem \cite{williams_probability_1991} leads to the following conclusion:
\begin{restatable}{theo}{theoiii}
\label{thm:convergence}
Given Assumption~\ref{asp:rs}, with stepsize $\eta_t$ for adaptive step-size SGD s.t.
\begin{equation}
    \frac{\|\mu_t\|_2^2}{2L\tr M_t}<\eta_t < \frac{3\|\mu_t\|_2^2}{2L\tr M_t}~,
    \label{step:vanilla}
\end{equation}
or with diagonal preconditioner $A_t=\text{diag}\{a_{t,1}, \ldots, a_{t,p}\}$ s.t.
\begin{equation}
\frac{\mu_{t, j}^2}{2L M_{t, j}}<a_{t,j} < \frac{3\mu_{t, j}^2}{2L M_{t, j}}~, 
\label{step:precond}
\end{equation}
we have the following:
\begin{enumerate}[(1)]
    \item the random processes $f(\theta_t)$ almost surely converges to a random variable $f(\theta^*)$, where $\theta^*$ is a stationary point, i.e., $\nabla f(\theta^*)=0$ for the empirical loss function $f$; 
    \item let $f^*$ be the global minimum of $f$, assume $\tr \Sigma_t \leq \sigma^2$ for a constant $\sigma > 0$, then for any $\epsilon > 0$, after $T \geq \frac{2L(f(\theta_0) - f^*)(\sigma^2 / m + \epsilon)}{\epsilon^2}$ iterations, we have $\min_t\E \|\nabla f(\theta_t)\|_2^2 \leq \epsilon$.
\end{enumerate}
\end{restatable}

Theorem \ref{thm:convergence} shows that by choosing a proper adaptive step size or preconditioner, SGD is able to converge to a stationary point. Unlike the analysis of \cite{rekk18}, the convergence does not require the objective $f$ to be convex. With the bounded covariance assumption, we have the convergence rate as $\min_t\E \|\nabla f(\theta_t)\|_2^2 = \mathcal{O}(1 / \sqrt{T})$, which matches the state of the art \cite{ghla13}. Our choice of step size and preconditioner requires knowledge of $\|\mu_t\|_2^2$ and $\tr M_t$. To make this algorithm practical, we need to estimate $\|\mu_t\|_2^2$ and $\tr M_t$ dynamically, which will be considered in future work.

\subsection{Comparison with ADAM}
We share brief remarks comparing our diagonal preconditioned SGD with ADAM \cite{kiba15,rekk18,strk19}, a popular adaptive gradient method for training deep nets.
In ADAM, the algorithm maintains two exponential moving averages, respectively of SGs and squares of SGs for each coordinate $j=1,\ldots,p$, given by 
\begin{align}
    g_{t,j} & = \gamma_1 g_{t - 1,j} + (1 - \gamma_1)\tilde{\mu}_{t,j}, \qquad 0 < \gamma_1 < 1 \\
    \nu_{t, j} & = \gamma_2 \nu_{t-1, j} + (1 - \gamma_2) \tilde{\mu}_{t, j}^2, \qquad 0 < \gamma_2 < 1~.
\end{align}

where $\tilde{\mu}_{t,j} = [\mu^{(m)}(\theta_t)]_j$, corresponding to the mini-batch based estimate of the gradient $\mu_{t,j}$. For the sake of the current discussion, $g_{t,j}$ and $\nu_{t,j}$ can be considered estimates of $\mu_{t,j}$, the first moment, and $M_{t,j}$ the second moment in our context.

ADAM uses a diagonal preconditioner $A_t=\text{diag}\{a_{t,1}, \ldots, a_{t,p}\}$ with  $a_{t,j} = \frac{\eta}{\sqrt{\nu_{t, j}}}$, where $\eta> 0$ is a constant step size. At each step t, the update of ADAM for $j=1,\ldots,p$ is given by
\begin{equation}
    \theta_{t + 1,j} = \theta_{t,j} - a_{t,j} g_{t,j} = \theta_{t,j} - \eta \frac{g_{t,j}}{\sqrt{\nu_{t,j}}}~.
\end{equation}
In the current context, using our notation, the update has the form:
\begin{equation}
    \theta_{t + 1,j} = \theta_{t,j} - \eta \frac{\mu_{t,j}}{\sqrt{M_{t,j}}}~.
\end{equation}

In this case, ADAM becomes a fixed algorithm. We can rewrite our proposed preconditioned SGD in the following form
\beq
    \theta_{t+1, j} = \theta_{t, j} - \eta \left(\frac{\mu_{t, j}}{\sqrt{M_{t, j}}} \right)^2 g_{t, j}
\eeq

Our proposed preconditioned SGD can be seen as a variety of ADAM where each entry is given by the square of ADAM update times the stochastic gradient. Therefore,
we used the magnitude of ADAM while we used the direction of SGD.

\section{Generalization Bound}
\label{sec:bound}

In this section, we present a scale-invariant PAC-Bayesian generalization bound \cite{mcda99,lash03, nebb17} which considers the local structure of the parameter learned from the training data. The PAC-Bayesian bound characterizes the generalization error in terms of the KL-divergence between the posterior distribution learned from the training data and the prior distribution independent of the training data. In current PAC-Bayes bounds  \cite{nebb17, vazk19, nebm17} for deep nets, the posterior distribution is assumed to be a Gaussian $\cN(\theta, \sigma^2 \I)$ for some $\sigma > 0$ with $\theta \in \R^p$ the parameter learned from from training data. The prior distribution is assumed to be $\cN(0, \sigma^2 \I)$.   

For the analysis, we assume that the posterior $\cQ_\theta$ to be an anisotropic multivariate Gaussian with mean corresponding to the parameter $\theta$ learned from the training data and the covariance relating to the Hessian of the loss $H_f(\theta)$. The prior distribution $\cP$ is fixed before training and is independent of the training data. Note that the proposed generalization bound holds for any specific parameter $\theta$ obtained from the SGD based learning process. In particular, the bound and associated analysis does not rely on $\theta$
being a minima or a stationary point. 

While existing PAC-Bayes analysis assumes the posterior to be an isotropic Gaussian \cite{smle18, nebm17}, assuming the posterior to be an anisotropic Gaussian related to the Hessian of the loss acknowledges the local flatness and sharpness structures and helps gain additional insights. In particular, we make the following specific assumption regarding the prior $\cP$ and posterior $\cQ_\theta$:

\begin{asmp} \label{asmp: p_q}
	We consider the prior distribution to be 
	multivariate Gaussian
	\begin{equation}
	\mathcal{P} \sim \mathcal{N}(\theta_0, \Sigma_\mathcal{P}) \ \ \ \text{where} \ \ \ \Sigma_\mathcal{P} = \diag(\sigma_1^2,...,\sigma_p^2 )~,
	\end{equation}
	where $\theta_0$ and $\Sigma_\cP$ are fixed before training. We assume the posterior distribution to be an anisotropic multivariate Gaussian 
    \begin{equation}
	\cQ_\theta \sim \mathcal{N}(\theta, \Sigma_{\cQ_{\theta}} )~,
    \end{equation}
	where the mean $\theta$ is the parameter learned from training data, and the precision matrix $\Sigma_{\cQ_{\theta}}^{-1}$ is given by 
	\begin{equation}
	\Sigma_{\cQ_{\theta}}^{-1} = \diag(\nu_1, \nu_2,..., \nu_p ) \qquad  \text{s.t.} \quad 
	\nu_j = \max \left\{ H_f(\theta)[j,j], \frac{1}{\sigma_j^2} \right\}~,
	\end{equation}
	where $H_f(\theta)[j,j]$ is the $j^{th}$ diagonal element of the Hessian $H_f(\theta)$ and $\sigma_j^2$ is the variance of the $j^{th}$ coordinate of the posterior $\cP$. 
\end{asmp}	

The posterior distribution $\mathcal{Q}_\theta$ considers the structure of the Hessian corresponding to the parameter $\theta$ learned from the training data. In particular, the precision matrix $\Sigma_{\cQ_{\theta}}^{-1}$ uses the diagonal elements of the Hessian $H_f(\theta)[j,j]$ as parameter wise precision capped below by the precision $\frac{1}{\sigma_j^2}$ of the prior. For the special case of isotropic prior, $\sigma_j^2 = 1$ and $\nu_j = \max \{ H_f(\theta)[j,j], 1 \}$ for all $j$. We consider two cases to understand the assumption on the posterior better. For dimension $j$ with flat curvature, the diagonal of the Hessian $H_f(\theta)[j,j] \leq \frac{1}{\sigma_j^2}$, so that the posterior precision $\nu_j = \frac{1}{\sigma_j^2}$ and the posterior variance is exactly the same as the prior variance $\sigma_j^2$. Thus, for flat directions, the posterior variance is large, where large is determined by the prior variance. The spectrum of the Hessian (Figure \ref{fig:full_spectrum}) shows large subspaces with flat directions, i.e., 0 eigenvalues (curvature) at the minima. As we show shortly, such flatness is also maintained in the diagonals of the Hessian. Note that our choice of posterior caps the large variance along these flat directions to the variance of the prior, and prevents the posterior variances in the flat directions from going to infinity. For dimensions $j$ with sharp curvature, Hessian $H_f(\theta)[j,j] \gg \frac{1}{\sigma_j^2}$, so that the posterior precision $\nu_j = H_f(\theta)[j,j]$, and the posterior covariance is $\frac{1}{H_f(\theta)[j,j]}$, which will be quite small since the curvature captured by $H_f(\theta)[j,j]$ is large. Thus, the few directions with sharp curvature will have small variance in the posterior. In other words, the posterior suggests that that these parameters $\theta_j$ corresponding to the sharp curvature directions need stay close to their learned values, which serve as the mean of the posterior. 

In recent work, \cite{dipb17} showed that the Hessian $H_f(\theta)$ can be modified by a certain $\alpha$-scale transformation which scales the weights by non-negative coefficients but does not change the function (see Definition~\ref{def: scale}). More importantly, $\alpha$-scale transformation invalidates certain recently proposed flatness based generalization bounds \cite{hosc97b, chcs16, kemn17} by arbitrarily changing the flatness of the loss landscape for deep networks with positively homogeneous activation without changing the functions represented by the networks.

\begin{defn} ($\alpha$-scale transformation \cite{dipb17}) \label{def: scale}
	Let $\theta = (\theta_1, ..., \theta_L)$ be the parameters of $L$-layer feedforward network with rectified activation function and  $\Pi_{l =1}^L \alpha_l = 1$, where $\alpha_l > 0, \forall l \in \{1,..,L\}$. We define the the family of transformations
	\begin{equation}
	T_\alpha : (\theta_1, ..., \theta_l) \to   (\alpha_1 \theta_1, ..., \alpha_L \theta_L) ~,
	\end{equation}
	as an $\alpha$-scale transformation.
\end{defn}

Recall that the PAC-Bayes bound relies on the $KL(\cQ_{\theta} || \cP)$, the KL divergence or differential relative entropy between the posterior and the prior. In order to obtain a generalization bound invariant to the $\alpha$-scale transformation, we first note (Lemma \ref{Lem: KLD}) that $KL(\cQ_{\theta} || \cP)$  between two continuous distributions remains invariant under invertible transformations (Definition \ref{Defn: inver}). 

\begin{restatable}{defn}{inver}
\label{Defn: inver} 
\cite{halm74} A transformation $T: \calU \rightarrow \calV$ is said to be invertible if for any $v \in \calV$ there is a unique $u\in \calU$ such that
\begin{equation}
    T(u) = v.
\end{equation}
We can then define the invert transform $S: \calV \rightarrow \calU$ such that
\beq
    S(v) = u,
\eeq
and we have
\beq \label{eq: linear_alg}
    S(T(u)) = u \ \ \  \text{and} \ \ \   T(S(v)) = v, ~~\text{for any}~ u \in \calU ~~\text{and}~~ v \in \calV.
\eeq
\end{restatable}

\begin{restatable}{lemm}{lemmi}
\label{Lem: KLD} 
\cite{kl2011} The differential relative entropy between two continuous distributions remains invariant under invertible transformations. Specifically, for distributions $\cQ$ and $\cP$ of a continuous random variable with support $\cX \subseteq \mathbb{R}^p$ and $\cQ$ is absolutely continuous with respect to $\cP$, and let $\cQ^\prime$ and $\cP^\prime$ be the distributions after invertible transformation corresponding to $\cQ$ and $\cP$ respectively. Then we have the following
\begin{equation}
    KL(\cQ^\prime ||\cP^\prime) = KL(\cQ||\cP)
\end{equation}
\end{restatable}

Note that $\alpha$-scale transformation is a special case of an invertible transformation, and is in fact an invertible linear transformation $T(\theta)=A\theta$ where $A$ is non-singular (Definition~\ref{def: scale}). Then we show (Corollary \ref{corr: KPQ}) that, as a special case of Lemma \ref{Lem: KLD}, although the local structure such as the Hessian $H_f(\theta)$ get modified by the $\alpha$-scale transformation, the KL-divergence between the posterior $\cQ_\theta$ and prior $\mathcal{P}$ defined by Assumption \ref{asmp: p_q} is invariant to the $\alpha$-scale transformation if  Assumption \ref{asmp: p_q} holds, which underlies our scale-invariant bound. 

Let $\cP_{T_\alpha(\theta_0)}$ be the prior distribution after $\alpha$-scale transformation, $T_\alpha(\theta)$ be the parameter after applying $\alpha$-scale transformation to $\theta$ and $\cQ_{T_\alpha(\theta)}$ be the corresponding posterior defined by Assumption \ref{asmp: p_q}. We have the following corollary.

\begin{restatable}{corr}{corriv}
 \label{corr: KPQ}
	If Assumption \ref{asmp: p_q} holds, the $KL(\cQ_\theta||\cP)$ is invariant to $\alpha$-scale transformation, i.e.,
	\beq
	KL(\cQ_{T_\alpha(\theta)}|| {\cP_{{T_\alpha}(\theta_0)}}) = KL(\cQ_\theta|| \cP).
	\eeq
\end{restatable}

Note that the prior distribution  $\cP_{{T_\alpha}(\theta_0)}$ is obtained by applying the scaling defined in Definition \ref{def: scale} to the original coordinate system. The posterior $\cQ_{T_\alpha(\theta)}$ is obtained by the local curvature of $T_\alpha(\theta)$, i.e., Hessian $H_f(T_\alpha(\theta))$. 

Let $S$ be a sample of $n$ pairs $(x_i,y_i)$ drawn i.i.d.~from the distribution $\cD$. For two Bernoulli distributions with event probability $p,q$, the relative entropy
$kl(p \| q) \triangleq p \log (p/q) + (1-p) \log (1-p)/(1-q)$. The main generalization bound can be stated as follows:

\begin{restatable}{theo}{theoiv} 
\label{theo: bound1}
	Let $\mathcal{P} \sim \mathcal{N}(\theta_0, \Sigma_\mathcal{P}\mathbb{I})$ to be the prior and $\cQ_\theta \sim \mathcal{N}(\theta, \Sigma_{\cQ_{\theta}})$ to be the posterior defined following Assumption \ref{asmp: p_q}. Let $\tilde d = | \{ j : H_f(\theta)[j,j] > 1/ \sigma_j^2 \}|$, the corresponding precision values be  $\{ \tilde \nu_{(1)}, ..., \tilde \nu_{(\tilde d)}\}$ and the corresponding thresholds be $\{1/ \tilde\sigma_{(1)}, ... , 1/\tilde \sigma_{(d)}\}$. With probability at least $(1-\delta)$ over the choice of $S$ we have the following scale-invariant generalization bound:

\begin{equation}
	kl(\ell(\cQ_\theta, S) || \ell(\cQ_\theta, D))
	 \leq \frac{1}{2n} \left( \underbrace{\sum_{l =1}^{\tilde d} \ln \frac{\tilde \nu_{(l)}}{1/\tilde \sigma_{(l)}^2}}_{\textup{effective curvature}} + \underbrace{\sum_{j =1}^{p}\frac{(\theta[j]-\theta_0[j])^2}{\sigma_j^2}}_{\textup{precision weighted Frobenius norm}}  \right) + \frac{\ln\frac{n+1}{\delta}}{n}
	 \label{eq: bound1}
\end{equation}	
	where 
	\begin{equation*}
	\ell(\cQ_\theta,S) = E_{\theta^\prime \sim \cQ_\theta}\left[ \frac{1}{n} \sum_{i=1}^n \ell(y_i, \phi(\x_i,\theta^\prime)) \right] \qquad \text{and} 
	\qquad \ell(\cQ_\theta,D) = E_{\theta^\prime \sim \cQ_\theta} \left[ E_{(x,y) \sim D} \left[ \ell(y, \phi(x,\theta^\prime)) \right] \right]~
	\end{equation*}
	are respectively the expected training and generalization error of the Bayesian model $\theta \sim \cQ_{\theta}$. 
\end{restatable}

\begin{figure*}[t] 
\centering
 \subfigure[Diagonal Element of $H_f(\theta_t)$.]{
 \includegraphics[width = 0.48 \textwidth]{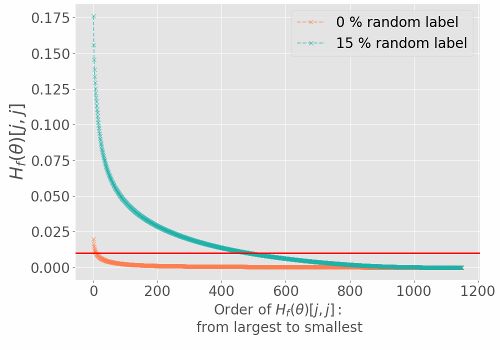}
 } 
 \subfigure[Effective Curvature.]{
 \includegraphics[width = 0.48 \textwidth]{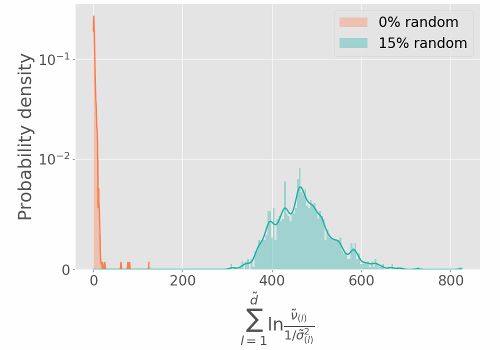}
 } 
 \subfigure[Precision Weighted Frobenius Norm.]{
 \includegraphics[width = 0.48 \textwidth]{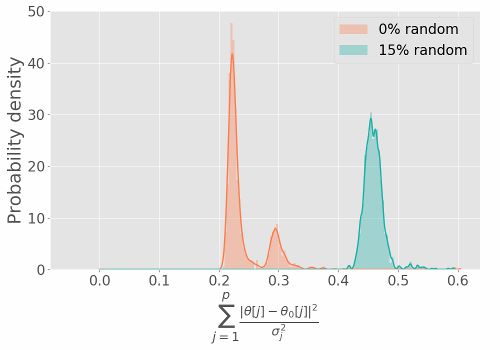}
 } 
 \subfigure[Scale-invariant Generalization Bound.]{
 \includegraphics[width = 0.48 \textwidth]{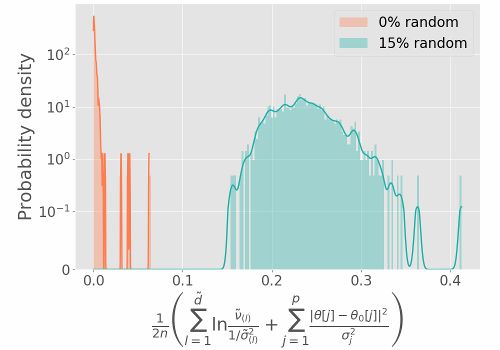}
 }
\caption[]{Gauss-10, batch size 5: Distributions from 10,000 runs. Note that (b), (c) and (d) are scale-invariant based on the analysis of Corollary \ref{corr: KPQ}. The plots indicate that with randomness increased from $0\%$ to $15\%$, the distribution of the effective curvature, precision weighted Frobenius norm and generalization error shift to a higher value, which validates the proposed bound.}
\label{fig:bound}
\end{figure*}
We step through each of the terms  in \eqref{eq: bound1} to gain insights into the bound. The first term, referred to as {\em effective curvature}, measures (in log scale) how the posterior precision (inverse variance) measured based on the diagonal elements of the Hessian cross a threshold based on the prior precision.
This term implies that only dimensions with high curvature that cross certain thresholds contribute to the generalization error. The effective curvature term implies that low curvature or `flat valley' models, where few $H_f(\theta)[j,j]$ cross the threshold $1/\sigma_j^2$ and only by a small amount, have small generalization error; on the other hand, high curvature models, where several $H_f(\theta)[j,j]$cross the threshold $1/\sigma_j^2$ or a few cross the threshold by a large amount, have larger generalization error. 

At a high level, this term captures similar qualitative dependencies as in recent advances in spectrally normalized bounds~\cite{bapf17, nebb17}, but with a more explicit dependence on the curvature base on a computable quantity: the diagonal of the Hessian.  In our notation, spectrally normalized bounds characterize the perturbation \cite{nebb17, vazk19} $f(\tilde \theta) - f(\theta) \approxeq (\tilde \theta - \theta)^T H_f(\theta) (\tilde \theta - \theta)$. Instead of using the Hessian, the existing advances have focused on uniform bounds on such perturbations in terms of the spectral and related norms of the layerwise weight matrices \cite{bapf17, nebb17}. Our results suggest that it is possible to get qualitatively similar but arguably more intuitive bounds by focusing on the structure of the Hessian. 

The second term is the precision weighted Frobenius norm, which measures the distance of the parameter $\theta$ from initialization $\theta_0$ weighted by the prior precision. The closer the parameter stays to the initialization, the smaller the term will be, implying a smaller generalization error. Similar qualitative dependencies on the Frobenius norm also showed up in recent advances in spectrally normalized PAC-Bayesian bounds \cite{bapf17, nebb17, vazk19} where the layer-wise Frobenius norm is normalized by the layer-wise spectral norm by picking special prior distribution for the PAC-Bayesian bound. 

The above generalization bound explains the generalization jointly in terms
of both the effective curvature and the precision weighted Frobenius norm. In the above generalization bound, the dependence on the prior covariance in the two terms illustrates a trade-off, i.e., a large prior variance $\sigma_j$ diminishes the dependence on $(\theta[j] -\theta_0[j])^2$ but increases the dependence on the effective curvature should dimension $j$ turn out to be a direction with sharp curvature, and vice versa.

The following corollary gives a generalization bound corresponding to the special case of an isotropic prior, with $\sigma_j = \sigma, j = 1,..., p$:
\begin{corr} \label{corr: bound2}
	Let $\mathcal{P} \sim \mathcal{N}(\theta_0, \sigma^2)$ be the prior and $\cQ_\theta \sim \mathcal{N}(\theta, \Sigma_{\cQ_{\theta}} )$ be the posterior, where $\Sigma_{\cQ_{\theta}}$ is defined as Assumption \ref{asmp: p_q}. Let $\tilde d$ be the number that $H_f(\theta)[j,j] > 1/ \sigma^2$ and let the corresponding precision values be  $\{ \tilde \nu_{(1)}, ..., \tilde \nu_{(\tilde d)}\}$ and corresponding thresholds be $\{1/ \tilde\sigma_{(1)}, ... , 1/\tilde \sigma_{(d)}\}$. With probability at least $(1-\delta)$ we have the following scale-invariant generalization bound:	
	\begin{equation}
	\begin{array}{ll}
	KL(\ell(\cQ_\theta, S) || \ell(\cQ_\theta, D))
	& \leq \frac{1}{2(n)} \left( \sum_{l =1}^{\tilde d} \ln \frac{\tilde \nu_{(s)}}{1/\sigma^2} + \frac{\|\theta -\theta_0\|^2}{\sigma^2}  \right) + \frac{\ln\frac{n+1}{\delta}}{n}\\
	& 
	\end{array}
	\end{equation}
	where $\ell(\cQ_\theta,S) = E_{\theta^\prime \sim \cQ_\theta}\left[ \frac{1}{n} \sum_{i=1}^n \ell(y_i, \phi(\x_i,\theta^\prime)) \right]$ and 
	$\ell(\cQ_\theta,D) = E_{\theta^\prime \sim \cQ_\theta} \left[ E_{(x,y) \sim D} \left[ \ell(y, \phi(x,\theta^\prime)) \right] \right]$
	are the training and generalization error of the hypotheses $\theta \sim \cQ_{\theta}$ respectively. 
\end{corr}

To evaluate the proposed generalization bound, we use SGD with isotropic Gaussian initialization (Gaussian prior: $\sigma_i = \sigma$ and $\theta_0 = 0$) to train the aforementioned models on true labeled data and random labeled data. We repeat the training for 10,000 time. Figure \ref{fig:bound} presents the results on Gauss-10 dataset.  Figure \ref{fig:bound}(a) plots the diagonal elements of $H_f(\theta)$, demonstrating that for true labeled data, few $H_f(\theta)[j,j]$ cross $1/\sigma^2$ (red line). For random labeled data, a larger number of $H_f(\theta)[j,j]$ cross $1/\sigma^2$, suggesting a larger generalization error than true labeled data. The large diagonal elements of the Hessian on the random labeled data also implies that the loss surface of the parameter learned from random labeled data is sharper than the one learned from true labeled data. Note that the spectrum of the diagonal elements of $H_f(\theta)$ can change by $\alpha$-scale transformation. However the ratio of the diagonal elements to the corresponding precision does not change, since the scaling on both terms gets canceled. Figure \ref{fig:bound}(b) presents the effective curvature, the first term in the bound. True labeled data has smaller effective curvature than data with random labels in line with the observations in \ref{fig:bound}(a) that fewer $H_f(\theta)[j,j]$ cross $1/\sigma^2$ for true labeled data than random labeled data. Figure \ref{fig:bound}(c) plots the weighted Frobenius norm  $\|\theta\|_F^2/\sigma^2$. As the $\sigma^2$ remains the same for random labeled data and true labeled data, Figure \ref{fig:bound}(c) implies that SGD goes further from the initialization with random labels suggesting larger generalization error than true labeled data. 
Finally, the results in Figure \ref{fig:bound}(d) shows with the randomness increased from $0\%$ to $15\%$, the generalization error shifts to a higher value, which is consistent with the observations in \cite{zhbh17,nebm17} that random labeled data has larger generalization error.  Note that Figure \ref{fig:bound}(b), (c) and (d) are scale-invariant based on the analysis of Corollary \ref{corr: KPQ}. Additional results can be found in Appendix \ref{app:general}.

\section{Conclusions}
\label{sec:conc}
\vspace*{-3mm}

In this paper, we empirically and theoretically study the dynamics and generalization of SGD for deep nets based on the Hessian of the loss. We find that the primary subspace of the second moment of SGs overlaps substantially with that of the Hessian, although the matrices are not equal. Thus, SGD seems to be picking up and using second order information of the loss. We empirically study the SGD dynamics and present large deviation bounds for the change in loss at each step characterized as a martingale sequence. We also characterize the convergence of SGD to a stationary point with adaptive step sizes as well as preconditioning. From a stochastic process perspective, such adaptivity makes the dynamics a super-martingale. We develop a scale-invariant PAC-Bayesian generalization bound where the anisotropic posterior depends on the Hessian at minima in an intuitive manner, e.g., flat directions of the Hessian have large variance in the posterior. 

\section*{Acknowledgments}
\label{sec:ack}
The research was supported by NSF grants IIS-1563950, IIS-1447566, IIS-1447574, IIS-1422557, CCF-1451986, IIS-1029711.

\bibliography{SGD_for_Deep_Nets_Dynamics_and_Generalization}
\bibliographystyle{plain}

\appendix
\section{Experimental Setup}

We perform experiments on the fully connected feed-forward network with Relu activation. All experiments on synthetic datasets have been run on a 56-core Intel\textsuperscript \textregistered~CPU @ 2.40 GHz with 256GB memory, while experiments on real datasets have been performed on a Tesla M4 GPU.

\subsection{Synthetic data with corrupted labels}
In this section, we provide discussions regrading the synthetic dataset, the network architecture, and the training process. The details of setting for each specific experiment have been summarized in Table \ref{table:synthetic}. 

\textbf{Synthetic datasets.} We generated synthetic datasets of size \textit{n} with $k$-class Gaussian blobs where equal number of points is randomly sampled from $k$ Gaussian distribution $\mathcal{N}(\bf{\mu}_k,\mathbb{I})$ with $\bf{\mu}_k$be generated uniformly between -10 and 10 in each dimension (the default setting provided by $sklearn$ \cite{scikit-learn}). 

We form the $k$-class classification problems with different degrees of difficulties by introducing different levels of randomness \textit{r} in labels \cite{zhbh17}. In our context, \textit{r} is the portion of labels for each class that has been replaced by random labels uniformly chosen from $k$ classes. $r = 0$ denotes the original dataset with no corruption, and $r = 1$ means a dataset with completely random labels.

\textbf{Network architecture for Guass-$k$.} The Relu-network has two hidden layers \cite{saeg17} with 10 and 30 hidden units respectively. The input layer of such Relu network is 50-dimensional, and the output layer is $k$-dimensional ($k\geq 2$) with softmax activation. The proposed network has approximately 1,000 parameters (the exact number of parameters can be found in Table \ref{table:synthetic}).

\begin{table}[t!]
\caption{Summary of the setting for each specific experiment with synthetic data. Since $ n=100 \ll p = 902 (1150)$, the 2-layer Relu network is over-parameterized.}
\label{table:synthetic}
\centering
\begin{tabular}{lllll}
\hline
\multicolumn{5}{c}{Data} \\
\hline 
No. of Classes k:            & 2           & 2   & 10  & 10 \\
Input Dimension:            & 50          & 50   & 50  & 50\\
No. of Training Samples n:  & 100         & 100  & 100 & 100\\
Random Labels:              & 0           & 20\% & 0\% & 15\%\\
\hline
\multicolumn{4}{c}{Network Structure}\\
\hline 
No. of Layers:& 2          & 2           & 2  \\
No. of Nodes per Layer:    & {[}10,30{]} & {[}10,30{]} & {[}10,30{]}& {[}10,30{]}\\
No. of Parameters p:         & 902         & 902         & 1150   & 1150      \\
\hline 
\multicolumn{4}{c}{Training Parameters} \\
\hline 
Batch Size m:               & 5,~50  & 5,~50   & 5,~50  & 5,~50 \\
Learning Rate $\eta$:       & 0.1   & 0.05     & 0.1   & 0.1\\
Max Iterations:             & 100    & 3,000   & 400   & 10,000\\
\hline        
\end{tabular}
\end{table}

\textbf{Training.} We use constant step-size SGD to train the Relu network on the above mentioned datasets repetitively 10,000 times to analyze the SGD dynamics, stationary distribution and generalization. For each independent run, we first generate $\theta_0 \sim P(\theta)$ from a Gaussian distribution $\mathcal{N}(\bf0,\mathbb{I})$, than train the Rule network using SGD with constant learning rate $\eta$, batch size $m$ (random samples with replacement) for $T$ iterations until converge. The corresponding $\theta_t$,training loss $f(\theta_t)$, Hessian of the loss $H_f(\theta_t)$, and $M_t$ at each iteration t are recorded. Training till convergence is repeated 10,000 times, and these 10,000 different runs let us compute the empirical distribution of several quantities of interest including  $f(\theta_t), \theta_t, H_f(\theta_t),M_t$ as well as eigen-spectra of $H_f(\theta_t),M_t$ and related matrices.

\subsection{MNIST and CIFAR-10}
We also conduct a series of experiments on two commonly used real datasets: MNIST~\cite{lecun_gradientbased_1998}, and CIFAR-10~\cite{krizhevsky_learning_2009} to demonstrate that, even though in the real-world scenario the problem can be significantly more challenging, observations we have made in the synthetic datasets are still valid.

\textbf{MNIST dataset.} The MNIST dataset contains 60,000 black and white training images, representing handwritten digits 0 to 9. Each image of size $28\times28$ is normalized by subtracting the mean and dividing the standard deviation of the training set and converted into a vector of size 784. 

\textbf{Network architecture for MNIST.} The $d$-hidden layer Relu network, with $d$ varying from 3 to 6, has 128 hidden units at each layer. Each Relu-network has more than 100,000 parameters (see Table \ref{table:real} for details). 

\textbf{CIFAR-10 dataset.} The CIFAR-10 dataset consists of 60,000 color images including 10 categories. 50,000 of them are for training, and the rest 10,000 are for validation/testing purpose. Every image is of size $32\times32$ and has 3 color channels. We first re-scale each image into [0, 1] by dividing each pixel value by 255, then each image is normalized by subtracting the mean and dividing the standard deviation of the training set for each color channel, and finally each image is converted into a vector of size $3072$ $(32\times32\times3)$. 

\textbf{Network architecture for CIFAR-10.} We consider two network architectures: a shallow 3-hidden layer Relu-network, and a deeper 6-hidden layer one. Each network structure has approximately 1 million parameters with 256 nodes at each layer.

{\bf Training.} We use constant step-size SGD to train the Relu network with 4 mini-batch sizes: 64, 128, 256, and 512 on MNIST, and 2 mini-batch sizes: 256 and 512 on CIFAR-10. Each experiment has been repeated 10 times. 

\begin{table}[t!]
\caption{Summary of the setting for each specific experiment with MNIST dataset and CIFAR-10 dataset.}
\label{table:real}
\centering
\begin{tabular}{l|cccc|cc}
\hline 
& \multicolumn{4}{c|}{MNIST} & \multicolumn{2}{c}{CIFAR-10}\\
\hline
No. of Classes k: & \multicolumn{4}{c|}{10} & \multicolumn{2}{c}{10}\\
Input Dimension:  & \multicolumn{4}{c|}{784}& \multicolumn{2}{c}{3072}\\
No. of Training Samples n: & \multicolumn{4}{c|}{60,000} & \multicolumn{2}{c}{50,000}\\
\hline
&\multicolumn{6}{c}{Network Structure}\\
\hline 
No. of Layers:& 3  & 4 & 5 & 6 & 3 & 6  \\
No. of Nodes per Layer: & 128& 128 & 128& 128&256&256\\
No. of Parameters p:& 134,794 &  151,306 & 167,818  & 184,330& 920,842 &1,118,218      \\
\hline 
& \multicolumn{6}{c}{Training Parameters} \\
\hline 
Batch Size m:  & \multicolumn{4}{c|}{[64, 128, 256, 512]} &  \multicolumn{2}{c}{[256, 512]} \\
Learning Rate $\eta$:&0.1 &0.1 &0.1 &0.1 &0.1 &0.1\\
Max Iterations: & 20,00 & 10,000 & 10,000 & 10,000 & 10,000 & 15,000\\
\hline        
\end{tabular}
\end{table}

\newpage
\section{Hessian of the Loss and the Second Moment of SGD}
\label{app:hess}

In this section, we first provide a full derivation of the Hessian decomposition in Proposition \ref{prop:hessian}. Then we present more experimental results about the overlap between the top eigenvectors of the Hessian $H_f(\theta_t)$ and the second moment $M_t$ based on principal angles \eqref{eq:principal_angle}, and additional analysis based on Davis-Kahan perturbation theorem \cite{daka70}.  Finally we discuss potential relationships between the decomposition in Proposition \ref{prop:hessian} and the Fisher Information matrix.

\subsection{Proof of the Proposition \ref{prop:hessian}}

Recall Proposition~\ref{prop:hessian} in Section~\ref{sec:hessian}:
\hessdecomp*

\proof  By definition,
\begin{align*}
\nabla^2 f(\theta_t) & = \frac{1}{n} \sum_{i=1}^n \nabla^2 f(\theta_t; z_i) \\
& =  \frac{1}{n} \sum_{i=1}^n -\frac{\partial^2 \log p(\theta_t; z_i)}{\partial^2 \theta_t} \\
& = \frac{1}{n} \sum_{i=1}^n -\frac{\partial}{\partial \theta_t} \left( \frac{1}{p(\theta_t;z_i)} \frac{\partial p(\theta_t;z_i) }{\partial \theta_t} \right)~\\
& = \frac{1}{n} \sum_{i=1}^n  \left( \frac{1}{p(\theta_t;z_i)} \right)^2 \frac{\partial p(\theta_t;z_i) }{\partial \theta_t} \frac{\partial p(\theta_t;z_i) }{\partial \theta_t}^T-\frac{1}{n} \sum_{i=1}^n \frac{1}{p(\theta_t;z_i)} \frac{\partial^2 p(\theta_t;z_i) }{\partial \theta_t^2} \\
& = \frac{1}{n} \sum_{i=1}^n  \frac{\partial \log p(\theta_t;z_i) }{\partial \theta_t} \frac{\partial \log p(\theta_t;z_i) }{\partial \theta_t}^T-\frac{1}{n} \sum_{i=1}^n \frac{1}{p(\theta_t;z_i)} \frac{\partial^2 p(\theta_t;z_i) }{\partial \theta_t^2} \\
& = \Sigma_t + \mu_t \mu_t^T - H_p(\theta_t)\\
& = M_t - H_p(\theta_t)~.
\end{align*}
That completes the proof. \qed

\newpage

{\bf Full eigen-spectrum.} Here we present the full eigen-spectrum of $H_f(\theta_t)$, $M_t$, and the residual term $H_p(\theta_t)$ for networks trained on Gauss-10 dataset with large batches (50/100), and Gauss-2 dataset with both small (10/100) and large (50/100) batches. Figure \ref{fig:full_spectrum_k10_b50} to \ref{fig:full_spectrum_k2} show the results at the first, one intermediate, and the last iteration.

\begin{figure}[H]
\vspace{5mm}
\centering
  \subfigure[Gauss-10, batch size 50, 0\% Random labels.]{
  \includegraphics[width = 0.9\textwidth]{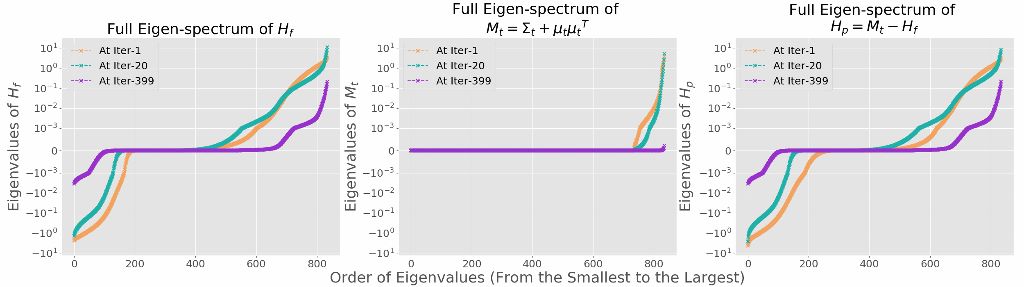}
  }
 \subfigure[Gauss-10, batch size 50, 15\% Random labels.]{
 \includegraphics[width = 0.9 \textwidth]{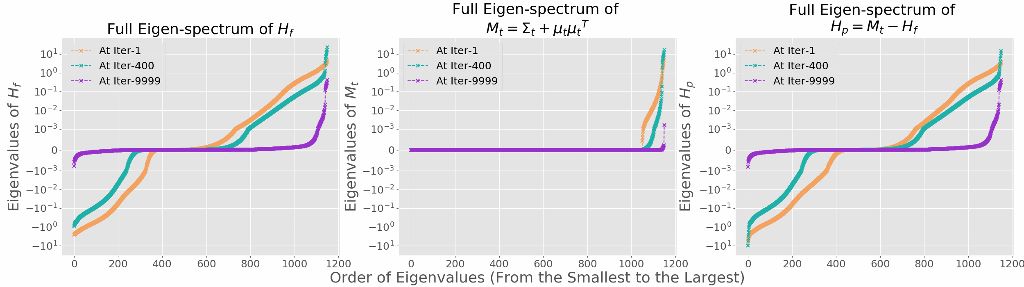}
 }
 \caption[]{Eigen-spectrum dynamics of $H_f(\theta_t)$ (left),  $M_t$ (middle), and $H_p(\theta_t)$ (right) for Gauss-10 dataset trained with large batches containing half of training samples (50/100). $H_p(\theta_t)$ remains significant even after SGD converges, and is close to $-H_f(\theta_t)$.}
 \label{fig:full_spectrum_k10_b50}
\end{figure}

\begin{figure}[p!]
\vspace{-5mm}
\centering
 \subfigure[Gauss-2, batch size: 5, 0\% random labels.]{
 \includegraphics[width = 0.9 \textwidth]{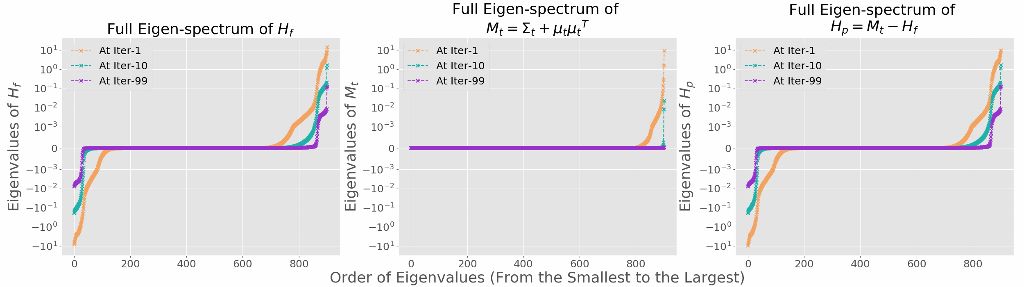}
 } 
 \subfigure[Gauss-2, batch size: 5, 20\% random labels.]{
 \includegraphics[width = 0.9 \textwidth]{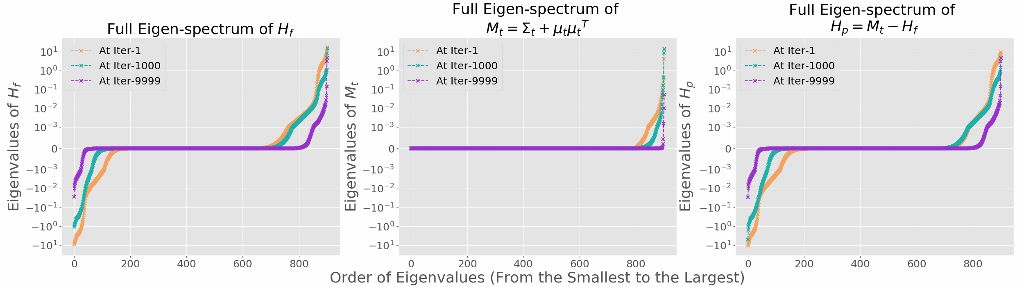}
 }
\subfigure[Gauss-2, batch size: 50, 0\% random labels.]{
 \includegraphics[width = 0.9 \textwidth]{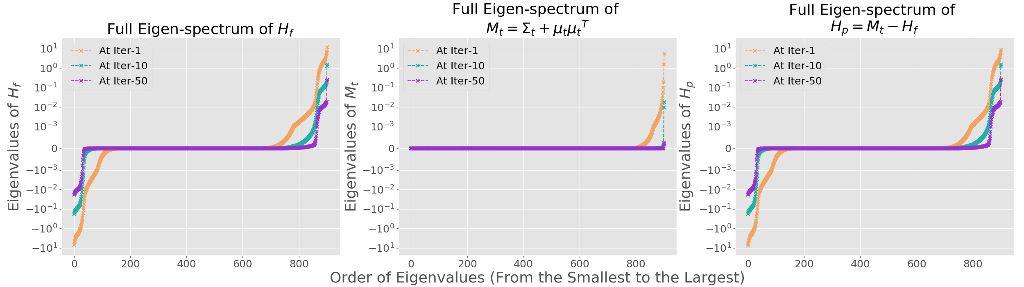}
 } 
 \subfigure[Gauss-2, batch size: 50, 20\% random labels.]{
 \includegraphics[width = 0.9 \textwidth]{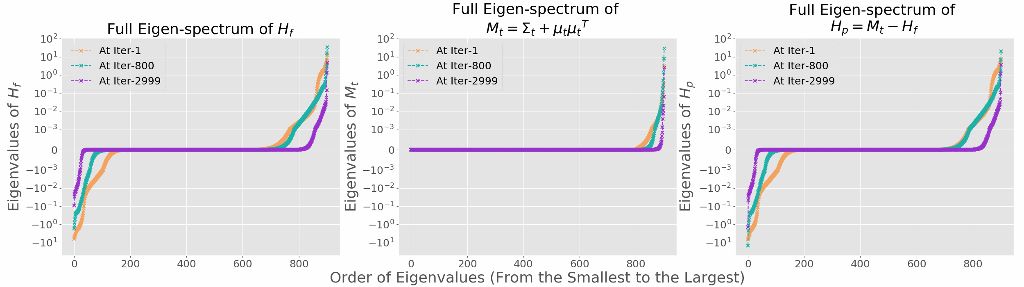}
 }
 \caption[]{Eigen-spectrum dynamics of $H_f(\theta_t)$ (left),  $M_t$ (middle), and $H_p(\theta_t)$ (right) for Gauss-2 dataset. (a) and (b): small batches containing one twentieth of training samples (5/100); (c) and (d): large batches containing half of training samples (50/100). $H_p(\theta_t)$ remains significant even after SGD converges, and is close to $-H_f(\theta_t)$.}
 \label{fig:full_spectrum_k2}
\vspace*{-3mm}
\end{figure}

\newpage

\subsection{Top subspaces: Hessian and Second Moment}
In Section~\ref{sec:B2-principal-angle}, we provide additional experimental results about the overlap between the top eigenvectors of the Hessian $H_f(\theta_t)$ and the second moment $M_t$ based on principal angles \eqref{eq:principal_angle}. Then in Section~\ref{sec:B2-davis-kahan}, we present supplemental analysis based on Davis-Kahan perturbation theorem \cite{daka70}.

\subsubsection{Principal Angles} \label{sec:B2-principal-angle}
We provide additional results for networks trained on Gauss-10 dataset with large batches (Figure~\ref{fig:principal_angle_k10_b50}), which contain half of the training samples (50/100), and Gauss-2 dataset with both small (one twentieth of the training samples) (Figure~\ref{fig:principal_angle_k2} (a) and (b)) and large batches (Figure~\ref{fig:principal_angle_k2} (c) and (d)).

\begin{figure}[H]
\centering
 \subfigure[Gauss-10, batch size 50, 0\% Random labels.]{
 \includegraphics[width = 0.97 \textwidth]{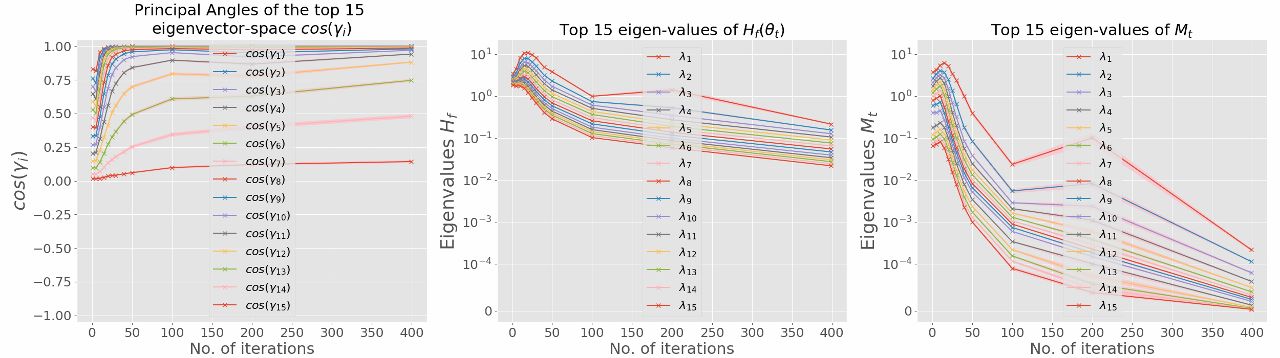}
 }
 \subfigure[Gauss-10, batch size 50, 15\% Random labels.]{
 \includegraphics[width = 0.97 \textwidth]{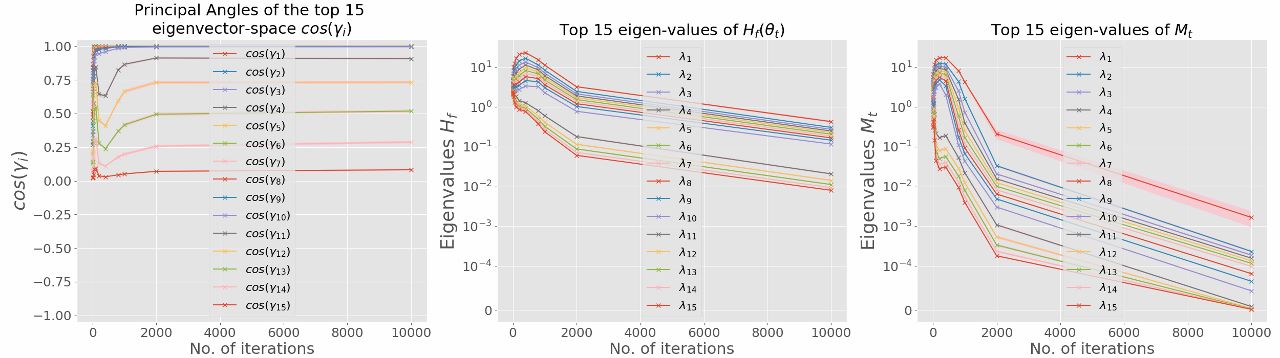}
 }
 \vspace{-4mm}
 \caption[]{Dynamics of principal angles of top 15 eigenvector space between $H_f(\theta_t)$ and $M_t$ (Left) and top 15 eigenvalues dynamics of $H_f(\theta_t)$ (Middle) and $M_t$ (Right) for Gauss-10 dataset trained with large batches containing half of the training samples (50/100). $\cos(\gamma_1)$ to $\cos(\gamma_{10}) \approxeq 1$, indicating the top 10 principal subspaces are well aligned.}
 \label{fig:principal_angle_k10_b50}
\end{figure}
\vfill

\afterpage{
\begin{figure}[t!]
\centering
 \subfigure[Gauss-2, batch size: 5, 0\% random labels.]{
 \includegraphics[width = 0.9 \textwidth]{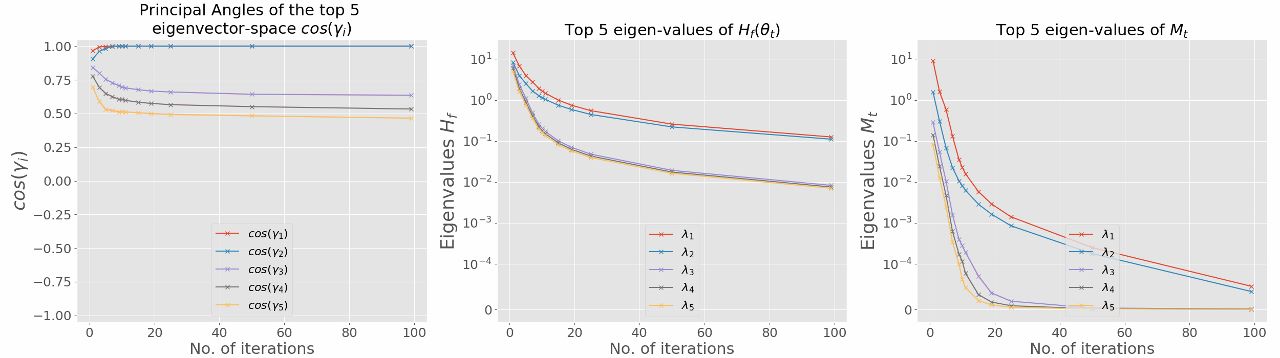}
 }  \vspace{-2mm}
 \subfigure[Gauss-2, batch size: 5, 20\% random labels.]{
 \includegraphics[width = 0.9 \textwidth]{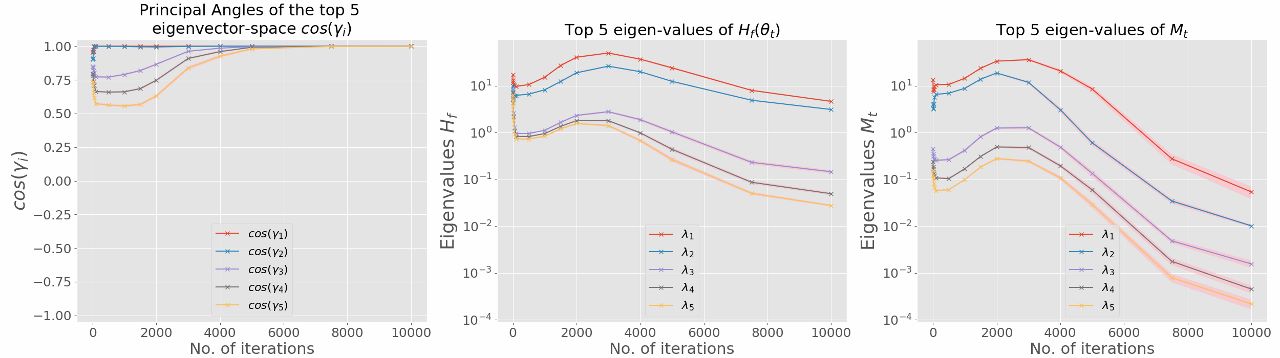}
 } \vspace{-2mm}
 \subfigure[Gauss-2, batch size: 50, 0\% random labels.]{
 \includegraphics[width = 0.9 \textwidth]{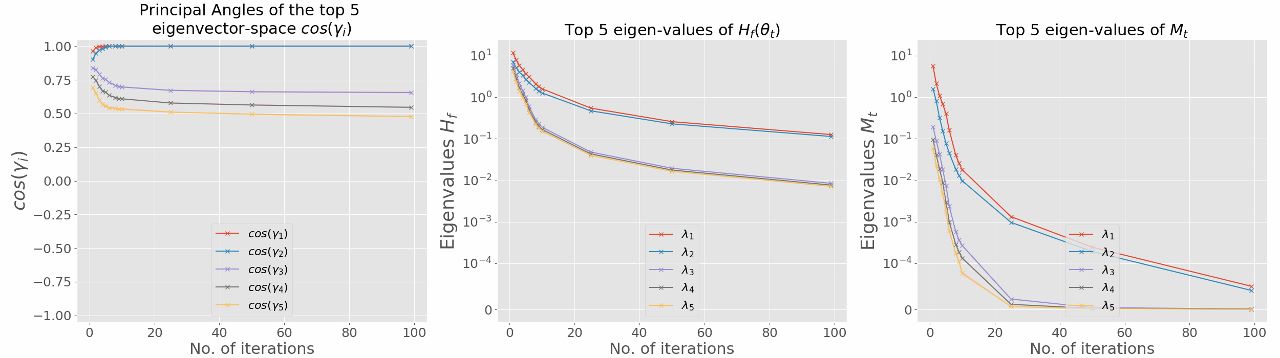}
 }  \vspace{-2mm}
 \subfigure[Gauss-2, batch size: 50, 20\% random labels.]{
 \includegraphics[width = 0.9 \textwidth]{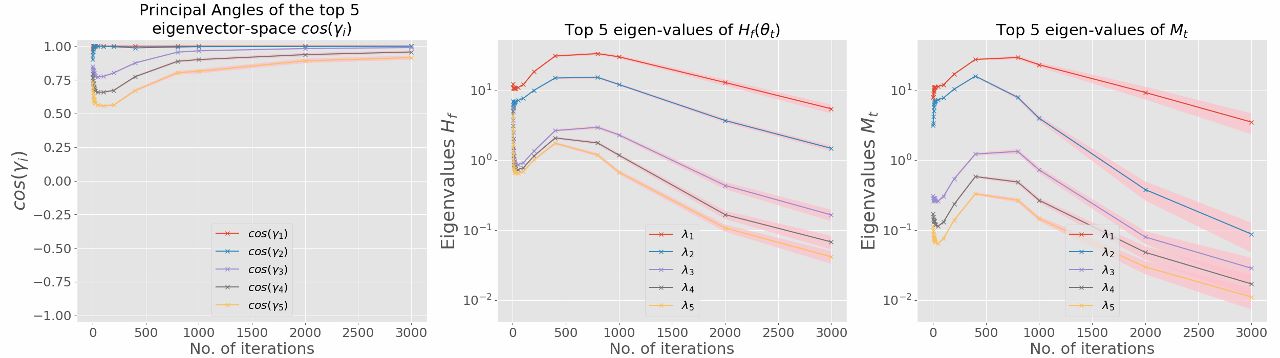}
 }
 \caption[]{Dynamics of principal angles of top 5 eigenvector space between $H_f(\theta_t)$ and $M_t$ (Left) and top 5 eigenvalues dynamics of $H_f(\theta_t)$ (Middle) and $M_t$ (Right) for Gauss-2 dataset. (a) and (b): small batches containing one twentieth of training samples (5/100); (c) and (d): large batches containing half of training samples (50/100). $\cos(\gamma_1)$ to $\cos(\gamma_{2}) \approxeq 1$, indicating the top 2 principal subspaces are always well aligned.}
 \label{fig:principal_angle_k2}
\end{figure}
\clearpage
}

\newpage

\begin{figure}[t!]
\vspace{-5mm}
\centering
\subfigure[Gauss-10, batch size 5, 0\% random labels.]{
\includegraphics[width = 0.48 \textwidth]{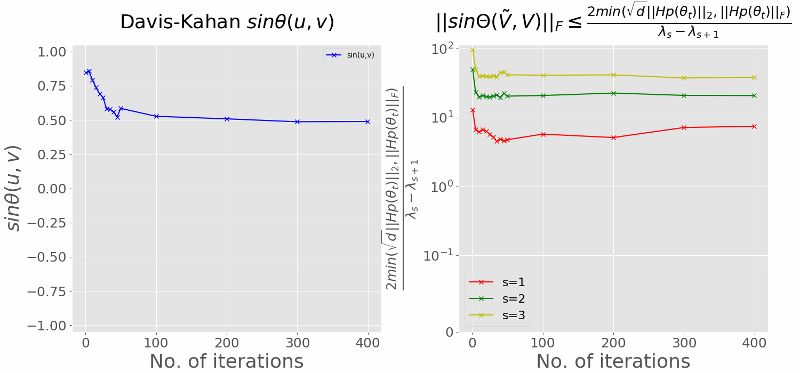}}\vspace{-1mm}
 \subfigure[Gauss-10, batch size 5, 15\% random labels.]{
 \includegraphics[width = 0.48 \textwidth]{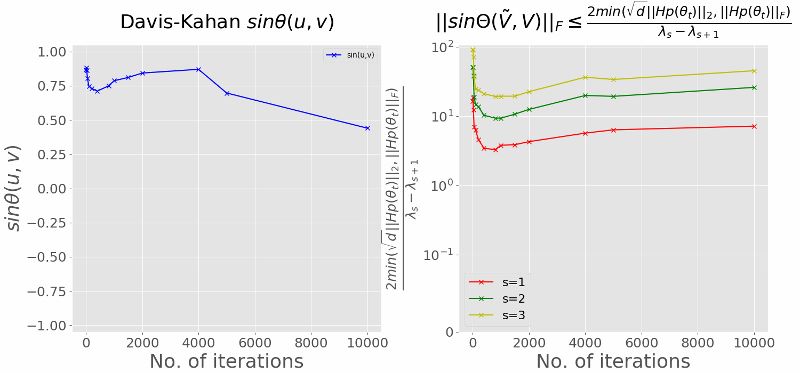}
 }\vspace{-1mm}
 \subfigure[Gauss-10, batch size 50, 0\% random labels.]{
 \includegraphics[width = 0.48 \textwidth]{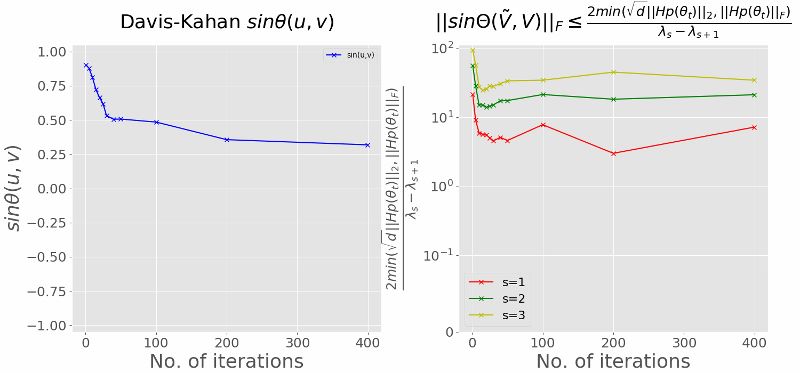}
 }\vspace{-1mm}
 \subfigure[Gauss-10, batch size 50, 15\% random labels.]{
 \includegraphics[width = 0.48 \textwidth]{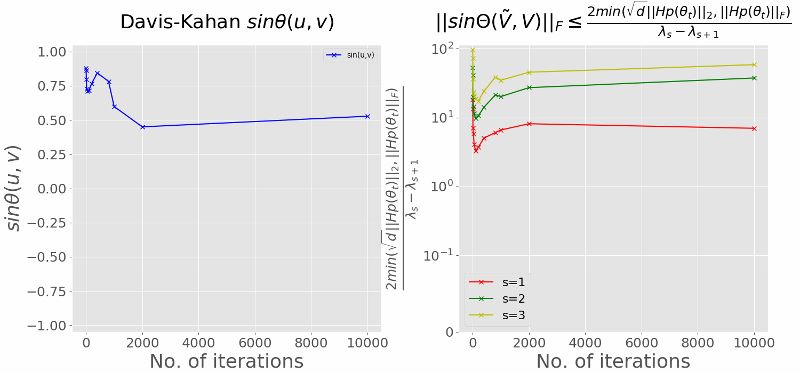}
 }
\caption[]{Davis-Kahan $\sin \theta(u, v) = \sqrt{1 - (u^Tv)^2}$ for Gauss-10 dataset. (a) and (b): small batches containing one twentieth of the training samples (5/100); (c) and (d): large batches containing half of the training samples (50/100). $\sin \theta(u, v) > 0.5$ serves as an extra evidence to suggest the primary subspaces spanned by the top eigen-vectors of $M_t$ and $H_f(\theta_t)$ significantly overlaps as training proceeds. However, the computed upper bounds are too loose to be useful.}
\label{fig:D-K-k10}
\end{figure}

\begin{figure}[t!]
\vspace{-4mm}
\centering
\subfigure[Gauss-2, batch size 5, 0\% random labels.]{
 \includegraphics[width = 0.48 \textwidth]{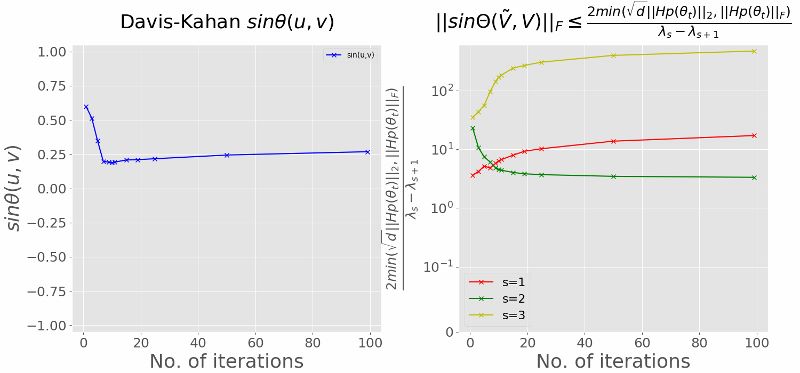}
 }
 \subfigure[Gauss-2, batch size 5, 20\% random labels.]{
 \includegraphics[width = 0.48 \textwidth]{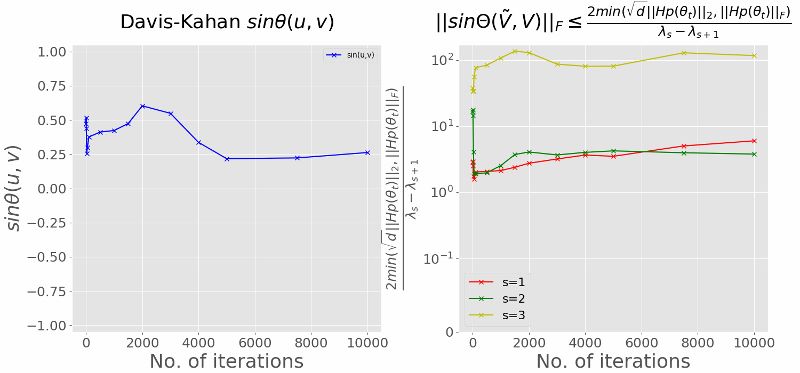}
 }
\subfigure[Gauss-2, batch size 50, 0\% random labels.]{
 \includegraphics[width = 0.48\textwidth]{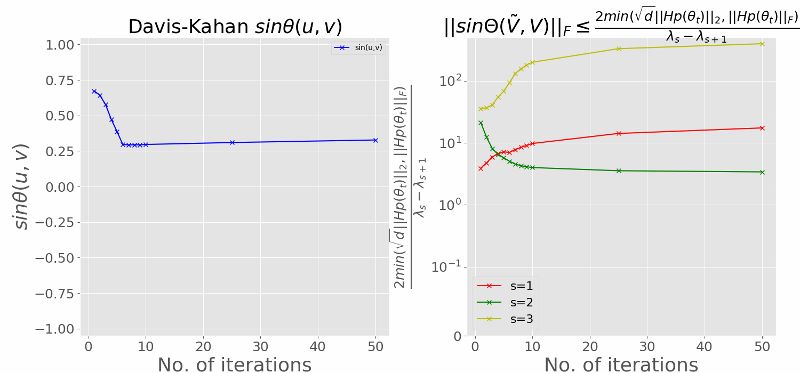}
 }
 \subfigure[Gauss-2, batch size 50, 20\% random labels.]{
 \includegraphics[width = 0.48\textwidth]{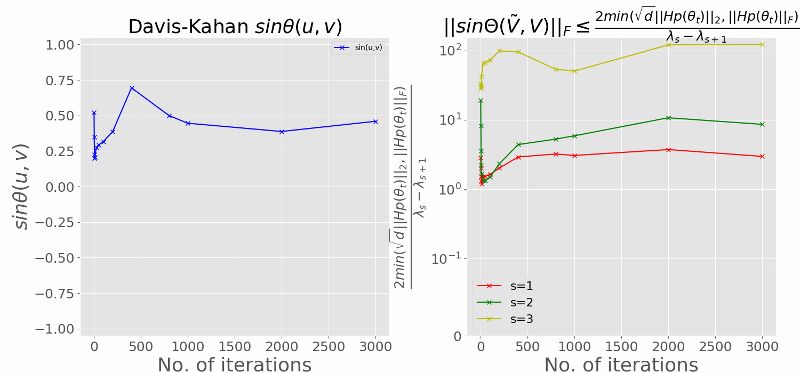}
 }
\caption[]{Davis-Kahan $\sin \theta(u, v) = \sqrt{1 - (u^Tv)^2}$ for Gauss-2 dataset. (a) and (b): small batches containing one twentieth of the training samples (5/100); (c) and (d): large batches containing half of the training samples (50/100). $\sin \theta(u, v) > 0.5$ serves as an extra evidence to suggest the primary subspaces spanned by the top eigen-vectors of $M_t$ and $H_f(\theta_t)$ significantly overlaps as training proceeds. However, the computed upper bounds are too loose to be useful.}
 \label{fig:D-K}
\end{figure}

\subsubsection{Davis-Kahan $\sin\theta$ theorem} \label{sec:B2-davis-kahan}

In this part, we introduce some matrix perturbation theories to bound the angles between top eigenvectors of $H_f$ and $M_t = \Sigma_t + \mu_t \mu_t^T$. 

Let
\beq
Q = P + \Delta~,
\eeq
where $\Delta$ is a symmetric matrix. We will refer to $\Delta$ as the perturbation. In the setting of deep-nets, let
\beq
P = H_f(\theta_t) \qquad Q = M_t \qquad \Delta = H_p(\theta_t)~.
\eeq

Note that we are treating $P$ as the true matrix, and $Q$ as the perturbation, but their roles can be reversed by using $-\Delta$ as the perturbation. Further, note that in deep-nets $(P,Q,\Delta)$ are dynamic, and we can study how the similarity between the eigen-spaces evolve over time.

Let $\lambda_1 \geq \ldots \geq \lambda_p$ be the eigen-values of $P$ (true matrix, $H_f(\theta_t)$) with corresponding eigen-vectors $\v_1,\ldots,\v_p$; further, let $\tilde{\lambda}_1 \geq \ldots \geq \tilde{\lambda}_p$ be the eigen-values of $Q$ (perturbed matrix, $H_p(\theta_t)$) with corresponding eigen-vectors $\tilde \v_1, \ldots, \tilde{\v}_p$. Let $\theta_r$ be the angle between $\v_r$ and $\tilde \v_r$. Let
\beq
\delta_r = \min_{s \neq r} ~| \tilde \lambda_s - \lambda_r |~.
\eeq
Then, the Davis-Kahan theorem \cite{daka70} says:
\beq
\sin \theta_r \leq \frac{ \| \Delta \|_2 }{\delta_r}~.
\eeq
Instead of a result per principal angle, we can use D-K on the subspace. We start by considering the subspace version of D-K. Fix $r, s$ such that $1 \leq r \leq s \leq p$ and let $d \triangleq s - r + 1$.
For the analysis, $d$ will serve as the dimensionality of the subspace of interest. Further, assume that $\min(\lambda_{r-1} - \lambda_r, \lambda_s - \lambda_{s+1}) > 0$. Let $V = (\v_r~ \cdots~ \v_s) \in \R^{p \times d}$ and $\tilde V = (\tilde \v_r ~ \cdots ~ \v_s) \in \R^{p \times d}$. Then, we have
\beq
\| \sin \Theta (\tilde V, V) \|_F \leq
\frac{ 2 \min (\sqrt{d} \| \Delta \|_2, \| \Delta \|_F )}{ \min(\lambda_{r-1} - \lambda_r, \lambda_s - \lambda_{s+1})}~.
\eeq
{\bf Top eigen-space:} Choose $r = 1$, so that for $s < p$, we have
\beq
\min(\lambda_{r-1} - \lambda_r, \lambda_s - \lambda_{s+1}) = \lambda_s - \lambda_{s+1}~.
\eeq
We choose $s$ such that $\lambda_s \geq \alpha$ and $\lambda_{s+1} < \alpha$; in practice, $\alpha$ can be chosen such that eigen-gap $\lambda_s - \lambda_{s+1}$ is significant. Then, we have
\beq
\| \sin \Theta (\tilde V, V) \|_F \leq \frac{ 2 \min( \sqrt{d} \| H_p(\theta_t) \|_2, \| H_p(\theta_t) \|_F ) }{\lambda_s - \lambda_{s+1}}~.
\label{eq:D-K}
\eeq
The above bound can be computed numerically, and we show the dynamics of the bound for $s=1,2,3$ in Figures~\ref{fig:D-K-k10} and \ref{fig:D-K}. The values of the $\sin \theta(u, v)$ for both true and random labeled datasets stay above 0.5 most of the time. In other words, the angle between the eigenvectors corresponding to the largest eigenvalues of $M_t$ and $H_f(\theta_t)$ stays below $30^{\degree}$, which serves as supplemental evidence to support our argument that there is a good amount of overlap between the two subspaces spanned by the top eigenvectors of $M_t$ and $H_f(\theta_t)$. However, the computed upper bounds are significantly above 1, making them not so helpful.

\subsection{Relationship with the Fisher Information Matrix} 
\label{sec:app:fisher}
The expected value of the Hessian of the log-loss goes by another name in the literature: the Fisher Information matrix \cite{leca98}.
We share brief remarks on how the above decomposition in Proposition~\ref{prop:hessian} relates to the Fisher Information matrix but does not quite explain the overlap of the primary subspaces of the Hessian $H_f(\theta_t)$ of the log-loss and the second moment matrix $M_t$. Let us denote $\theta^*$ the true parameter, recall that the Fisher Information matrix \cite{leca98,rao45} is defined as:
\beq
I(\theta^*) \triangleq E_Z \left[ \nabla \log p(\theta^*; Z) \nabla \log p(\theta^*;Z)^T \right]~,
\eeq
where $\nabla \log p(\theta; Z)$ is often referred to as the score function. In the current context, the result of interest is the fact that under suitable regularity conditions \cite{coth06,amna00} the Fisher Information matrix can be written in terms of the expectation of the Hessian of the log-loss, i.e.,
\beq
I(\theta^*) =  -E_Z \left[ \nabla^2 \log p(\theta^*; Z)  \right]~,
\eeq
Starting with $f(\theta; z) = - \log p(\theta; z)$, a direct calculation by chain rule shows:
\begin{align}
E_Z \left[ \nabla^2 f(\theta; Z) \right] & = E_Z \left[ \nabla f(\theta; Z) \nabla f(\theta; Z)^T \right] \nonumber
- E_Z \left[ \frac{1}{p(\theta;Z)} \nabla^2 p(\theta; Z) \right]~\\
\Rightarrow \quad \bar{H}_f(\theta) & = \bar{M_t} - \bar{H}_p(\theta)~, 
\end{align}
where $\bar{\cdot}$ denotes the population expectation corresponding to sample expectations in Proposition \ref{prop:hessian}. In the context of Fisher information matrix, we have 
\begin{align}
\bar{\mu} & = - E_Z[ \nabla \log p(\theta^*; Z)] = 0~, \label{eq:fish1-2} \\
\bar{H}_p(\theta^*) & = E_Z \left[ \frac{1}{p(\theta^*;Z)} \nabla^2 p(\theta^*; Z) \right] = 0~, \label{eq:fish2-2}
\end{align}

where \eqref{eq:fish1-2} follows by assuming an unbiased estimator in the context of statistical estimation \cite{cabe01,wass10}, and \eqref{eq:fish2-2} follows the so-called regularity conditions \cite{coth06,amna00} which allows switching the integral and second derivatives so that
\begin{align*}
E_Z \left[ \frac{1}{p(\theta^*;Z)} \nabla^2 p(\theta^*; Z) \right] 
 & 
  =\int_z \nabla^2 p(\theta^*; Z) dz 
  = \nabla^2 \int_z p(\theta^*; Z) = \nabla^2 1 = 0~.    
\end{align*}
Then, we have 
\begin{align*}
E_Z \left[ \nabla^2 f(\theta^*; Z) \right] & = E_Z \left[ \nabla f(\theta^*; Z) \nabla f(\theta^*; Z)^T \right] \\
\Rightarrow  \quad -E_Z \left[ \nabla^2 \log p(\theta^*; Z)  \right] & = E_Z \left[ \nabla \log p(\theta^*; Z) \nabla \log p(\theta^*;Z)^T \right]~, \\
\Rightarrow  \qquad \bar{H}_f(\theta^*) & = \bar{M}_t~,
\end{align*}
which are both equivalent definitions of the Fisher Information matrix.

\subsection{Examples}
In this subsection, we give detailed derivations of Table \ref{tab:ex:fisher}.
\subsubsection{Least Squares}\label{subsec:ex1}
Let $x_i \in \R^p$ and $y_i \in \R$ for $i=1, 2, \ldots, n$. In this section we focus on the theoretical analysis of SGD for least squares. We assume the probability model is given by:
\[ p (y_i |x_i, \theta ) = \frac{1}{\sqrt{2 \pi} \sigma} e^{- \frac{(x_i^T
		\theta - y_i)^2}{2 \sigma^2}} . \]
Given a sample $z_i$, the stochastic loss function is:
\[ f (\theta ; z_i) = -\log p(y_i|x_i, \theta) = \frac{1}{2 \sigma^2} (x_i^T \theta - y_i)^2 + C, \]
where $C > 0$ is a constant.
		
Let us denote $X = [x_1, x_2, \ldots, x_n]^T$ and $y = [y_1, y_2, \ldots, y_n]^T$, the empirical loss of least squares is given by
\beq
f(\theta) = \frac{1}{2n\sigma^2} \sum_{i=1}^{n} (x_i^T\theta - y_i  )^2 + C = \frac{1}{2n \sigma^2} \|X \theta - y\|_2^2 + C.
\eeq

The gradient of the empirical loss is
\beq
\mu_t = \nabla f(\theta_t) = \frac{1}{n\sigma^2} \sum_{i=1}^{n} x_i (x_i^T\theta_t - y_i  ) = \frac{1}{n \sigma^2} X^T(X\theta - y).
\eeq
The second moment of the stochastic gradient is given by
\[
	M_t = \frac{1}{n} \sum_{i=1}^n \nabla f(\theta_t; z_i)\nabla f(\theta_t; z_i)^T = \frac{1}{n \sigma^4} \sum_{i=1}^n (x_i^T \theta_t - y_i)^2 x_i x_i^T
\]
The Hessian of the empirical loss function is given by
\[
    H_f = \nabla^2 f(\theta) = \frac{1}{n\sigma^2} \sum_{i=1}^n x_i x_i^T = \frac{1}{n\sigma^2} X^T X.
\]
And
\[
    H_p(\theta_t) = M_t - H_f =  \frac{1}{n \sigma^4} \sum_{i=1}^n (x_i^T \theta_t - y_i)^2 x_i x_i^T - \frac{1}{n\sigma^4} \sum_{i=1}^n \sigma^2x_i x_i^T = \frac{1}{n \sigma^4} \sum_{i=1}^{n} [(x_i^T \theta_t - y_i)^2 - \sigma^2] x_i x_i^T.
\]
In $n < p$ case, the optimal solution $\hat{\theta}$ satisfies $X\hat{\theta} = y$. As $\theta_t$ approaches $\hat{\theta}$, we have $H_p(\theta_t)$ approaches
\beq
	H_p(\hat{\theta}) = - \frac{1}{n \sigma^2} X^T X,
\eeq
and the second moment $M_t$ approaches a zero matrix.

\subsubsection{Logistic Regression}\label{subsec:ex2}
Let $x_i \in \R^p$ and $y_i \in \{0, 1\}$ for $i=1, 2, \ldots, n$. The probability model of logistic regression is given by
\[
    p(y_i | x_i, \theta) = p_{\theta}(x_i)^{y_i} (1 - p_{\theta}(x_i))^{1 - y_i}
\]
where $p_{\theta}(x_i) = (1 + \exp{x_i^T \theta})^{-1}$.
Given a sample $z_i$, the stochastic loss function is:
\[ f (\theta ; z_i) = -\log p(y_i|x_i, \theta) = y_i\log (1 + \exp(- x_i^T \theta)) + (1 - y_i) \log (1 + \exp( x_i^T \theta)), \]
The empirical loss of logistic regression is given by
\beq
f(\theta)  = - \frac{1}{n}\sum_{i = 1}^n \log p(y_i| x_i, \theta)
 = \frac{1}{n}\sum_{i=1}^n [y_i\log (1 + \exp(- x_i^T \theta)) + (1 - y_i) \log (1 + \exp( x_i^T \theta)) ].
\eeq
Let us denote $\sigma_{\theta_t}(x_i) = p_{\theta_t}(x_i)(1 - p_{\theta_t}(x_i))$, the gradient $\mu_t$ and Hessian $H_f(\theta_t)$ of empirical loss $f(\theta_t)$ are given by
\[
	\mu_t = \frac{1}{n} \sum_{i=1}^n (p_{\theta_t}(x_i) - y_i)x_i,
\]
and
\[
	H_f(\theta_t) = \frac{1}{n} \sum_{i=1}^{n} p_{\theta_t}(x_i) (1 - p_{\theta_t}(x_i)) x_i x_i^T = \frac{1}{n} \sum_{i=1}^{n} \sigma_{\theta_t}(x_i) x_i x_i^T.
\]

The second moment of the stochastic gradient is given by
\beq
	M_t = \frac{1}{n} \sum_{i=1}^{n} (p_{\theta_t}(x_i) - y_i)^2 x_i x_i^T.
\eeq
And
\[
	\hat{H}_p(\theta_t) = M_t - H_f = \frac{1}{n} \sum_{i=1}^{n} [(p_{\theta_t}(x_i) - y_i)^2 - \sigma_{\theta_t}(x_i)] x_i x_i^T.
\]

\subsection{Computation of the Hessian}
In the following Section ~\ref{sec:B4-1} and \ref{sec:B4-2}, we share some brief discussions regarding the computation aspect of Hessian used in our analysis.

\subsubsection{Relu Network Trained on Gauss-$k$} \label{sec:B4-1}

In our experiments, the exact Hessian and the Second Moment of Relu-networks trained on Gauss-$k$ datasets are directly computed using Autograd \cite{mada15}. Then we evaluate the full eigen-spectrum of $H_f(\theta_t)$ and $M_t$ using numpy. 

Consider a general loss function $l(y^T \hat{y}(\theta))$, where $y \in \{0, 1\}^K$ is the true label and $\hat{y}(\theta)$ is the prediction of our learning algorithm. Our prediction in this paper is given by

\beq
    v^T\sigma(W_{2}\sigma(W_1 x)),~ \theta = \{v, W_2, W_1\},
\eeq
where ${W}_1 \in \R^{m_1 \times m}$, $W_2 \in \R^{m_2 \times m_1}$ are weight matrices. 
Function $\sigma$ is the activation function applied element-wisely and $v \in \R^{m_D \times K}$ denotes the weights of the output layer. Autograd is able to compute $\nabla l(y^T \hat{y}(\theta))$ and $\nabla^2 l(y^T \hat{y}(\theta))$ using calculus rules \cite{rudi76}. Let us denote 
\[
    a_d = W_d \sigma (W_{d-1} \sigma(W_{d-2} \ldots \sigma(W_1 x)))
\]
be the value of the $d$-th layer before activation.
When $\sigma$ is the ReLu function, Autograd computes the first order and second order derivative of $\sigma$ denoted as $\sigma'$ and $\sigma''$ by the following rule:
\[
    \sigma'(a_d)_j = \left\{ \begin{array}{cc}
          1 & a_{d,j} > 0 \\
          \frac{1}{2} & a_{d,j} = 0 \\
          0 & a_{d, j} < 0
    \end{array}\right.
\]
and $\sigma''(a_d) = 0$. 

\subsubsection{Relu Network Trained on MNIST and CIFAR-10} \label{sec:B4-2}

For Relu-networks trained on MNIST and CIFAR-10 dataset, the number of parameters exceeds 100,000. Thus the direct computation of $H_f(\theta_t)$ and $M_t$ is inapplicable. We compute such high-dimensional eigen-spectra based on the Hessian-vector products \cite{Pear94} and the Lanczos algorithm \cite{Lanc50,papy18}.

\section{SGD Dynamics: Additional Experimental Results}
Additional figures for the analysis performed in Section \ref{sec:dynamics} are presented below.

\subsection{ SGD Dynamics} 
Additional SGD dynamics for networks trained on Gauss-10 dataset with large batches, which contain half of the training samples (50/100), and Gauss-2 dataset with both small (one twentieth of the training samples) and large batches are presented below.

\begin{figure}[H]
\centering
 \subfigure[Gauss-10, Batch size 50, 0\% Random labels.]{
 \includegraphics[width = 0.95 \textwidth]{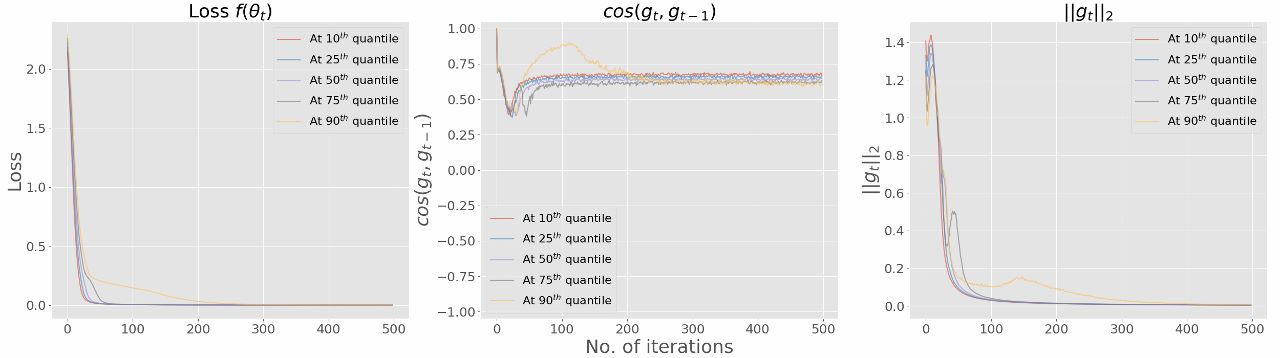}
 }
 \subfigure[Gauss-10, Batch size 50, 15\% Random labels.]{
 \includegraphics[width = 0.95 \textwidth]{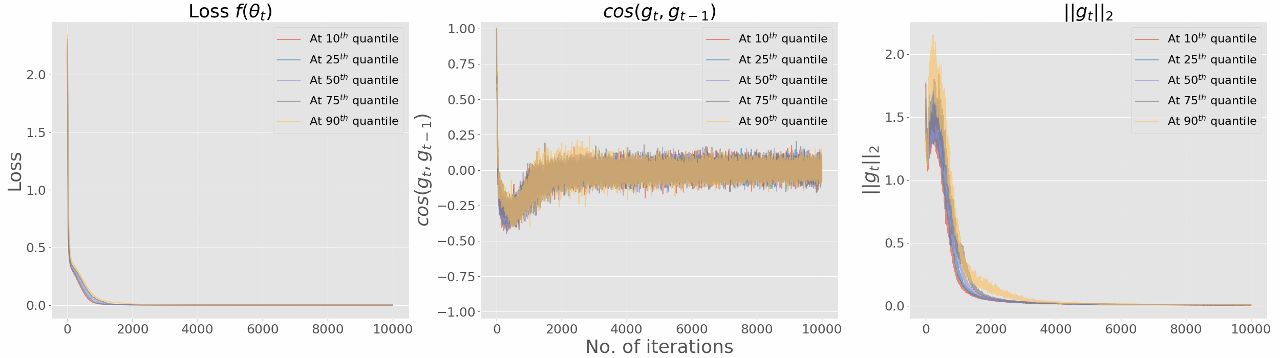}
 }
 \caption[]{Dynamics of the loss $f(\theta_{t})$ (left), the angle of two successive SGs $\cos(g_t,g_{t-1})$ (middle), and the norm of the SGs $\|g_t\|_2$ (right) at different quantiles of $f(\theta_{t})$ for Gauss-10 dataset trained with large batches containing half of training samples (50/100).}
 \label{fig:loss_cos_g2_k10_b50}
\end{figure}

\begin{figure}[p!]
\vspace{-5mm}
\centering
 \subfigure[Gauss-10, batch size 5, 0\% Random labels]{
 \includegraphics[width = 0.9 \textwidth]{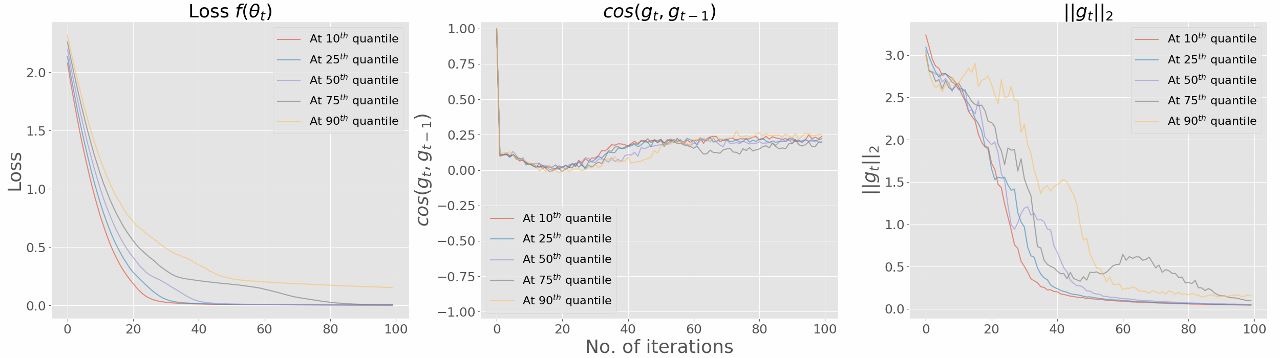}
 } 
 \subfigure[Gauss-10, batch size 5, 15\% random labels]{
 \includegraphics[width = 0.9 \textwidth]{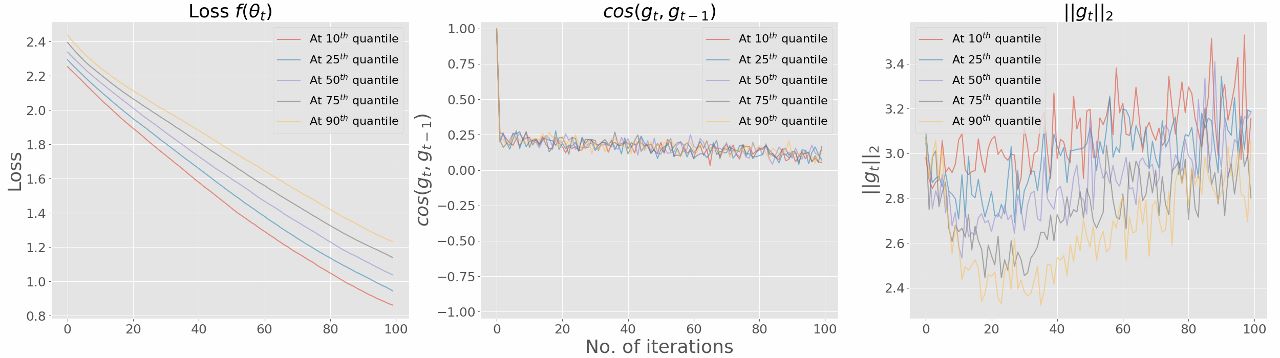}
 }
\subfigure[Gauss-10, Batch size 50, 0\% Random labels.]{
 \includegraphics[width = 0.9 \textwidth]{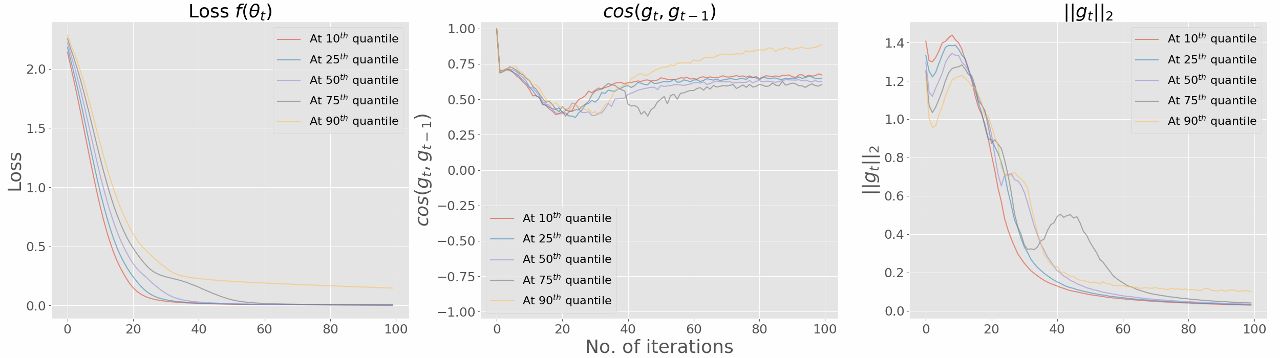}
 } 
\subfigure[Gauss-10, Batch size 50, 15\% Random labels.]{
\includegraphics[width = 0.9 \textwidth]{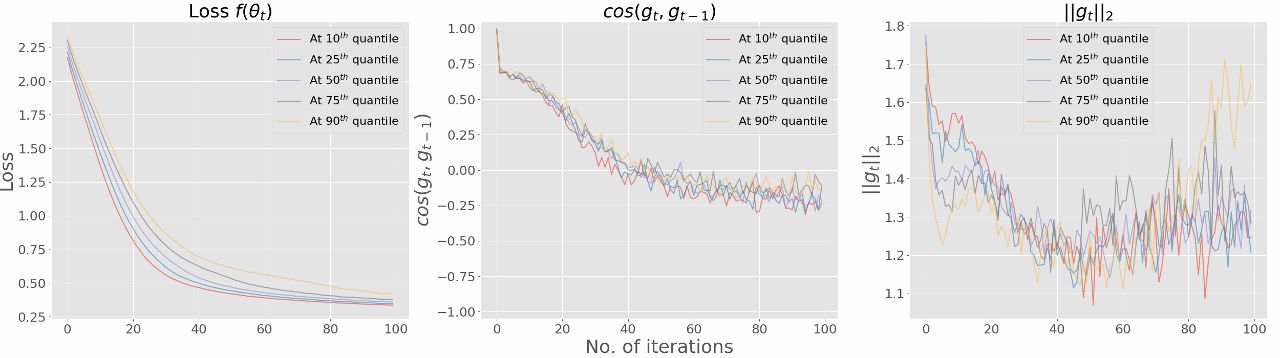}
 }\vspace{-3mm}
\caption[]{Dynamics of the first 100 iterations for the loss $f(\theta_{t})$ (left), the angle of two successive SGs $\cos(g_t,g_{t-1})$ (middle), and the norm of the SGs $\|g_t\|_2$ (right) at different quantiles of $f(\theta_{t})$ for Gauss-10 dataset. (a) and (b): small batches containing one twentieth of training samples (5/100); (c) and (d): large batches containing half of training samples (50/100). The gradient norm always has a decreasing phase at the first few iterations.}
 \label{fig:loss_cos_g2_first100}
\end{figure}

\begin{figure}[p!]
\vspace{-5mm}
\centering
 \subfigure[Gauss-2, batch size: 5, 0\% random labels.]{
 \includegraphics[width = 0.9 \textwidth]{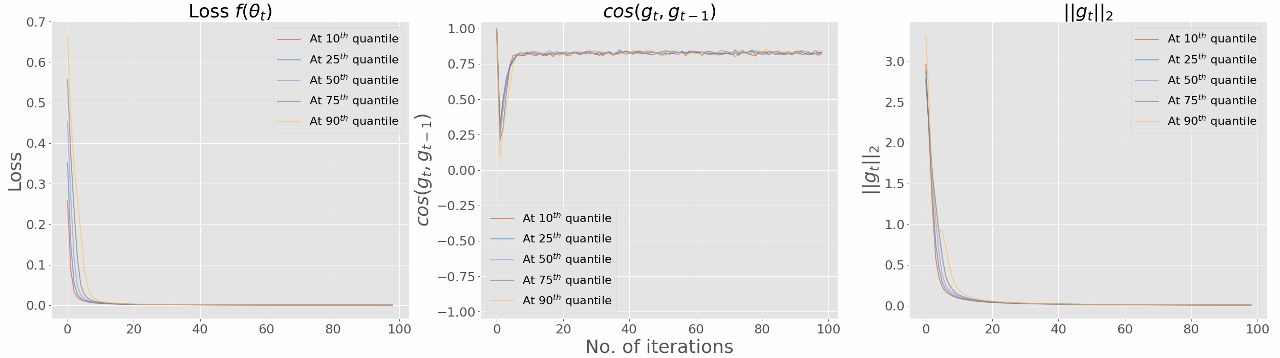}
 } 
 \subfigure[Gauss-2, batch size: 5, 20\% random labels.]{
 \includegraphics[width = 0.9 \textwidth]{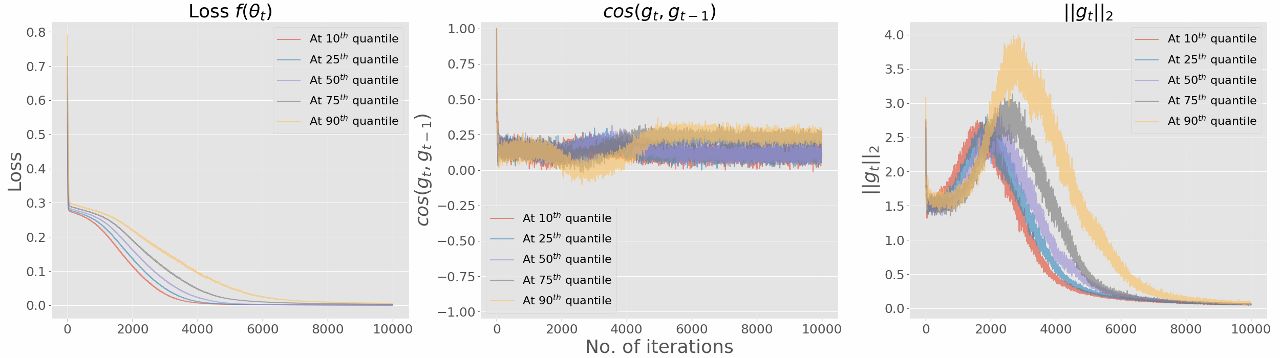}
 } 
  \subfigure[Gauss-2, batch size: 50, 0\% random labels.]{
 \includegraphics[width = 0.9 \textwidth]{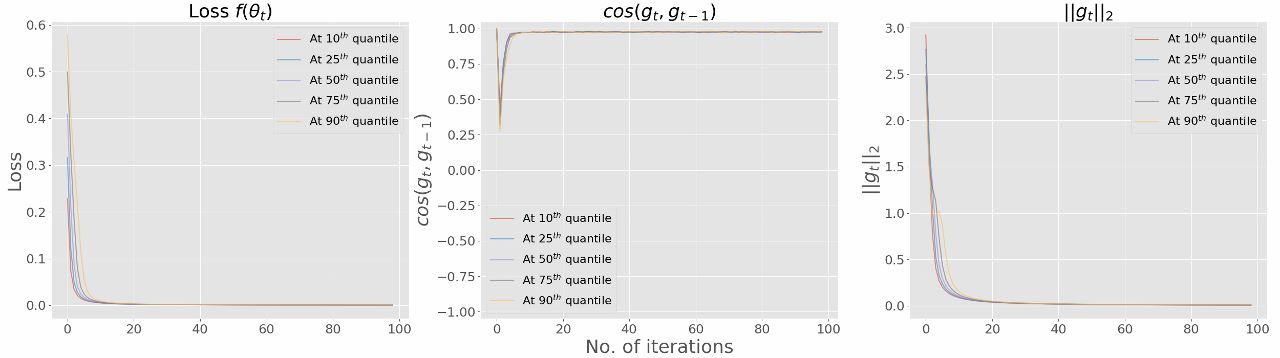}
 } 
 \subfigure[Gauss-2, batch size:50, 20\% random labels..]{
 \includegraphics[width = 0.9 \textwidth]{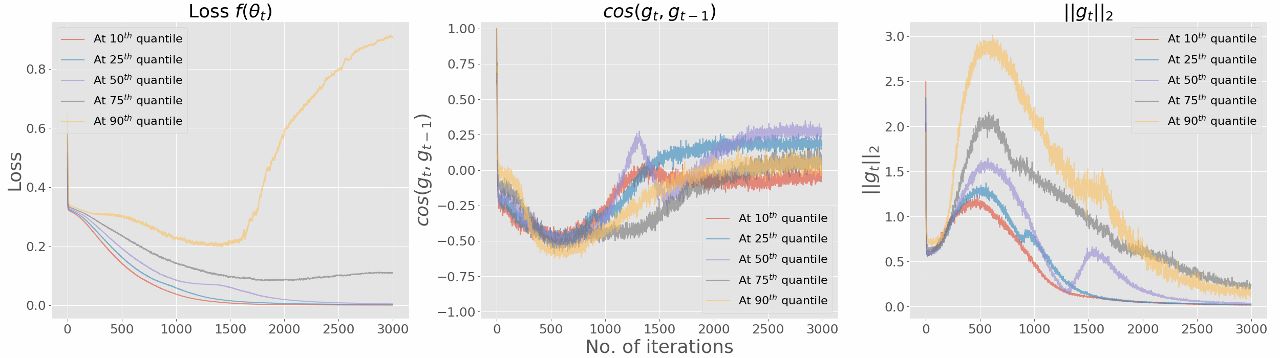}
 }\vspace{-3mm}
 \caption[]{Dynamics of the loss $f(\theta_{t})$ (left), the angle of two successive SGs $\cos(g_t,g_{t-1})$ (middle), and the norm of the SGs $\|g_t\|_2$ (right) at different quantiles of $f(\theta_{t})$ for Gauss-2 dataset. (a) and (b): small batches containing one twentieth of training samples (5/100); (c) and (d): large batches containing half of training samples (50/100).}
 \label{fig:loss_cos_g2_k2}
\end{figure}

\begin{figure}[p!]
\vspace{-5mm}
\centering
\subfigure[MNIST, Network with 3 hidden layers.]{ \includegraphics[width = 0.9 \textwidth]{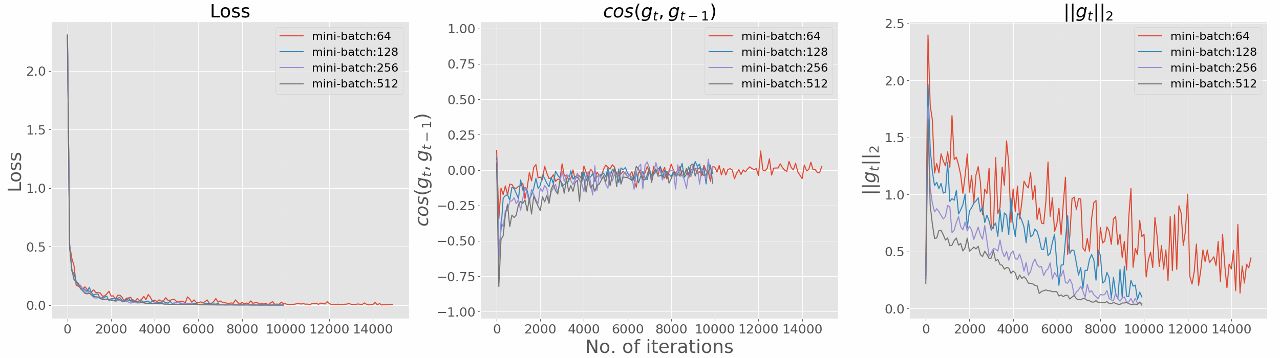}
} 
\subfigure[MNIST, Network with 4 hidden layers.]{
\includegraphics[width = 0.9\textwidth]{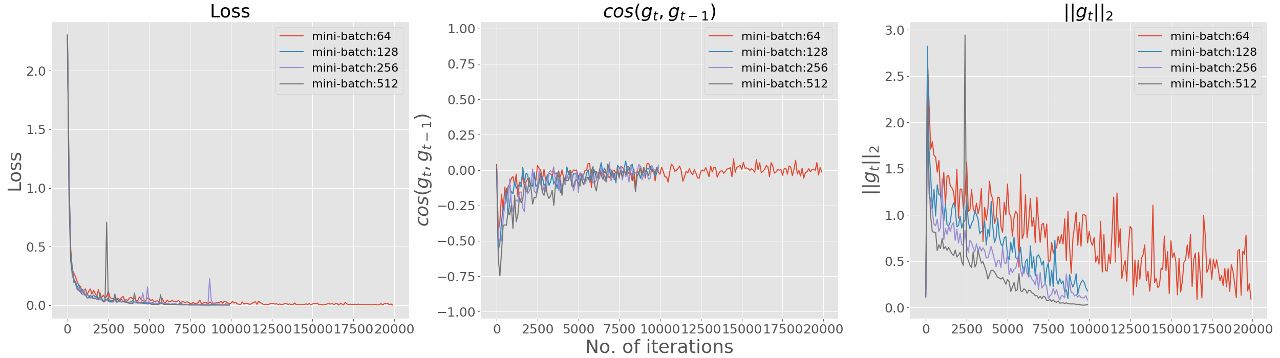}
 }
\subfigure[MNIST, Network with 5 hidden layers.]{
 \includegraphics[width = 0.9 \textwidth]{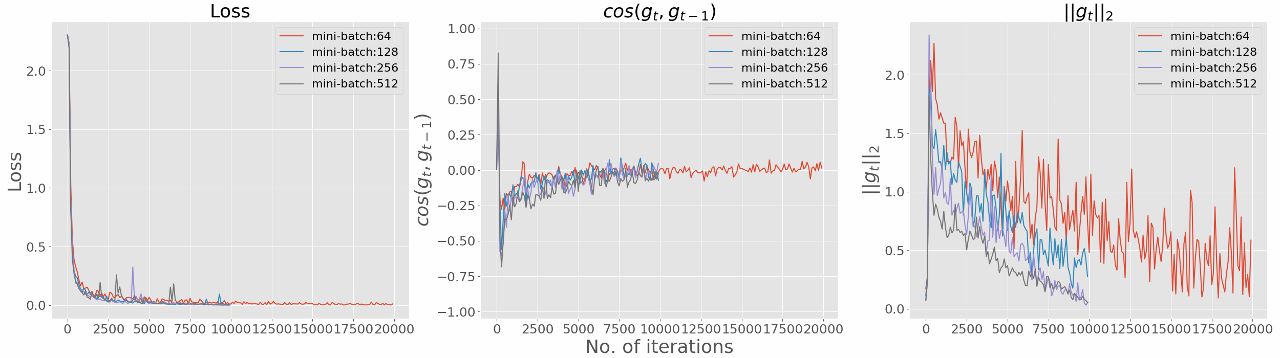}
 } 
 \subfigure[MNIST, Network with 6 hidden layers.]{
 \includegraphics[width = 0.9 \textwidth]{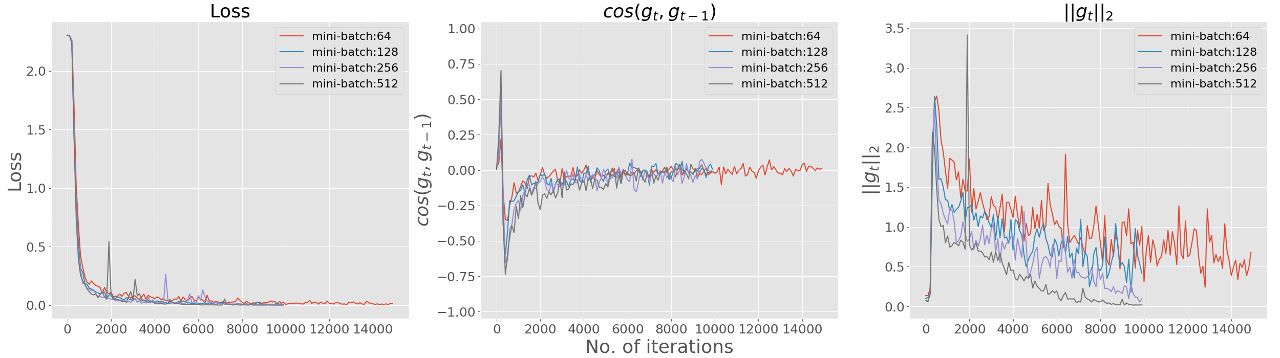}
 } \vspace{-2mm}
\caption[]{Dynamics of the loss $f(\theta_{t})$ (left), the angle of two successive SGs $\cos(g_t,g_{t-1})$ (middle), and the norm of the SGs $\|g_t\|_2$ (right) for the MNIST trained on networks with 3 to 6 hidden layers and various batch sizes. Networks have more than 100,000 parameters. SGD dynamics behave similarly to our Gauss-$k$ datasets with random labels.}
\label{fig:loss_cos_g2_mnist}
\end{figure}

\begin{figure}[p!]
\centering
 \subfigure[CIFAR-10, Network with 3 hidden layers.]{
 \includegraphics[width = 0.9 \textwidth]{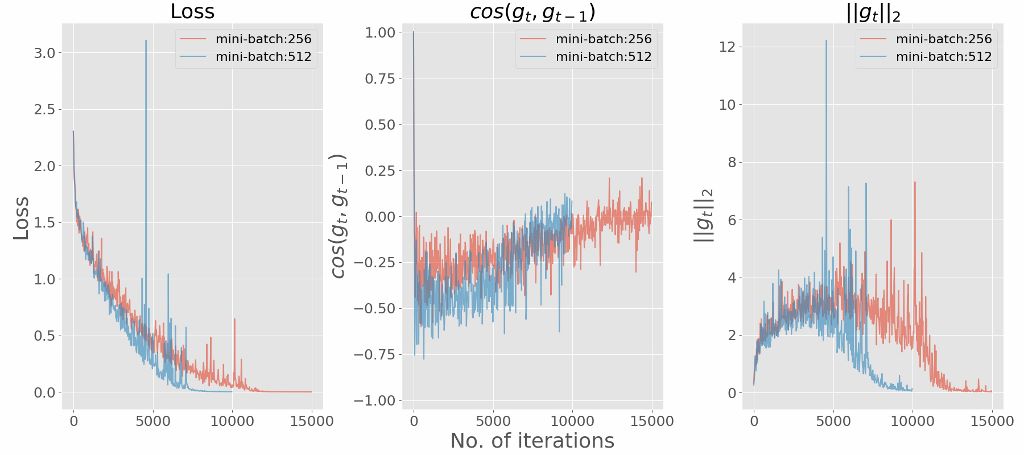}
 } 
\subfigure[CIFAR-10, Network with 6 hidden layers.]{
 \includegraphics[width = 0.9 \textwidth]{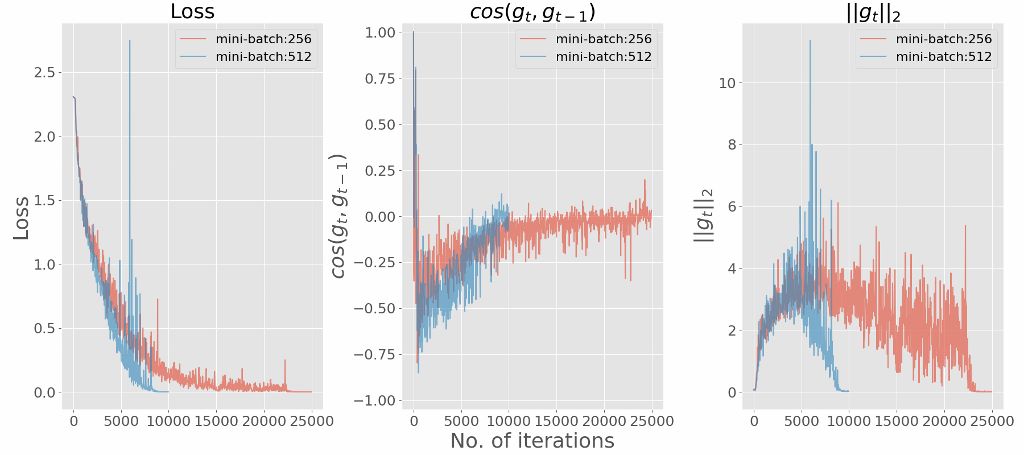}
 }
 \caption[]{Dynamics of the loss $f(\theta_{t})$ (left), the angle of two successive SGs $\cos(g_t,g_{t-1})$ (middle), and the norm of the SGs $\|g_t\|_2$ (right) for the CIFAR-10 dataset with different network architectures and batch sizes. Networks have approximately 1 million parameters. SGD dynamics exhibit similar behavior to Gauss-k datasets.}
 \label{fig:loss_cos_g2_cifar}
\end{figure}

\newpage

\subsection{Loss Difference Dynamics}

Additional loss difference dynamics for networks trained on Gauss-10 dataset with large batches, which contain half of the training samples (50/100), and Gauss-2 dataset with both small (one twentieth of the training samples) and large batches are presented below.

\begin{figure}[b!]
\centering
 \subfigure[Gauss-10, Batch size: 50, 0\% Random labels.]{
 \includegraphics[width = 0.95 \textwidth]{fnn-k10-random0-b50-fnn-loss-distribution.jpg}
 } 
 \subfigure[Gauss-10, Batch size: 50, 15\% Random labels.]{
 \includegraphics[width = 0.95 \textwidth]{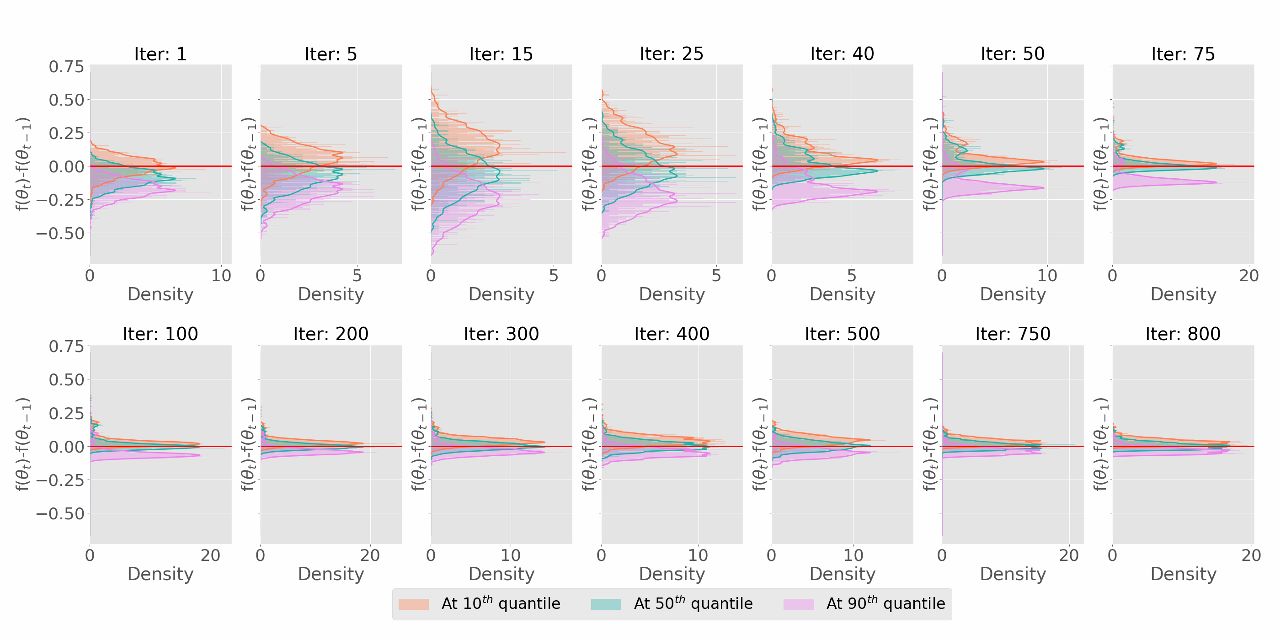}
 }
\caption[]{Dynamics of the distribution $\Delta_t(f)=f(\theta_{t+1}) - f(\theta_t)$ conditioned at $\theta_t$ for Gauss-10 datasets trained with large batches containing half of training samples (50/100). Red horizontal line highlights the value of $\Delta_t(f)=0$. The loss-difference dynamics mainly consist of two phases: 1) the mean of $\Delta_t(f)$ decreases with an increase of variance; 2) the mean of $\Delta_t(f)$ increases and reaches 0 while the variance shrinks.} 
 \label{fig:loss_dynamic_k10}
\end{figure}

\begin{figure}[t!]
\centering
 \subfigure[Gauss-2, Batch size: 50, 0\% Random labels.]{
 \includegraphics[width = 0.95 \textwidth]{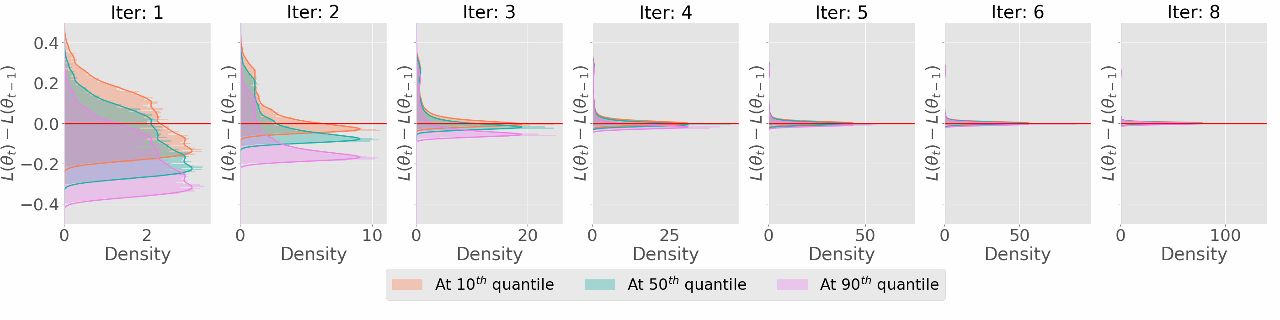}
 } 
 \subfigure[Gauss-2, Batch size: 50, 20\% Random labels.]{
 \includegraphics[width = 0.95 \textwidth]{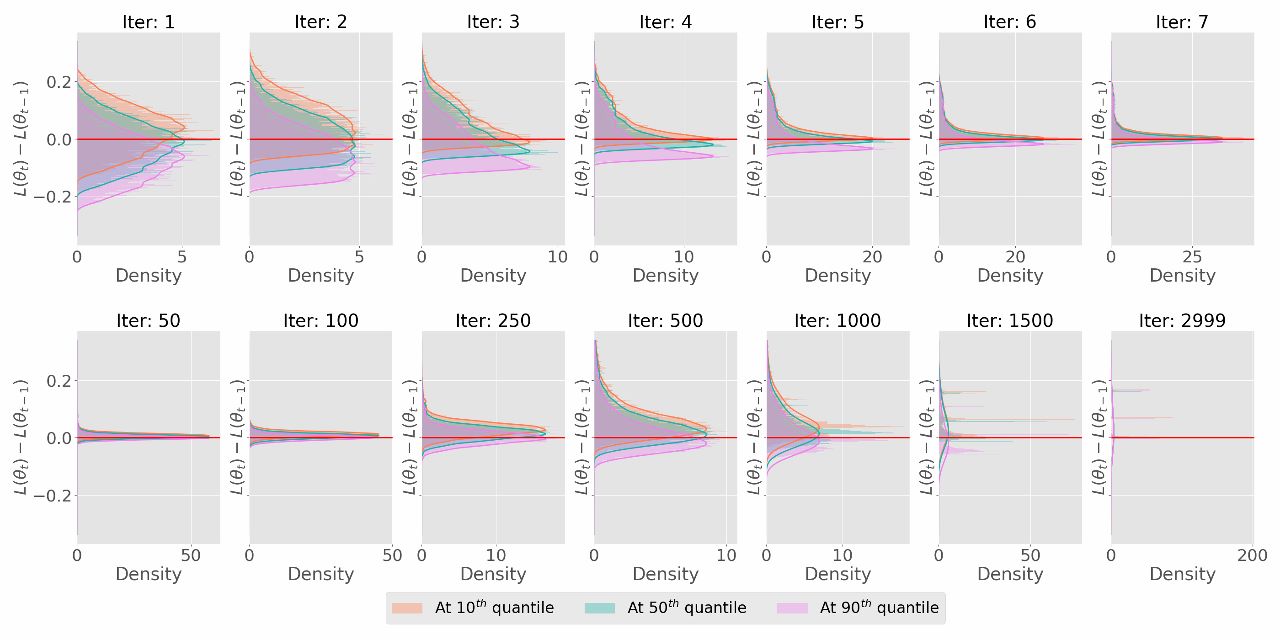}
 }
 \caption[]{Dynamics of the distribution $\Delta_t(f)=f(\theta_{t+1}) - f(\theta_t)$ conditioned at $\theta_t$ for Gauss-2 datasets trained with large batches containing half of training samples (50/100). Red horizontal line highlights the value of $\Delta_t(f)=0$. The loss-difference dynamics changes from two phases to three phases as we increase the difficulty level of the problem.}
 \label{fig:loss_dynamic_k2}
\end{figure}

\newpage
\section{SGD Dynamics: Proof of Theorem \ref{thm:dynamic}}
\label{app:dynam}

\label{sec:app:dynam}
Recall Theorem~\ref{thm:dynamic} in Section \ref{sec:dynamics}:
\theoi*

\proof 
In terms of the empirical loss, we have
\beq
\label{eq:hessian}
   f(\theta_{t+1}) =  f(\theta_t) + (\theta_{t+1} - \theta_t)^T \nabla f(\theta_t)
    + \frac{1}{2} (\theta_{t+1} - \theta_t)^T H_f(\theta_{t+\frac{1}{2}}) (\theta_{t+1} - \theta_t)~,
\eeq
where $\theta_{t+\frac{1}{2}}$ denotes a suitable parameter of the form $(1-\tau) \theta_t + \tau \theta_{t+1}$ which satisfies the above equality by the mean-value theorem \cite{rudi76,bert99}. Replacing $(\theta_{t+1} - \theta_t)$ in \eqref{eq:hessian} using the SGD update \eqref{eq:sgd}, we have
\beq
\label{eq:step1} 
f(\theta_{t+1}) - f(\theta_t) = - \eta_t \nabla \tilde{f}(\theta_t)^T \nabla f(\theta_t) + \frac{\eta_t^2}{2} \nabla \tilde {f(\theta_t)^T} H_f(\theta_{t+\frac{1}{2}}) \nabla \tilde{f(\theta_t)}
\eeq
we can represent the SG as
\begin{equation}
 \nabla \tilde{f}(\theta_t) = \mu_t + \frac{1}{\sqrt{m}}\Sigma_t^{1/2} g~,
\end{equation}

Based on Proposition \ref{prop:hessian}, we have

\begin{equation} \label{eq:emphess}
 H_f(\theta_t) = M_t - H_p(\theta_t)~,   \tag{\ref{eq:hess}}
\end{equation}
where $H_p(\theta_t) =  \frac{1}{n} \sum_{i=1}^n \frac{1}{p(\theta_t;z_i)} \frac{\partial^2 p(\theta_t;z_i)}{\partial \theta^2}$.

From Assumption \ref{asp:rs}, replacing \eqref{eq:generative} and \eqref{eq:emphess} in \eqref{eq:step1}, we get

\beq
\begin{split}
f&(\theta_{t+1}) - f(\theta_t) \\
& \leq  - \eta_t \left(\mu_t + \Sigma_t^{1/2} g_t\right)^T \mu_t + \frac{\eta^2L}{2}  \left(\mu_t + \frac{1}{\sqrt{m}}\Sigma_t^{1/2} g_t\right)^T \left(\mu_t +\frac{1}{\sqrt{m}} \Sigma_t^{1/2} g_t\right) \\
& = \underbrace{- \eta_t \mu_t^T \left( 1 - \frac{\eta L}{2}\right) \mu_t}_{T_1}  \underbrace{- \eta_t \mu_t^T \left( 1 - \frac{1}{\sqrt{m}}\eta_t L \right) \Sigma_t^{1/2} g_t}_{T_2} + \underbrace{\frac{\eta^2 L}{2m} g_t^T \Sigma_t g_t}_{T_3} ~.
\label{bnd:upper}
\end{split}
\eeq

and
\beq
\begin{split}
f&(\theta_{t+1}) - f(\theta_t) \\
& \geq  - \eta_t \left(\mu_t + \Sigma_t^{1/2} g_t\right)^T \mu_t - \frac{\eta_t^2L}{2}  \left(\mu_t + \frac{1}{\sqrt{m}}\Sigma_t^{1/2} g_t\right)^T \left(\mu_t + \frac{1}{\sqrt{m}}\Sigma_t^{1/2} g_t\right) \\
& = \underbrace{- \eta_t \mu_t^T \left( 1 + \frac{\eta_t L}{2}\right) \mu_t}_{T_4} \underbrace{- \eta \mu_t^T \left( 1 + \frac{1}{\sqrt{m}}\eta_t L \right) \Sigma_t^{1/2} g}_{T_5} \underbrace{- \frac{\eta_t^2 L}{2m} g_t^T \Sigma_t g_t}_{T_6} 
\end{split}
\label{bnd:lower}
\eeq

Since $g_t$ is isotropic, conditional on $\theta_t$, let us take expectation on both sides of inequalities \eqref{bnd:upper} and \eqref{bnd:lower}

\beq
\begin{split}
\mathbb{E}[&f(\theta_{t+1}) - f(\theta_t) |\theta_t] \\
& \leq \mathbb{E} [T_1 |\theta_t] + \mathbb{E}[T_2 |\theta_t] + \mathbb{E}[T_3 |\theta_t] \\
& = - \eta_t \mu_t^T \left( 1 - \frac{\eta_t L}{2} \right) \mu_t - \eta_t \mathbb{E} \left[\mu_t^T \left( 1 - \frac{1}{\sqrt{m}}\eta_t L \right) \Sigma_t^{1/2} g \bigg|\theta_t \right] + \frac{\eta_t^2 L}{2m} \mathbb{E} \left[ g_t^T \Sigma_t g_t \bigg|\theta_t \right] \\
& = - \eta_t \|\mu_t\|_2^2 + \frac{\eta_t^2 L}{2} \|\mu_t\|_2^2 + \frac{\eta_t^2 L}{2m} \tr\left(\Sigma_t\right) \\
& = -\eta_t \|\mu_t\|_2^2 + \frac{\eta_t^2 L}{2} \tr \left(\frac{1}{m}\Sigma_t + \mu_t \mu_t^T \right) \\
& = -\eta_t \|\mu_t\|_2^2 + \frac{\eta_t^2 L}{2} \tr M_t
\end{split}
\label{bnd:e:upper}
\eeq
Similarly,
\beq
\begin{split}
\mathbb{E}[&f(\theta_{t+1}) - f(\theta_t) |\theta_t] \\
&\geq \mathbb{E} [T_4 |\theta_t] + \mathbb{E}[T_5 |\theta_t] + \mathbb{E}[T_6 |\theta_t] \\
& \geq - \eta_t \mu_t^T \left( 1 + \frac{\eta_t L}{2}\right) \mu_t - \eta \mathbb{E} \left[\mu_t^T \left( 1 + \frac{1}{\sqrt{m}}\eta_t L \right) \Sigma_t^{1/2} g \bigg|\theta_t \right] - \frac{\eta_t^2 L}{2m} \mathbb{E} \left[ g_t^T \Sigma_t g_t \bigg|\theta_t \right] \\
& = - \eta_t \|\mu_t\|_2^2 - \frac{\eta_t^2 L}{2} \|\mu_t\|_2^2 - \frac{\eta_t^2 L}{2m} \tr(\Sigma_t) \\
& = -\eta_t \|\mu_t\|_2^2 - \frac{\eta_t^2 L}{2} \tr \left(\frac{1}{m}\Sigma_t + \mu_t \mu_t^T \right) \\
& = -\eta_t \|\mu_t\|_2^2 - \frac{\eta_t^2 L}{2} \tr M_t
\end{split}
\label{bnd:e:lower}
\eeq

Next, we focus on a large deviation bound for the $f(\theta_{t+1})-f(\theta_t)|\theta_t$. First, note that $T_1$ is deterministic. For $T_2$ and $T_5$, let $\varsigma_t=\eta_t ( 1 - \frac{1}{\sqrt{m}}\eta_t L ) \Sigma_t^{1/2} \mu_t \in \R^p$ and $\varsigma_t^{'}=\eta_t ( 1 + \frac{1}{\sqrt{m}}\eta_t L ) \Sigma_t^{1/2} \mu_t \in \R^p$. Since $g_t$ is uniform on a sphere, by concentration inequality on a sphere \cite{leta91}, we have
\beq
P[T_2-ET_2\geq s]=P \left[  \varsigma_t^T g_t  \geq s \right] \leq \exp\left(- \frac{s^2}{ \| \varsigma_t \|^2} \right)~.
\label{eq:t2dev-1}
\eeq
\beq
P[T_2-ET_2\leq -s]=P \left[  \varsigma_t^{'T} g_t  \leq -s \right] \leq \exp\left(- \frac{s^2}{ \| \varsigma^{'}_t \|^2} \right)~.
\label{eq:t2dev-2}
\eeq

For $T_3$ and $T_6$, with $A_t = \frac{\eta_t^2 L}{2m} \Sigma_t$ a positive semidefinite matrix, from the Hanson-Wright inequality \cite{hskz12}, we have

\beq
P[T_3-ET_3\geq s]=P \left[  g_t^T A_t g_t - \frac{\eta_t^2L}{2} \tr(\Sigma_t ) \geq s \right]
\leq \exp\left[ - c_2 \min\left( \frac{s^2}{ \| A_t \|_F^2} , \frac{s}{ \| A_t \|_2 } \right) \right]~,
\label{eq:t3dev-1}
\eeq

\beq
P[T_3-ET_3\leq -s]=P \left[  g_t^T A_t g_t - \frac{\eta_t^2L}{2} \tr(\Sigma_t ) \leq -s \right]
\leq \exp\left[ - c_2 \min\left( \frac{s^2}{ \| A_t \|_F^2} , \frac{s}{ \| A_t \|_2 } \right) \right]~,
\label{eq:t3dev-2}
\eeq

where $c_2 > 0$ is an absolute constant.

Then we use the concentration inequality for $T_i,\; i=3,4,5,6$ to construct the two sided tail bounds for $f(\theta_{t+1})-f(\theta_t)|\theta_t$, which is: 
\beq
\begin{split}
& P  \bigg[ \Delta_t(f) - \bigg\{ - \eta_t \| \mu_t \|_2^2 + \frac{\eta_t^2}{2}L\tr M_t \bigg\}  \geq  ~~~ s  \bigg| \theta_t \bigg]\\
\leq& P  \bigg[ T_1+T_2+T_3 - \bigg\{ - \eta_t \| \mu_t \|_2^2 + \frac{\eta_t^2}{2}L\tr M_t \bigg\}  \geq  ~~~ s  \bigg| \theta_t \bigg]\\
=& P  \bigg[ T_2-\mathbb{E}T_2+T_3-\mathbb{E}T_3\geq  ~~~ s  \bigg| \theta_t \bigg]\\
\leq& P  \bigg[ T_2-\mathbb{E}T_2\geq  ~~~ s/2  \bigg| \theta_t \bigg]+P  \bigg[ T_3-\mathbb{E}T_3\geq  ~~~ s/2  \bigg| \theta_t \bigg]\\
\leq &\exp\left(- \frac{s^2}{ \| \varsigma_t \|^2} \right)+\exp\left[ - c_2 \min\left( \frac{s^2}{ \| A_t \|_F^2} , \frac{s}{ \| A_t \|_2 } \right)\right]\\
\leq &2\exp\left(-c\min\left(\frac{s^2}{ \| \varsigma_t \|^2}, \frac{s^2}{ \| A_t \|_F^2},\frac{s}{ \| A_t \|_2 }\right)\right)
\end{split}
\eeq

and 

\beq
\begin{split}
& P  \bigg[ \Delta_t(f) - \bigg\{ - \eta_t \| \mu_t \|_2^2 - \frac{\eta_t^2}{2}L\tr M_t \bigg\}  \leq  ~~~ -s  \bigg| \theta_t \bigg]\\
\leq& P  \bigg[ T_4+T_5+T_6 - \bigg\{ - \eta_t \| \mu_t \|_2^2 - \frac{\eta_t^2}{2}L\tr M_t \bigg\}  \leq  ~~~ -s  \bigg| \theta_t \bigg]\\
=& P  \bigg[ T_4-\mathbb{E}T_4+T_5-\mathbb{E}T_5\leq  ~~~ -s  \bigg| \theta_t \bigg]\\
\leq& P  \bigg[ T_4-\mathbb{E}T_4\leq  ~~~ -s/2  \bigg| \theta_t \bigg]+P  \bigg[ T_5-\mathbb{E}T_5\leq  ~~~ -s/2  \bigg| \theta_t \bigg]\\
\leq &\exp\left(- \frac{s^2}{ \| \varsigma^{'}_t \|^2} \right)+\exp\left[ - c_2 \min\left( \frac{s^2}{ \| A_t \|_F^2} , \frac{s}{ \| A_t \|_2 } \right)\right]\\
\leq &2\exp\left(-c\min\left(\frac{s^2}{ \| \varsigma^{'}_t \|^2}, \frac{s^2}{ \| A_t \|_F^2},\frac{s}{ \| A_t \|_2 }\right)\right)
\end{split}
\eeq

Here $c$ is an absolute constant. \qed

\section{SGD Dynamics: Examples} 

\label{app_Ex_regression}
We apply Theorem \ref{thm:dynamic} to characterize the loss difference of two simple special cases in the over-parameterized setting: high dimensional least squares and high dimensional logistic regression.

\subsection{High Dimensional Least Squares}
For high dimensional least squares discussed in Example~\ref{ex:ls}, we have the following result:

\corri*

\proof Let us denote $\Delta_t(f) = f(\theta_{t + 1}) - f(\theta_t)$, we have 
\[
\Delta_t(f) = \frac{\eta^2}{2} \left(\mu_t + \Sigma_t^{\frac{1}{2}}g \right)^T H_f \left(\mu_t +  \Sigma_t^{\frac{1}{2}}g \right) - \eta \mu_t^T \left(\mu_t +  \Sigma_t^{\frac{1}{2}}g \right).
\]
Then
\beq
\begin{aligned}
\E [\Delta_t(f)|\theta_t] &= \frac{\eta^2}{2} \tr]\left( H_f\left( \Sigma_t + \mu_t \mu_t^T\right)\right) - \eta \|\mu_t\|_2^2 \\
& \leq \frac{\eta^2 L}{2} \tr\left( \Sigma_t + \mu_t \mu_t^T\right) - \eta \|\mu_t\|_2^2
\end{aligned}
\eeq

Let $A = L\Sigma_t$, following Theorem \ref{thm:dynamic}, we have

\begin{align*}
    & P(|\Delta_t(f) - \E \Delta_t(f)| \geq s) 
    \leq  2 \exp\left[ -c \min \left( \frac{s^2}{ \alpha_{t,1}^2}, \frac{s^2}{\alpha_{t,2}^2}, \frac{s}{ \alpha_{t,3}} \right) \right]~,
\end{align*}
where 
\[
    \alpha_{t,1} = |1 - \frac{\eta L}{2\sqrt{m}}|\|\Sigma_t^{\frac{1}{2}} \mu_t\|_2~, \qquad \alpha_{t,2} = \frac{L}{m}\|\Sigma_t\\|_F~, \qquad \alpha_{t,3} = \frac{L}{m}\|\Sigma_t\|_2~,
\]

and $c$ is a positive constant. If we analyze the convergence of least squares in expectation, we have
\begin{align*}
	\E [\Delta_t(f)|\theta_t]  &\leq \frac{\eta^2 L}{2 } \frac{1}{n} \sum_{i=1}^n\|x_i\|_2^2 (x_i^T \theta - y_i)^2 - \eta \|\mu_t\|_2^2 \\
    &= \frac{\eta^2 L}{2 } \frac{1}{n} \sum_{i=1}^n\|x_i\|_2^2 (x_i^T \theta - y_i)^2 
	 - \frac{\eta}{n^2} (X\theta - y)^T (X X^T) (X\theta - y).
\end{align*}

Let us assume $\sigma_{\min} (\frac{1}{n^2}X X^T) \geq \alpha > 0$ and $x_i^T H_f x_i  \leq \beta$ for $i = 1, \ldots, n$, we have
\[
	\E [\Delta_t(f)|\theta_t]  \leq \frac{\eta^2 \beta L}{2} \cdot \frac{1}{n}\|X \theta_t - y\|_2^2 - \eta \alpha \cdot \frac{1}{n} \|X \theta_t - y\|_2^2~.
\]
Let $\eta = \frac{\alpha}{\beta L n}$, we have
\[
		\E [\Delta_t(f)|\theta_t] \leq - \frac{\alpha^2}{2 \beta L n} \|X \theta_t - y\|_2^2.
\]

The analysis above proves Corollary \ref{cor:ls} following Theorem \ref{thm:dynamic}. 
\qed

From Corollary \ref{cor:ls}, SGD for high dimensional least squares has two phases. In the early phase, $-\frac{\alpha^2}{2 \beta L n} \|X \theta_t - y\|_2^2$ will be much smaller than zero, thus SGD can sharply decrease $f$; $\alpha_{t,1}, \alpha_{t,2}, \alpha_{t,3}$ are large, allowing SGD to explore loss surface. In the later phase both $\frac{\alpha^2}{2 \beta L n} \|X \theta_t - y\|_2^2$ and $\|\Sigma_t\|_2$ are small, therefore SGD will do a steady decrease. Therefore, SGD is able to hit the global minima of $\hat{f}(\theta)$.

\subsection{High Dimensional Logistic Regression.} The empirical loss of binary logistic regression is given by
\beq
\hat{f}(\theta; z) = -\log p_{\theta}(x)^{y} (1 - p_{\theta}(x))^{1 - y},\tag{\ref{loss:lr}}
\eeq
where $p_{\theta}(x) = (1 + \exp(- \theta^T x))^{-1}$. 
Then, we have:
\corrii*

\proof Let $\Delta_t (f) = f(\theta_{t + 1}) - f(\theta)$, we have 
\begin{align*}
	\E [\Delta_t(f)|\theta_t] & \leq \frac{L}{2}\E \|\bar{\mu}_t\|_2^2 - \eta \|\mu_t\|_2^2 \\
	& = \frac{\eta^2 L}{2} \tr(\Sigma_t + \mu_t \mu_t^T) - \eta \|\mu_t\|_2^2 \\
	& = \frac{\eta^2 L}{2} \frac{1}{n}\sum_{i=1}^n \|x_i\|_2^2 (y_i - p_{\theta_t}(x_i))^2 - \eta \|\mu_t\|_2^2~,
\end{align*}
Assume $\sigma_{\min}(\frac{1}{n}X X^T) \geq \alpha > 0$ and $\max_i \|x_i\|_2^2 \leq \beta$ then
\begin{align*}
		\E [\Delta_t(f)|\theta_t] \leq & \frac{\eta^2 \beta L}{2} \frac{1}{n} \sum_{i=1}^{n}(y_i - p_{\theta_t}(x_i))^2
		 - \eta \alpha \frac{1}{n} \sum_{i=1}^{n}(y_i - p_{\theta_t}(x_i))^2.
\end{align*}
We choose $\eta = \frac{\alpha}{\beta L}$, we have
\[
		\E [\Delta_t(f)|\theta_t] \leq - \frac{ \alpha^2}{2\beta L n} \sum_{i=1}^{n}(y_i - p_{\theta_t}(x_i))^2.
\]
The analysis above proves Corollary \ref{cor:lr} following Theorem \ref{thm:dynamic}. \qed

In the high dimensional case, the data are always linearly separable. In this case $\sum_{i=1}^n (y_i - p_{\theta_t}(x_i))^2$ can be arbitrarily small (depending on $\| x_i \|_2$), therefore SGD will have a similar behavior as least squares.

\section{SGD Dynamics: Proof of Theorem \ref{thm:Bernstein}}
Recall Theorem \ref{thm:Bernstein} in Section \ref{sec:dynamics}:
\theoii*

\proof
Recall that from Theorem \ref{thm:dynamic}, we already have two sided bounds:
\begin{align}
P  \left[\left. \Delta_t(f) - \bigg\{ - \eta_t \| \mu_t \|_2^2 + \frac{\eta_t^2}{2}L\tr M_t \bigg\}\right|\theta_t  \geq s \right] 
  \leq 2\exp \left[ -c \min \left( \frac{s^2}{ \alpha_{t,1}^2} , \frac{s^2}{ \alpha_{t,2}^2}, \frac{s}{ \alpha_{t,3}} \right) \right]~,\\
P  \left[\left. \Delta_t(f) - \bigg\{ - \eta_t \| \mu_t \|_2^2 - \frac{\eta_t^2}{2}L\tr M_t \bigg\} \right|\theta_t \leq -s \right] 
  \leq 2\exp \left[ -c \min \left( \frac{s^2}{ \beta_{t,1}^2} , \frac{s^2}{ \alpha_{t,2}^2}, \frac{s}{ \alpha_{t,3}} \right) \right]~,
\end{align}
where $\alpha_{t,1}  = \eta_t |1 - \frac{1}{\sqrt{m}}\eta_t L| \| \Sigma_t^{1/2} \mu_t \|_2~,\beta_{t,1}  = \eta_t |1 + \frac{1}{\sqrt{m}}\eta_t L| \| \Sigma_t^{1/2} \mu_t \|_2~, \alpha_{t,2}  = \frac{\eta_t^2 L}{2m}\left\| \Sigma_t \right\|_F~,\alpha_{t,3}  = \frac{\eta_t^2 L}{2m}\left\| \Sigma_t \right\|_2~$, and $c > 0$ is an absolute constant. Based on inequalities \eqref{bnd:upper} and \eqref{bnd:lower}, we have:  
\[
\mathbb{E}[\Delta_t(f)|\theta_t] \in \left[ - \eta_t \| \mu_t \|_2^2 - \frac{\eta_t^2}{2}L\tr M_t, - \eta_t \| \mu_t \|_2^2 + \frac{\eta_t^2}{2}L\tr M_t \right]
\]

For $\forall s\geq 2\eta_t^2L\tr M_t$, we have: 
\begin{equation}\label{eq: prob_ma}
\begin{split}
P[&\left.|\Delta_t(f)-\mathbb{E}\Delta_t(f)|\geq s\right|\theta_t] \\
&=P[\left.\Delta_t(f)-\mathbb{E}\Delta_t(f)\geq s|\theta_t]+P[\Delta_t(f)-\mathbb{E}\Delta_t(f)\leq -s\right|\theta_t]\\
  &= P\left[\left.\Delta_t(f)-\bigg\{ - \eta_t \| \mu_t \|_2^2 + \frac{\eta_t^2}{2}L\tr M_t \bigg\}\geq s+\mathbb{E}\Delta_t(f)-\bigg\{ - \eta_t \| \mu_t \|_2^2 + \frac{\eta_t^2}{2}L\tr M_t \bigg\}\right|\theta_t\right]\\
  &\phantom{=} +P\left[\left.\Delta_t(f)-\bigg\{ - \eta_t \| \mu_t \|_2^2 - \frac{\eta_t^2}{2}L\tr M_t \bigg\}\leq -s+\mathbb{E}\Delta_t(f)-\bigg\{ - \eta_t \| \mu_t \|_2^2 - \frac{\eta_t^2}{2}L\tr M_t \bigg\}\right|\theta_t\right]\\
  &\leq P \left[\left.\Delta_t(f)-\bigg\{ - \eta_t \| \mu_t \|_2^2 + \frac{\eta_t^2}{2}L\tr M_t \bigg\}\geq s-\eta_t^2L\tr M_t\right|\theta_t \right]\\
   &\phantom{=} +P \left[\left.\Delta_t(f)-\bigg\{ - \eta_t \| \mu_t \|_2^2 - \frac{\eta_t^2}{2}L\tr M_t \bigg\}\leq -s+\eta_t^2L\tr M_t \right|\theta_t\right]\\
   &\leq 2\exp\left[-c\min \left(\frac{\left(s-\eta_t^2L\tr M_t\right)^2}{\alpha_{t,1}^2},\frac{(s-\eta_t^2L\tr M_t)^2}{\alpha_{t,2}^2},\frac{s-\eta_t^2L\tr M_t}{\alpha_{t,3}} \right)\right] \\
   &\phantom{=} +2\exp\left[-c\min \left(\frac{(s-\eta_t^2L\tr M_t)^2}{\beta_{t,1}^2},\frac{(s-\eta_t^2L\tr M_t)^2}{\alpha_{t,2}^2},\frac{s-\eta_t^2L\tr M_t}{\alpha_{t,3}} \right)\right]\\
   &\leq 4\exp\left[-c\min \left(\frac{(s-\eta_t^2L\tr M_t)^2}{\beta_{t,1}^2},\frac{(s-\eta_t^2L\tr M_t)^2}{\alpha_{t,2}^2},\frac{s-\eta_t^2L\tr M_t}{\alpha_{t,3}} \right)\right]\\
   &\leq 4\exp\left[-c\min \left(\frac{(s-s/2)^2}{\beta_{t,1}^2},\frac{(s-s/2)^2}{\alpha_{t,2}^2},\frac{s-s/2}{\alpha_{t,3}} \right) \right]\\
   &\leq 4\exp\left[-c\min \left(\frac{s^2}{\beta_{t,1}^2},\frac{s^2}{\alpha_{t,2}^2},\frac{s}{\alpha_{t,3}} \right)\right]~,
\end{split}
\end{equation}
where $c$ is an absolute constant which can be different in different steps.

For non-negative random variable $X$, we have property $\mathbb{E}(X)=\int_{0}^{\infty}\mathbb{P}(X\geq u)du$. Let $X=\|\Delta_t-\mathbb{E}\Delta_t\|^p$, for any $p\geq 1$, based on \cite{vershynin_high_2018} and use \eqref{eq: prob_ma},  we have the moments of $X$ satisfy 
\begin{align*}
    \|\Delta_t-\mathbb{E}\Delta_t\|_{L^p}^p &= \mathbb{E}|\Delta_t-\mathbb{E}\Delta_t|^p\\
    &= \int_0^{\infty}{P}(|\Delta_t-\mathbb{E}\Delta_t|^p\geq u)du  \\
    &= \int_0^{\infty}
    {P}(|\Delta_t-\mathbb{E}\Delta_t|\geq s)ps^{p-1}ds  \\
    &= \int_0^{2\eta_t^2L\tr M_t}{P}(|\Delta_t-\mathbb{E}\Delta_t|\geq s)ps^{p-1}ds+\int_{2\eta_t^2L\tr M_t}^{\infty}
    {P}(|\Delta_t-\mathbb{E}\Delta_t|\geq s)ps^{p-1}ds \\
    &\leq p\left(2\eta_t^2L\tr M_t\right)^p+\int_{2\eta_t^2L\tr M_t}^{\infty}4\exp\left(-c\min\left(\frac{s^2}{\beta_{t,1}^2},\frac{s^2}{\alpha_{t,2}^2},\frac{s}{\alpha_{t,3}}\right)\right)ps^{p-1}ds \\
    &\leq p\left(2\eta_t^2L\tr M_t\right)^p+\int_{0}^{\infty}4\exp\left(-c\min\left(\frac{s^2}{\beta_{t,1}^2},\frac{s^2}{\alpha_{t,2}^2},\frac{s}{\alpha_{t,3}}\right)\right)ps^{p-1}ds)^{1/p}\\
    &\leq p\left(2\eta_t^2L\tr M_t\right)^p+\int_{0}^{\infty}4\exp\left(-c\frac{s^2}{\beta_{t,1}^2}\right)ps^{p-1}ds\\
    &\phantom{=} +\int_{0}^{\infty}4\exp\left(-c\frac{s^2}{\alpha_{t,2}}\right)ps^{p-1}ds+\int_{0}^{\infty}4\exp\left(-c\frac{s^2}{\alpha_{t,2}}\right)ps^{p-1}ds~.
\end{align*}
Taking $p$-th root, we have
\begin{align*}
    \|\Delta_t-\mathbb{E}\Delta_t\|_{L^p} 
    &\leq  2p^{1/p}\eta_t^2L\tr M_t+p^{1/p}\left(\int_{0}^{\infty}4\exp\left(-c\frac{s^2}{\beta_{t,1}^2}\right)s^{p-1}ds\right)^{1/p}\\
    &\phantom{=} +p^{1/p}\left(\int_{0}^{\infty}4\exp\left(-c\frac{s^2}{\alpha_{t,2}}\right)ps^{p-1}ds\right)^{1/p}+\left(\int_{0}^{\infty}4\exp\left(-c\frac{s^2}{\alpha_{t,2}}\right)ps^{p-1}ds\right)^{1/p}\\
    &=2p^{1/p}\eta_t^2L\tr M_t+(4p)^{1/p}\left(\int_0^{\infty}\exp(-v)\frac{\beta_{t,1}^p}{c^{p/2}}v^{\frac{p-2}{2}}\frac{1}{2}dt\right)^{1/p}\\
    &\phantom{=} +(4p)^{1/p}\left(\int_0^{\infty}\exp(-v)\frac{\alpha_{t,2}^p}{c^{p/2}}v^{\frac{p-2}{2}}\frac{1}{2}dt\right)^{1/p}+(4p)^{1/p}\left(\int_0^{\infty}\exp(-v)\frac{\alpha_{t,3}^p}{c^p}v^{p-1}dv\right)^{1/p}\\
    &=2p^{1/p}\eta_t^2L\tr M_t+(4p)^{1/p}\left(\frac{1}{2}\frac{\beta_{t,1}^p}{c^p}\Gamma\left(\frac{p-2}{2}\right)\right)^{1/p}+(4p)^{1/p}\left(\frac{1}{2}\frac{\alpha_{t,2}^p}{c^p}\Gamma\left(\frac{p-2}{2}\right)\right)^{1/p}\\
    & \phantom{=} +(4p)^{1/p}\left(\frac{\alpha_{t,3}^p}{c^p}\Gamma(p-1)\right)^{1/p}~,
\end{align*}
where $\Gamma(x)$ represents gamma function \cite{sebah02}. Using the property of gamma function \cite{sebah02}: $\Gamma(x)\leq x^x$, we have: 
\begin{align*}
    \|\Delta_t&-\mathbb{E}\Delta_t\|_{L^p} \\
    & \leq \frac{1}{c}\left(2p^{1/p}\eta_t^2L\tr M_t+(2p)^{1/p}\beta_{t,1}\left(\frac{p-2}{2}\right)^{\frac{p-2}{2p}}+(2p)^{1/p}\alpha_{t,2}(2(p-2))^{\frac{p-2}{2p}}+(4p)^{1/p}\alpha_3(p-1)^{\frac{p-1}{p}}\right)\\
&\leq \frac{1}{c}\left(2p^{1/p}\eta_t^2L\tr M_t+4^{1/p}\beta_{t,1}\left(\frac{p}{2}\right)^{1/2}+4^{1/p}\alpha_{t,2}\left(\frac{p}{2}\right)^{1/2}+4^{1/p}\alpha_{t,3} p\right)\\
&\leq \frac{1}{4c}(\eta_t^2L\tr M_t+\beta_{t,1}+\alpha_{t,2}+\alpha_{t,3})p~.
\end{align*}
From bounded assumption: stepsize $\eta_t\leq\eta$, gradient $\|\mu_t\|_2\leq\mu$, covariance $\|\Sigma_t^{1/2}\|_2\leq \sigma^{1/2}$, then $\|\Sigma_t\|_2\leq \|\Sigma_t^{1/2}\|_2^2\leq \sigma$, and $\|\Sigma_t\|_F\leq m\|\Sigma_t\|_2\leq m\sigma$, we have the following 
\begin{align*}
    \|\Delta_t&-\mathbb{E}\Delta_t\|_{L^p} \\
    & \leq \frac{1}{4c}(\eta_t^2L\tr M_t+\beta_{t,1}+\alpha_{t,2}+\alpha_{t,3})p\\
    &= \frac{1}{4c}\left( \eta_t^2 L\tr M_t+\eta_t |1 + \eta_t L| \frac{1}{\sqrt{m}}\| \Sigma_t^{1/2} \mu_t \|_2+\frac{\eta_t^2 L}{2m}\left\| \Sigma_t \right\|_F+ \frac{\eta_t^2 L}{2m}\left\| \Sigma_t \right\|_2\right)p\\
    &\leq \frac{1}{4c}\left( \eta^2 L\left(\sigma+\mu^2\right)+\frac{\eta\mu^2}{\sqrt{m}}\left(1+\eta L\right)\sqrt{\sigma}+\frac{(m+1)\eta^2L\sigma}{2m}\right)p\\
    &:=Kp
\end{align*}

Therefore, $\Delta_t(f)-\mathbb{E}\Delta_t(f)|\theta_t$ is a  sub-exponential MDS, which has so-called sub-exponential norm or Orlicz norm equivalent to $K=\frac{1}{4c}\left( \eta^2 L\left(\sigma+\mu^2\right)+\frac{\eta\mu^2}{\sqrt{m}}\left(1+\eta L\right)\sqrt{\sigma}+\frac{(m+1)\eta^2L\sigma}{2m}\right)$ up to an absolute constant. Another equivalent form of sub-exponential property \cite{vershynin_high_2018} is:
\beq
\label{subex_pro}
\mathbb{E}[\exp(\lambda (\Delta_t(f)-\mathbb{E}\Delta_t(f)))|\theta_t] \leq \exp(K^2\lambda^2), \qquad \forall\lambda \text{ s.t. }~~|\lambda|\leq 1/K ~.
\eeq
In the next step, using this sub-exponential property on MDS, we can derive an Azuma-Bernstein inequality for sub-exponential MDSs, which is a generalization of classical Azuma-Hoeffding inequality \cite{melnyk_estimating_2016}.\\
For any $0<\lambda\leq 1/K$, we have:
\begin{align*}
    P &\left[f(\theta_T)-\mathbb{E}f(\theta_T)>s\right]\\
    & = P\left[\sum_{t=0}^{T-1}(\Delta_t(f)-\mathbb{E}[\Delta_t(f)])>s\right]\\
    & = P\left[\exp\left(\lambda\sum_{t=0}^{T-1}(\Delta_t(f)-\mathbb{E}[\Delta_t(f)])\right)>\exp(\lambda s)\right]\\
    & \leq \exp(-\lambda s)\mathbb{E}\left[\exp\left(\lambda\sum_{t=0}^{T-1}\left(\Delta_t(f)-\mathbb{E}\left[\Delta_t(f)\right]\right)\right)\right]\\
    & = \exp(-\lambda s)\mathbb{E}\left[\left.\mathbb{E}\left[\prod_{t=0}^{T-1}\exp(\lambda (\Delta_t(f)-\mathbb{E}[\Delta_t(f)]))\right|\theta_{T-1}\right]\right]\\
    & = \exp(-\lambda s)\mathbb{E}\left[\mathbb{E}\left[\left.\exp(\lambda (\Delta_{T-1}(f)-\mathbb{E}[\Delta_{T-1}(f)]))\right|\theta_{T-1}\right]\prod_{t=0}^{T-1}\exp(\lambda (\Delta_t(f)-\mathbb{E}[\Delta_t(f)]))\right]\\
    &=\exp(-\lambda s)\mathbb{E}[\mathbb{E}\left[\left.\exp(\lambda (\Delta_{T-1}(f)-\mathbb{E}[\Delta_{T-1}(f)|\theta_{T-1}]))\right|\theta_{T-1}\right]\\
    &\mathbb{E}\left[\exp(\lambda (\mathbb{E}[\Delta_{T-1}|\theta_{T-1}]-\mathbb{E}[\Delta_{T-1}]))\right]\prod_{t=0}^{T-1}\exp(\lambda (\Delta_t(f)-\mathbb{E}[\Delta_t(f)]))]
    \end{align*}
    Notice that here when we want to apply the Equation \eqref{subex_pro}, the main issue here is the difference between $\mathbb{E}[\Delta_{T-1}|\theta_{T-1}]$ and $\mathbb{E}[\Delta_{T-1}]$. From $\mathbb{E}\left[\mathbb{E}[\Delta_{T-1}|\theta_{T-1}]=\mathbb{E}[\Delta_{T-1}]\right]$ and we have bounded $\Delta_t$:
    \begin{align*}
        |\Delta_t|=|f(\theta_{t+1})-f(\theta_t)|&\leq L_1\|\theta_{t+1}-\theta_t\|_2\\
        &=L_1\|\eta_t\mu_t+\eta_t\frac{1}{\sqrt{m}}\Sigma_t^{1/2}g\|\\
        &\leq L_1\eta(\mu+\frac{\sqrt{\sigma p}}{\sqrt{m}})
    \end{align*}
    so we can apply Hoeffding Lemma on random variable $\mathbb{E}[\Delta_{T-1}|\theta_{T-1}]-\mathbb{E}[\Delta_{T-1}]$: For any $\lambda\in \mathbb{R}$
    \begin{align*}
        \mathbb{E}\left[\exp(\lambda(\mathbb{E}[\Delta_{T-1}|\theta_{T-1}]-\mathbb{E}[\Delta_{T-1}])\right]\leq \exp(2L_1\eta(\mu+\frac{\sqrt{\sigma p}}{\sqrt{m}})\lambda^2):=\exp(K_1\lambda^2)
    \end{align*}
    Combining Equation \eqref{subex_pro}, the last inequality can be further computed as:
    \begin{align*}
    &\leq \exp(-\lambda s)\exp(\lambda^2  K^2+\lambda^2K_1^2)\mathbb{E}\left[\prod_{t=0}^{T-2}\exp(\lambda (\Delta_t(f)-\mathbb{E}[\Delta_t(f)]))\right]\\
    & \leq   \exp(-\lambda s)\exp(2\lambda^2 K^2+2\lambda^2 K_1^2)\mathbb{E}\left[\prod_{t=0}^{T-3}\exp(\lambda (\Delta_t(f)-\mathbb{E}[\Delta_t(f)]))\right]\\
    & \leq\cdots\leq \exp(-\lambda s+T\lambda^2K^2+T\lambda^2K_1^2)
\end{align*}
Denote $K_2=\max(K,K_1)$, choosing $\lambda =\min\left(\frac{s}{4K_2^2T},\frac{1}{K_2}\right)$, we obtain
\beq
P\left[f(\theta_T)-\mathbb{E}f(\theta_T)>s\right]\leq \exp \left(-c\min \left(\frac{s^2}{8K_2^2T},\frac{s}{2K_2}\right) \right)
\eeq
Repeating the same argument with $-(\Delta_t(f)-\mathbb{E}[\Delta_t(f)])$ instead of $\Delta_t(f)-\mathbb{E}[\Delta_t(f)]$, we obtain the similar bound for ${P}\left[f(\theta_T)-\mathbb{E}f(\theta_T)<-s\right]$. Combining the two results give:
\beq
P\left[|f(\theta_T)-\mathbb{E}f(\theta_T)|>s\right]\leq 2\exp\left(-c\min\left(\frac{s^2}{8K_2^2T},\frac{s}{2K_2}\right)\right)
\eeq
Take $s=\tau \sqrt{T}$, we can have another form:
\beq
P\left[|f(\theta_T)-\mathbb{E}f(\theta_T)|>\tau \sqrt{T}\right]\leq 2\exp\left(-c\min\left(\frac{\tau^2}{8K_2^2},\frac{\tau\sqrt{T}}{2K_2}\right)\right)
\eeq
The proof is similar for preconditioned SGD. That completes the proof. 
\qed

\section{Proof of Theorem \ref{thm:convergence}}

Recall Theorem~\ref{thm:convergence} in Section~\ref{sec:dynamics}:
\theoiii*

\proof
We first show the proof of statement (1). By inequality \eqref{bnd:e:lower}, $f(\theta_t)$ is a non-negative martingale for our choices of step size and preconditioner. Then statement (1) is a direct result from Martingale Convergence Theorem for non-negative super-martingale \cite{williams_probability_1991}. We proof statement (2) by following the  proof strategy of Theorem 2.1 in \cite{ghla13}, then we show the convergence using contradiction. We show the proof of both vanilla SGD \eqref{step:vanilla} and preconditioned SGD \eqref{step:precond}. Let us first look at vanilla SGD \eqref{step:vanilla}. By inequality \eqref{bnd:e:upper} and our choice of step size \eqref{step:vanilla}, 
\[
	\E [ f(\theta_{t +1}) - f(\theta_t) | \theta_t ] \leq - \frac{\|\mu_t\|_2^4}{4 L \tr M_t} = - \frac{\|\mu_t\|_2^4}{4 L (\tr \Sigma_t + \|\mu_t\|_2^2)}.
\]
By our assumption $\tr \Sigma_t \leq \sigma^2$, we have
\[
	\E [ f(\theta_{t +1}) - f(\theta_t) | \theta_t ] \leq - \frac{\|\mu_t\|_2^4}{4 L (\sigma^2 + \|\mu_t\|_2^2)}.
\]
Let us take expectation over all $\theta_t$, 
\begin{align*}
	\E [ f(\theta_{t +1}) - f(\theta_t) ] & \leq - \E \frac{\|\mu_t\|_2^4}{4 L (\sigma^2 + \|\mu_t\|_2^2)} \\
	\Rightarrow \qquad \E \frac{\|\mu_t\|_2^4}{4 L (\sigma^2 + \|\mu_t\|_2^2)} & \leq \E [ f(\theta_t) - f(\theta_{t +1}) ]~.
\end{align*}
Since function $\frac{x^2}{\sigma^2 + x}$ is convex when $x \geq 0$, by Jensen's inequality, we have
\[
	\frac{(\E \|\mu_t\|_2^2 )^2}{4 L (\sigma^2 + \E \|\mu_t\|_2^2)} \leq \E \frac{\|\mu_t\|_2^4}{4 L (\sigma^2 + \|\mu_t\|_2^2)} \leq \E [ f(\theta_t) - f(\theta_{t +1})].
\]
Let us sum over all $t = 0, 1, \ldots, T$, we have  
\[
	\frac{1}{4L} \sum_{t = 0}^T \frac{\E \|\mu_t\|_2^2 }{ \sigma^2 + \E \|\mu_t\|_2^2} \cdot \E \|\mu_t\|_2^2 \leq \E [f(\theta_0) - f(\theta_T)] \leq \E [f(\theta_0) - f^*],
\]
where the second inequality follows the fact that $f^*$ is the global minimum of function $f$. We can divide both side of the inequality by $\frac{1}{2L} \sum_{t = 0}^T \frac{\E \|\mu_t\|_2^2 }{ \sigma^2 + \E \|\mu_t\|_2^2} $ follows the proof of Theorem 2.1 in \cite{ghla13} to get
\[
	\min_t ~\E \|\mu_t\|_2^2 \leq \frac{ 4 L \E [f(\theta_0) - f^*]}{\sum_{t = 0}^T \frac{\E \|\mu_t\|_2^2 }{ \sigma^2 + \E \|\mu_t\|_2^2}}.
\]
Let us choose 
\[
	T \geq \frac{4L \E [f(\theta_0) - f^*] (\sigma^2 + \epsilon)}{\epsilon^2}.
\]
For any $\epsilon > 0$, if we assume $\min_t \E \|\mu_t\|_2^2 \geq \epsilon$, we have 
\[
	\min_t \E \|\mu_t\|_2^2 \leq \frac{ 4 L \E [f(\theta_0) - f^*]}{\sum_{t = 0}^T \frac{\E \|\mu_t\|_2^2 }{ \sigma^2 + \E \|\mu_t\|_2^2}} \leq \frac{ 4 L \E [f(\theta_0) - f^*]}{(T + 1) \frac{\epsilon }{ \sigma^2 + \epsilon}} < \epsilon~,
\]
and we get a contradiction. Therefore $\min_t \E \|\mu_t\|_2^2 = O\left( \frac{1}{\sqrt{T}} \right)$.

We then show the convrgence of preconditioned SGD \eqref{step:precond}. Let us denote $\sigma_{j}^2$ be upper bound of the $j$-th diagonal element of $\Sigma_t$. For the preconditioned SGD case, the conditional expectation can be bounded by 
\[
    \E \Delta_t(f) \leq - \frac{1}{4L}\E\sum_{j=1}^p \frac{\mu_{t, j}^2}{ M_{t, j}} \leq - \E\frac{1}{4L}\sum_{j=1}^p \frac{\mu_{t, j}^2}{ \sigma_j^2 + \mu_{t, j}^2}.
\]
Since function $\frac{x^2}{\sigma^2 + x}$ is convex when $x \geq 0$, by Jensen's inequality, we have
\[
    - \frac{1}{4L}\E\sum_{j=1}^p \frac{\mu_{t, i}^2}{4 L \sigma_j^2 + \mu_{t, j}^2} \leq -\frac{1}{4L} \E \frac{\|\mu_t\|_2^4}{\sum_{j=1}^p\sigma_j^2 + \|\mu_t\|_2^2} \leq -\frac{1}{4L} \E \frac{\|\mu_t\|_2^4}{\sigma^2 + \|\mu_t\|_2^2}.
\]
By using Jensen's inequality again, we have 
\[
	\frac{(\E \|\mu_t\|_2^2 )^2}{4 L (\sigma^2 + \E \|\mu_t\|_2^2)} \leq \E \frac{\|\mu_t\|_2^4}{4 L (\sigma^2 + \|\mu_t\|_2^2)} \leq \E [ f(\theta_t) - f(\theta_{t +1})].
\]
Summing over all $t = 0, \ldots, T$, we have
\[
	\frac{1}{4L} \sum_{t = 0}^T \frac{\E \|\mu_t\|_2^2 }{ \sigma^2 + \E \|\mu_t\|_2^2} \cdot \E \|\mu_t\|_2^2 \leq \E [f(\theta_0) - f(\theta_T)] \leq \E [f(\theta_0) - f^*].
\]

We can divide both side of the inequality by $\frac{1}{2L} \sum_{t = 0}^T \frac{\E \|\mu_t\|_2^2 }{ \sigma^2 + \E \|\mu_t\|_2^2} $ follows the proof of Theorem 2.1 in \cite{ghla13} to get
\[
	\min_t ~\E \|\mu_t\|_2^2 \leq \frac{ 4 L \E [f(\theta_0) - f^*]}{\sum_{t = 0}^T \frac{\E \|\mu_t\|_2^2 }{ \sigma^2 + \E \|\mu_t\|_2^2}}.
\]
Let us choose 
\[
	T \geq \frac{4L \E [f(\theta_0) - f^*] (\sigma^2 + \epsilon)}{\epsilon^2}.
\]
For any $\epsilon > 0$, if we assume $\min_t \E \|\mu_t\|_2^2 \geq \epsilon$, we have 
\[
	\min_t \E \|\mu_t\|_2^2 \leq \frac{ 4 L \E [f(\theta_0) - f^*]}{\sum_{t = 0}^T \frac{\E \|\mu_t\|_2^2 }{ \sigma^2 + \E \|\mu_t\|_2^2}} \leq \frac{ 4 L \E [f(\theta_0) - f^*]}{(T + 1) \frac{\epsilon }{ \sigma^2 + \epsilon}} < \epsilon~,
\]
and we get a contradiction. Therefore $\min_t \E \|\mu_t\|_2^2 = O\left( \frac{1}{\sqrt{T}} \right)$. That completes the proof. 
\qed

\section{Proof of scale-invariant generalization bound}\label{app:general}

In this section, we first present the proof of  Lemma \ref{Lem: KLD} and Corollary \ref{corr: KPQ}. Then we present the proof of  Theorem \ref{theo: bound1}.

\subsection{Proof of Lemma \ref{Lem: KLD} and Corollary \ref{corr: KPQ} }

Recall Lemma~\ref{Lem: KLD} in Section~\ref{sec:bound}:
\lemmi*

\begin{proof}
For distributions $\cQ$ and $\cP$ of a continuous random  variable with support $\cX \subseteq \mathbb{R}^p$, and $\cQ$ is absolutely continuous with respect to  $\cP$, the  differential relative entropy is defined to be
\begin{equation}
    KL(\cQ||\cP) = \int_{\cX} q(x) \log \frac{q(x)}{p(x)} dx,
\end{equation}
where $q(x)$ and $p(x)$ are probability density functions of $\cQ$ and $\cP$.

Suppose we have the following general invertible change of variables:
\begin{equation}
    r: \cX \rightarrow \cX
\end{equation}
and
\begin{equation}
    y =  r(x)
\end{equation}

Let $\cQ^\prime$ and $\cP^\prime$ to the transformed distributions corresponding to $\cQ$ and $\cP$ respectively. The transformed probability densities of 
$\cQ^\prime$ and $\cP^\prime$ is 
\begin{equation}
    q^\prime (y) = q(r^{-1}(y)) \left\vert 
    \det\left\{ \frac{\partial r^{-1}(y)}{\partial y}\right\} \right\vert = q(r^{-1}(y)) \left\vert \det J \right\vert,
\end{equation}
where $J = \frac{\partial r^{-1}(y)}{\partial y}$ is the so-called Jacobian of the transformation. The change of variables formula for integration is
\begin{equation}
    dx = dy  \left\vert \det J \right\vert.
\end{equation}
Let $\cX^\prime$ be the support of $y$.  So we have:
\begin{align}
	KL(\cQ^\prime || \cP^\prime) &= \int_{\cX^\prime}q^\prime(y) \frac{q^\prime(y)}{p^\prime(y)}dy \nonumber\\
	& = \int_{\cX} q(x) \log \frac{q(x)|\det J|}{p(x)|\det J|} dx \nonumber\\
	& = KL(\cQ||\cP)
\end{align}
provided the determinant of the transformation does not vanish since the Jacobian $J$ is non-singular given $r$ is invertible linear transformation. \qed

\end{proof}

\corriv*

This follows from the general result above. Since $\alpha$-transformation is an invertible transformation, we also give an independent proof for this special case of interest.  
	
\begin{proof} We use $T_\alpha(\theta) $ to denote $(\alpha_1 \theta_1, ..., \alpha_L \theta_L)$. Then we have $T_\alpha(\theta) = A\theta$ and $A$ is diagonal matrix
	\[  
	A = \begin{bmatrix} 
	\alpha_1 \mathbb{I}_{n_1} & \hdots &0 \\
	\vdots & \ddots & \vdots\\
	0 & \hdots & \alpha_L \mathbb{I}_{n_L} 
	\end{bmatrix}
	\]
where the $n_1, ..., n_L$ represent the number of parameters of each layer.

Let $\Sigma_{\cP_{{T_\alpha}(\theta_0)}}$ be the covariance of $\cP_{{T_\alpha}(\theta_0)}$, we have:
	\begin{equation} \label{eq: cov_P}
	\Sigma_{\cP_{{T_\alpha}(\theta_0)}} = A\Sigma_\cP A^T
	\end{equation}

	By the property of $\alpha$-scale transformation:
	\begin{eqnarray}\label{eq: Hess}
	H_f(T_\alpha(\theta)) =  (A^T)^{-1} H_f(\theta) A^{-1} 
	\end{eqnarray}

By the definition of $H(\theta)$, with  \eqref{eq: Hess} and  \eqref{eq: cov_P} after $T_\alpha$ transformation we have
	\begin{equation}
	\Sigma^{-1}_{\cQ_{T_\alpha(\theta)}} = [ \nu_1^\alpha, ..., \nu_p^\alpha ]
	\end{equation}
	where 	

	\begin{align}
	\nu_i^\alpha &= \max \left(H_f(T_\alpha(\theta)[i,i], 1/ (\alpha_i^2\sigma_i^2) \right) \nonumber\\
	& = \max \left((A^T)^{-1} H_f(\theta) A^{-1}[i,i],   1/ (\alpha_i^2\sigma_i^2) \right) \nonumber\\
	& = \max \left(\frac{1}{\alpha_i^2} H_f(\theta)[i,i],  1/ (\alpha_i^2\sigma_i^2) \right) \nonumber\\
	& = \frac{1}{\alpha_i^2} \max\left(H_f(\theta)[i,i], 1/ \sigma_i^2\right)
	\end{align}

	By the definition of the posterior covariance in Assumption \ref{asmp: p_q}, we have the covariance after $\alpha$-scale transformation is
	
	\begin{equation} \label{eq: cov_1}
	\Sigma_{\cQ_{T_\alpha(\theta)}} = A \Sigma_{\cQ_{\theta}}A^T
	\end{equation}

By definition, the $KL$-divergence before transformation is 
	\begin{align}
	KL(\cQ_\theta||\mathcal{P}) = \frac{1}{2}\left( \tr(\Sigma_\mathcal{P}^{-1}\Sigma_{\cQ_{\theta}} + (\theta_0 -\theta)^T\Sigma_\mathcal{P}^{-1}(\theta_0 -\theta)-d + \ln\left( \frac{\det \Sigma_\mathcal{P}}{\det \Sigma_{\cQ_{\theta}}} \right) \right) 	\end{align}

Equation \eqref{eq: cov_1} shows that after transformation, the posterior distribution is $\cQ_{T_\alpha(\theta)} \sim \cN(A\theta, A \Sigma_{\cQ_{\theta}}A^T)$ and the prior distribution is $\cP_{{T_\alpha}(\theta_0)} \sim \cN(A\theta_0,A\Sigma_\mathcal{P} A^T)$. So the $KL$-divergence after transformation is 
	\begin{align}
	&	KL( \cQ_{T_\alpha(\theta)} || \cP_{{T_\alpha}(\theta_0)}) \nonumber\\ 
	&=\frac{1}{2}\left( \tr
	\left( (A\Sigma_\mathcal{P} A^T)^{-1}A\Sigma_{\cQ_{\theta}} A^T \right) + (A\theta_0 -A\theta)^T(A\Sigma_\mathcal{P}A^T)^{-1}(A\theta_0 -A\theta)
	-d + \ln\left( \frac{\det A\Sigma_\mathcal{P}A^T}{\det A\Sigma_{\cQ_{\theta}}A^T} \right) \right) \nonumber\\
	& = \frac{1}{2}\left( \tr
	\left( (A^T)^{-1} \Sigma_\mathcal{P}^{-1} \Sigma_{\cQ_{\theta}} A^T \right) + (\theta_0 -\theta)^T\Sigma_\mathcal{P}^{-1}(\theta_0 -\theta)
	-d + \ln\left( \frac{\det (A) \det(\Sigma_\mathcal{P}) \det(A^T)}{\det (A) \det(\Sigma_{\cQ_{\theta}}) \det(A^T)} \right) \right) \nonumber\\
	& = \frac{1}{2}\left( \tr(\Sigma_\mathcal{P}^{-1}\Sigma_{\cQ_{\theta}}) + (\theta_0 -\theta)^T\Sigma_\mathcal{P}^{-1}(\theta_0 -\theta)-d + \ln\left( \frac{\det \Sigma_\mathcal{P}}{\det \Sigma_{\cQ_{\theta}}} \right) \right)  \nonumber\\
	& = KL(\cQ_\theta||\mathcal{P})
	\end{align}
	
Here we complete the proof.
\qed	
\end{proof}

\begin{figure}[t!] 
\centering
\subfigure[Diagonal Element of $H_f(\theta_t)$.]{
 \includegraphics[width = 0.48 \textwidth]{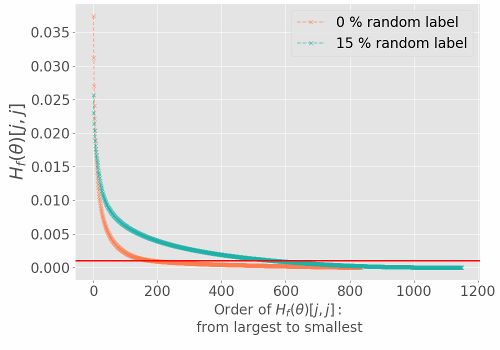}
 } 
 \subfigure[Effective Curvature.]{
 \includegraphics[width = 0.48 \textwidth]{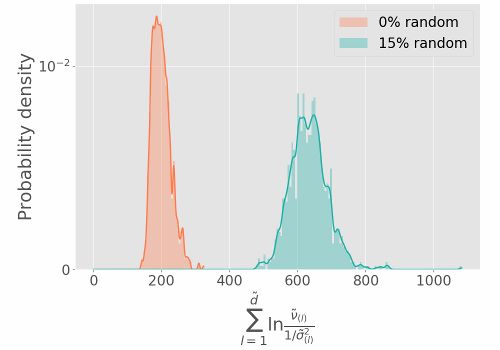}
 } 
 \subfigure[Precision Weighted Frobenius Norm.]{
 \includegraphics[width = 0.48 \textwidth]{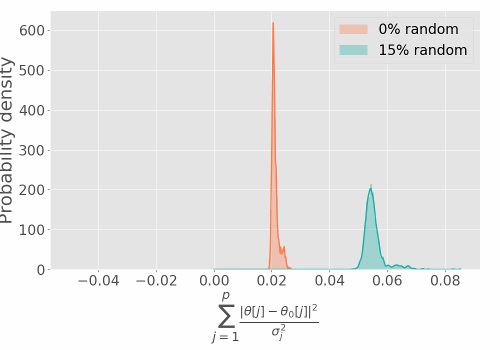}
 } 
 \subfigure[Scale-invariant Generalization Bound.]{
 \includegraphics[width = 0.48 \textwidth]{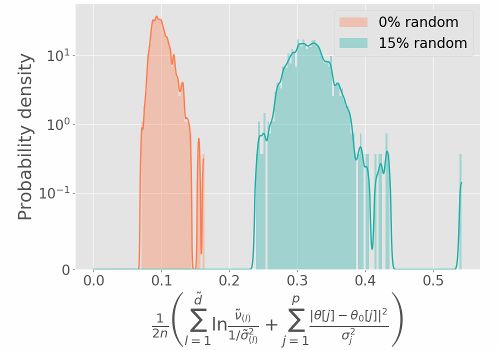}
 } 
\caption[]{Distributions from 10,000 runs on Gauss-10 datasets trained with large batches containing half of training samples (50/100): a) the diagonal elements (mean) of $H_f(\theta)$; random labels increase the value of diagonal elements implying a large first term in the generalization error bound. b) The first term in the generalization bound; as fraction of random labels increases, the term increases, partly explaining the poor generalization performance. c) The second term in the generalization bound; SGD goes further from the initialization with random labels ($\sigma_i = \sigma$, second term $=\|\theta\|^2/\sigma^2$), suggesting larger generalization error than true labels. d) The generalization bound, with randomness increased from $0\%$ to $15\%$, the distribution of the generalization error shifts to a higher value, which validates the proposed bound. }
\end{figure}

\subsection{Proof of Theorem \ref{theo: bound1}}
Recall Theorem \ref{theo: bound1} in Section~\ref{sec:bound}:
\theoiv*

\begin{proof}
	The result is an application of the PAC-Bayesian bound \cite{mcda99,lash03}, which states that for n i.i.d. samples, with probability at least $1-\delta$ we have
	\begin{equation}
	kl  ( \ell(\cQ,S) \| \ell(\cQ,D)) \leq \frac{1}{n} \left[ KL(\cQ\| \mathcal{P} ) + \ln \frac{n+1}{\delta} \right]~. 
	\label{eq: pac}
	\end{equation}
	
	The primary focus of our analysis is to bound $KL(\cQ \| \mathcal{P})$ under Assumption~\ref{asmp: p_q}. We have:
	\begin{align}
	2KL(\cQ|| P) &= \sum_{j = 1}^{p} (\frac{1}{\sigma_j^2 \nu_j} -1)+ \sum_{j =1}^{p}\frac{|\theta[j]-\theta_0[j]|^2}{\sigma_j^2} + \sum_{i =1}^{p} \ln \frac{\nu_j}{1/\sigma_j^2} \nonumber\\
	& \leq \sum_{l =1}^{\tilde d} \ln \frac{\tilde \nu_{(l)}}{1/\tilde \sigma_{(l)}^2} + \sum_{j =1}^{p}\frac{|\theta[j]-\theta_0[j]|^2}{\sigma_j^2}  
    \label{Eq: A3}
	\end{align}
	The proof follows by replacing this upper bound Eq. \eqref{Eq: A3} in the PAC-Bayes bound in Eq. \eqref{eq: pac}
	
	Based on Corollary \ref{corr: KPQ}, the $KL(\cQ \| \mathcal{P})$ in the right hand side is scale-invariant, so the bound Eq. \eqref{Eq: A3} is scale-invariant.
\end{proof}
That completes the proof. \qed

\end{document}